\documentclass[twoside,11pt]{article}
\usepackage[preprint, abbrvbib, nohyperref]{jmlr2e}
\usepackage[english]{babel} % document language
\usepackage[expansion=false]{microtype}      % better typographic spacing
\usepackage{csquotes}       % intelligent quotation handling
\usepackage[normalem]{ulem} % underline, strikeout, etc.
\usepackage{xspace}         % smart spacing in custom macros
\usepackage{textcomp}       % additional text symbols

\usepackage{type1cm}        % scalable Type1 Computer Modern font
\usepackage{bold-extra}     % bold small caps and bold italics
\usepackage{bm}             % bold math symbols
\usepackage{bbm}            % blackboard bold for indicator functions
\usepackage{bbold}          % blackboard bold for numbers
\usepackage{siunitx}        % proper number and unit formatting
\usepackage{xfrac}          % diagonal (slanted) fractions

\usepackage{amsmath,amsfonts,amssymb,mathtools} % AMS math extensions
\usepackage{physics}        % physics-style notation
\usepackage{thmtools,thm-restate} % advanced theorem handling

\usepackage{booktabs}       % professional-quality tables
\usepackage{multirow}       % multi-row cells in tables
\usepackage{array}          % extended table column definitions

\usepackage{graphicx}       % include graphics
\usepackage[style=base, tableposition=top]{caption} % consistent captions
\usepackage[caption=false,font=normalsize,labelfont=sf,textfont=sf]{subfig} % subfigures with sans-serif captions
\usepackage{tikz}           % TikZ graphics
\usetikzlibrary{patterns, positioning, arrows, bayesnet, calc} % TikZ libraries
\usepackage{tikz-cd}        % commutative diagrams
\usepackage{pgfplots}       % plots
\usepgfplotslibrary{patchplots}
\pgfplotsset{compat=1.15}   % pgfplots compatibility mode
\usepackage[most]{tcolorbox}      % colored boxes
\usepackage{wrapfig}
\usepackage{algorithm}
\usepackage{algorithmic}
\usepackage[bookmarks,colorlinks=true]{hyperref} % clickable refs
\usepackage[capitalize]{cleveref} % smarter cross-references
\usepackage{balance}        % balance columns on last page
\usepackage{hyphenat}       % hyphenation control
\usepackage[show]{chato-notes} % margin notes (for draft/review)
\usepackage{stfloats}       % place floats at the bottom of the page
\usepackage{verbatim}       % comment out large blocks of text

\newcommand{\spara}[1]{\smallskip\noindent\textbf{#1}}

\newcommand{\epara}[1]{\smallskip\noindent\emph{#1}}

{\par\egroup\vskip 0.25ex}

\newenvironment{squishlist}
{\begin{list}{$\bullet$}
 {\setlength{\itemsep}{0pt}
     \setlength{\parsep}{3pt}
     \setlength{\topsep}{3pt}
     \setlength{\partopsep}{0pt}
     \setlength{\leftmargin}{1.5em}
     \setlength{\labelwidth}{1em}
     \setlength{\labelsep}{0.5em} } }
{\end{list}}

\usepackage[dvipsnames]{xcolor}
\hypersetup{
   colorlinks=true,
   linkcolor={red!50!black},
   filecolor={green!50!black},
   citecolor={green!0!black}, 
   urlcolor={blue!80!black},
}

\definecolor{mypurple}{RGB}{254, 68, 218}
\definecolor{myred}{HTML}{E13D66}
\definecolor{mycyan}{HTML}{70D7D0}
\definecolor{mylightblue}{HTML}{2274A5}
\definecolor{mydarkblue}{HTML}{0C0A3E}

\newcommand{\myred}[1]{\textcolor{myred}{#1}}
\newcommand{\mycyan}[1]{\textcolor{mycyan}{#1}}
\newcommand{\mylightblue}[1]{\textcolor{mylightblue}{#1}}
\newcommand{\mydarkblue}[1]{\textcolor{mydarkblue}{#1}}

\newenvironment{restated}[2]{%
  \par\medskip
  \noindent\textbf{#1~\ref{#2} (restated).}\itshape\ %
}{%
  \par\medskip
}

\tcolorboxenvironment{example}{
  colback=black!5!white,
  colframe=black!5!white,
  boxrule=0.8pt,
  arc=3mm,
  fonttitle=\bfseries,
  before skip=10pt,
  after skip=10pt,
  breakable
}

\crefname{theorem}{Thm.}{Thms.}
\crefname{proposition}{Prop.}{Props.}
\crefname{lemma}{lem.}{lems.}
\crefname{corollary}{Cor.}{Cors.}
\crefname{definition}{Def.}{Defs.}
\crefname{section}{Sec.}{Secs.}
\crefname{figure}{Fig.}{Figs.}
\crefname{problem}{Prob.}{Probs.}
\crefname{appendix}{App.}{Apps.}
\crefname{equation}{Eq.}{Eqs.}
\crefname{algorithm}{Alg.}{Algs.}
\crefname{table}{Tab.}{Tabs.}

\newcommand{\nb}{~}

\DeclareMathOperator*{\argmin}{arg\,min}

\newcommand{\at}[2][]{#1|_{#2}}
\newcommand{\cc}[3]{\ensuremath{cc_{#1}\left(#2,#3\right)}\xspace}

\newcommand{\frob}[1]{\ensuremath{\norm{#1}_{\mathrm{F}}}\xspace}

\newcommand{\measure}{\ensuremath{\chi}\xspace}
\newcommand{\myvec}[1]{\ensuremath{\mathrm{vec}\left(#1\right)}\xspace}

\newcommand{\prox}{\ensuremath{\mathrm{prox}}\xspace}

\newcommand{\spectral}{\textsc{spectral}\xspace}
\newcommand{\mcalsep}{\textsc{mixture-calsep}\xspace}
\newcommand{\ntrials}{\textsf{\small{ntrials}}\xspace}
\newcommand{\im}[1]{\ensuremath{\mathrm{Im}\left(#1\right)}\xspace}
\newcommand{\spn}[1]{\ensuremath{\mathrm{span}\left(#1\right)}\xspace}
\newcommand{\supp}{\ensuremath{\mathrm{supp}}\xspace}
\newcommand{\Wtwo}{\ensuremath{\mathrm{W}_2}\xspace}
\newcommand{\MPW}{\ensuremath{\mathrm{MPW}_2}\xspace}

\newcommand{\CSprob}{\ensuremath{\mathsf{CSprob}}\xspace}

\newcommand{\faceincidenceposet}{\ensuremath{(\Pi,\trianglelefteq)}\xspace}
\newcommand{\Index}{\ensuremath{\mathsf{Ind}}\xspace}

\newcommand{\Prob}{\ensuremath{\mathsf{Prob}}\xspace}
\newcommand{\SCMcat}{\ensuremath{\mathsf{SCM}(\Index,\Prob)}\xspace}
\newcommand{\MCMcat}{\ensuremath{\mathsf{MCM}(\Index,\CSprob)}\xspace}

\newcommand{\catidentity}[1]{\ensuremath{\mathrm{id}_{#1}}\xspace}
\newcommand{\Poset}[1]{\ensuremath{({#1}, \leqslant)}\xspace}

\newcommand{\pd}{\ensuremath{\mathcal{S}_{++}}\xspace}

\newcommand{\reall}{\ensuremath{\mathbb{R}}\xspace}
\newcommand{\scmcollection}{\ensuremath{\mathcal{K}}\xspace}
\newcommand{\mcmcollection}{\ensuremath{\mathcal{K}}\xspace}

\newcommand{\starr}[1]{\ensuremath{\star\left({#1}\right)}\xspace}
\newcommand{\cochainspace}[3]{\ensuremath{\mathcal{C}^{#1}\left({#2}; {#3}\right)}\xspace}
\newcommand{\gsspace}[2]{\ensuremath{\Gamma\left({#1}; {#2}\right)}\xspace}

\newcommand{\edgeset}{\ensuremath{\mathcal{E}}\xspace}
\newcommand{\graph}{\ensuremath{\mathcal{G}}\xspace}
\newcommand{\parents}{\ensuremath{\mathcal{P}}\xspace}
\newcommand{\vertexset}{\ensuremath{\mathcal{V}}\xspace}

\newcommand{\stiefel}[2]{\ensuremath{\mathrm{St}({#1},{#2})}\xspace}

\newcommand{\simplex}[1]{\ensuremath{\Delta_{#1}}\xspace}

\newcommand{\ones}{\ensuremath{\boldsymbol{1}}\xspace}
\newcommand{\zeros}{\ensuremath{\boldsymbol{0}}\xspace}
\newcommand{\veca}{\ensuremath{\mathbf{a}}\xspace}
\newcommand{\vecb}{\ensuremath{\mathbf{b}}\xspace}
\newcommand{\vecc}{\ensuremath{\mathbf{c}}\xspace}
\newcommand{\x}{\ensuremath{\mathbf{x}}\xspace}
\newcommand{\y}{\ensuremath{\mathbf{y}}\xspace}
\newcommand{\w}{\ensuremath{\mathbf{w}}\xspace}
\newcommand{\vecu}{\ensuremath{\mathbf{u}}\xspace}
\newcommand{\vecv}{\ensuremath{\mathbf{v}}\xspace}
\newcommand{\A}{\ensuremath{\mathbf{A}}\xspace}
\newcommand{\B}{\ensuremath{\mathbf{B}}\xspace}
\newcommand{\C}{\ensuremath{\mathbf{C}}\xspace}

\newcommand{\identity}{\ensuremath{\mathbf{I}}\xspace}
\newcommand{\eye}[1]{\ensuremath{\identity_{#1}}\xspace}

\newcommand{\M}{\ensuremath{\mathbf{M}}\xspace}

\newcommand{\myP}{\ensuremath{\mathbf{P}}\xspace}

\newcommand{\R}{\ensuremath{\mathbf{R}}\xspace}

\newcommand{\U}{\ensuremath{\mathbf{U}}\xspace}
\newcommand{\V}{\ensuremath{\mathbf{V}}\xspace}
\newcommand{\Vhat}{\ensuremath{\widehat{\mathbf{V}}}\xspace}
\newcommand{\X}{\ensuremath{\mathbf{X}}\xspace}
\newcommand{\Y}{\ensuremath{\mathbf{Y}}\xspace}
\newcommand{\Z}{\ensuremath{\mathbf{Z}}\xspace}
\newcommand{\blockM}{\ensuremath{\mathbb{M}}\xspace}
\newcommand{\blockMl}{\ensuremath{\mathbb{M}^\ell}\xspace}

\newcommand{\BPsi}{\ensuremath{\boldsymbol{\Psi}}\xspace}
\newcommand{\BUpsilon}{\ensuremath{\boldsymbol{\Upsilon}}\xspace}
\newcommand{\covM}{\ensuremath{\boldsymbol{\Sigma}}\xspace}

\newcommand{\myendogenous}{\ensuremath{\mathcal{X}}\xspace}
\newcommand{\myendogenousvals}{\ensuremath{\mathcal{V}}\xspace}
\newcommand{\myexogenous}{\ensuremath{\mathcal{Z}}\xspace}
\newcommand{\myexogenousvals}{\ensuremath{\mathcal{U}}\xspace}
\newcommand{\myfunctional}{\ensuremath{\mathcal{F}}\xspace}

\newcommand{\scm}[1]{\ensuremath{\mathsf{M}^{#1}}\xspace}

\newcommand{\mcm}[1]{\ensuremath{\mathsf{MCM}^{#1}}\xspace}

\newcommand{\abst}{\ensuremath{\boldsymbol{\alpha}}\xspace}

\newcommand{\covhigh}{\ensuremath{\boldsymbol{\Sigma}^{h}}\xspace}
\newcommand{\covlow}{\ensuremath{\boldsymbol{\Sigma}^{\ell}}\xspace}

\newcommand{\measurehigh}{\ensuremath{\chi^{h}}\xspace}
\newcommand{\measurelow}{\ensuremath{\chi^{\ell}}\xspace}

\newcommand{\scmhigh}{\ensuremath{\mathsf{M}^h}\xspace}
\newcommand{\scmlow}{\ensuremath{\mathsf{M}^\ell}\xspace}

\newcommand{\CANA}{\ensuremath{\mathbb{A}}\xspace}
\newcommand{\CANB}{\ensuremath{\mathbb{B}}\xspace}
\newcommand{\CAND}{\ensuremath{\mathbb{D}}\xspace}
\newcommand{\CANL}{\ensuremath{\mathbb{L}}\xspace}
\newcommand{\CAN}{\ensuremath{\mathbb{G}}\xspace}

\newcommand{\CANAl}{\ensuremath{\mathbb{A}^\ell}\xspace}
\newcommand{\CANDl}{\ensuremath{\mathbb{D}^\ell}\xspace}
\newcommand{\CANLl}{\ensuremath{\mathbb{L}^\ell}\xspace}

\newcommand{\GMM}[2]{\ensuremath{\mathrm{GM}_{#1}\left(\reall^{#2}\right)}\xspace}

\usepackage{lastpage}
\jmlrheading{23}{2026}{1-\pageref{LastPage}}{1/21; Revised 5/22}{9/22}{21-0000}{Gabriele D'Acunto, Paolo Di Lorenzo and Sergio Barbarossa}

\ShortHeadings{The Causal Abstraction Network}{D'Acunto, Di Lorenzo, Barbarossa}
\firstpageno{1}

\begin{document}

\title{Networks of Causal Abstractions:\\A Sheaf-theoretic Framework}

\author{\name Gabriele D'Acunto \email gabriele.dacunto@uniroma1.it \\
\addr  DIET\\ 
Sapienza University of Rome\\
00184 Rome, Italy
\AND
\name Paolo Di Lorenzo \email paolo.dilorenzo@uniroma1.it \\
\addr  DIET\\ 
Sapienza University of Rome\\
00184 Rome, Italy
\AND
\name Sergio Barbarossa \email sergio.barbarossa@uniroma1.it \\
\addr  DIET\\ 
Sapienza University of Rome\\
00184 Rome, Italy}

\editor{My (mean) editor}

\maketitle

\begin{abstract}%   <- trailing '%' for backward compatibility of .sty file
A core challenge in causal artificial intelligence is the principled coordination of multiple, imperfect, and subjective causal perspectives arising from distributed agents with limited and heterogeneous access to the environment. 
This problem has received little formal treatment, as the existing framework assumes a single shared global causal model. 
This work introduces the causal abstraction network (CAN), a general sheaf-theoretic framework for representing, learning, and reasoning across collections of mixture of causal models (MCMs)--a class that unifies several existing models of context-dependent causal mechanisms. 
Sheaf theory provides a natural foundation for this task, offering a rigorous framework to coherently align distributed causal knowledge without requiring explicit causal graphs, functional mechanisms, interventional data, or jointly sampled observations.

At the theoretical level, we provide a categorical formulation of MCMs and characterize key properties of CANs, including consistency and smoothness. 
Under consistency, we establish necessary and sufficient conditions: \emph{(i)} for the existence of global sections, linked to spectral properties of an associated connection Laplacian; and \emph{(ii)} for the convergence of causal knowledge diffusion over the CAN to the space of global sections. 
At the methodological level, we exploit the compositionality of causal abstractions to decompose the learning of consistent CANs into local problems on network edges, extending our prior work on Gaussian variables to Gaussian mixtures via the proposed \mcalsep algorithm. 
We validate the framework on synthetic data and through a financial application involving a multi-agent trading system, demonstrating CAN recovery, CAN-based portfolio optimization, and counterfactual reasoning.
\end{abstract}


\clearpage

% !TEX root =  ../main.tex
\section{Introduction}\label{sec:introduction}
\emph{Causal AI} \citep{b2025causal} is an emerging paradigm aiming at explainability, trustworthiness, and robustness to domain and distribution shifts.
Unlike predictive AI--which makes decisions based solely on observed associations--causal AI aims to intervene, reason, and simulate alternative scenarios based on cause and effect. 
For instance, a healthcare AI recommending treatments can go beyond \enquote{patients with symptoms X often respond well to drug Y} \emph{(predictive)} and instead evaluate \enquote{if we administer drug Y to this specific patient, what is the expected causal effect on recovery, and what if we had chosen drug Z instead?} \emph{(causal)}. 
At the heart of causal AI lies the Structural Causal Modeling (SCM,\nb\citealp{pearl2009causality}) framework, in which causal variables are described in terms of \emph{causal mechanisms}, the incoming causal relations from other variables.
SCMs enable the three levels of Pearl's Causal Hierarchy (PCH, \citealp{pearl2018book}). 
At the first level (L1: \emph{seeing}), an AI agent--regarded as any natural or artificial entity that can interact with an environment--passively observes data and learns statistical associations, without true causal reasoning. 
For example, it may learn that \enquote{smokers are more likely to develop lung disease,} but without distinguishing correlation from causation. 
At the second level (L2: \emph{doing}), the agent actively intervenes on the SCM, much like a child learning through trial and error, to uncover cause-effect relations. 
For instance, it may simulate \enquote{what happens if I ban smoking?} and observe downstream changes. 
Finally, the third level (L3: \emph{imagining}), is counterfactual: the agent can retrospectively ask \enquote{what if I had acted differently?}--for example, \enquote{if this patient had received a different treatment, would they have recovered faster?} 
An SCM thus induces three families of distributions, i.e. observational at L1, interventional at L2, and counterfactual at L3, which we will collectively refer to as its \emph{causal knowledge} (CK).

In the canonical setting in causal AI, the environment is governed by a single ground-truth SCM \scm{*} \citep{b2025causal}. 
Each agent interacts with this environment at different levels of the PCH--by observing, intervening, and imagining through \scm{*}. 
In this framework, all other agents are irrelevant to this interaction, thus exerting no external influence. 
However, as humans we know that causality is not only about uncovering an objective truth, but--arguably mainly--about constructing subjective models of how the world works--models that are continuously shaped by social interactions. 
For example, individuals may refuse to participate in a randomized clinical trial not because of statistical evidence, but because relatives or friends experienced negative outcomes from similar treatments (spillover effect, \citealp{chernozhukov2024applied}). 
Also in application domains, such as signal processing, this classical assumption is unrealistic.
In this setting, agents are distributed computational entities equipped with local processing and communication capabilities, interacting over a network topology. Depending on the application, these agents may correspond to autonomous vehicles, robots, or wireless devices embedded in the infrastructure.
Each agent acquires measurements from the same system through its sensors--thus from its own perspective and at the relevant resolution,--processes signals locally, and exchanges information with other agents over the network.
These examples--and many others--motivate the investigation of principled way to compare, integrate, or abstract \emph{local} SCMs.

In this pursuit, and taking into account the \emph{collaborative} dimension central to agentic AI, \citet{d2025relativity} postulated \emph{\enquote{SCMs as subjective and imperfect representations of the world, that cannot be severed from the network of relations they are immersed in}}.
Such a shift opens to a novel \emph{relative framework of causal knowledge} where, unlike the canonical setting, agents cannot interact with the environment while detaching neither from their subjective SCM (the agent \emph{perspective}) nor from their relations represented as a network. 
Subjective CK can thus be altered not only by the environment but also through interactions with other agents. 
Ideally, this interplay benefits both the individual agent and the entire AI system. 
Additionally, in this relative framework, asking for a global ground-truth \scm{*} is ill-posed: causality is inherently relative, and global CK emerges only through the interaction and alignment of subjective CK across the network.
The central challenge then becomes how to compare, integrate, and abstract CK across agents in a principled manner, while preserving the causal semantics encoded in their local models.
Sheaf-theoretic approaches--namely network sheaves and cosheaves--offer a natural formal answer to this challenge. 
Intuitively, they can be thought as \emph{formal tools for gluing local pieces of information into a global picture}. 
Here, subjective CK corresponds to the local pieces, while consistency across overlaps ensures that a coherent global CK can emerge.

This novel family of network sheaves and cosheaves potentially opens to AI systems where \emph{global CK} emerges \emph{inductively} from the interaction of causal AI agents, while local CK at the agent level can, in turn, be \emph{deductively} shaped by the system.
Our work moves in this direction, investigating the theoretical, learning, and applied aspects of a specific instance of network sheaves and cosheaves of CK, termed \emph{causal abstraction network} (CAN).
Recently,\nb\citet{richens2024robust} and\nb\citet{richens2025general} showed that different agent skills, such as task and domain generalization, provably require different degrees of CK.
From this stance, CAN is particularly relevant not only to the local dimension of causal models, but also because the degree of CK can be modulated according to the agent’s goals within the causal abstraction framework.
Additionally, CAN is relevant to the field of \emph{mechanistic interpretability}, as recent evidence suggests that the internal states of LLMs may be better explained by combinations of related causal models\nb\citep{pislar25a}.

\spara{Roadmap.} 
\Cref{sec:related_works} situates our work within the literature by reviewing the main related works.
\Cref{sec:contributions} gives an overview of our contribution and
\Cref{sec:background} provides the essential background on the SCM framework, causal abstraction, network sheaves and cosheaves of CK, and Gaussian mixtures.
Next, \Cref{sec:categorical-CMM,sec:CAN} presents our theoretical contributions. 
Specifically, \Cref{sec:categorical-CMM} unifies the main existing frameworks of mixtures of causal models and provides a category-theoretic treatment showing that their associated category constitutes a natural extension of the category of SCMs introduced by \cite{d2025relativity}.
Then, \Cref{sec:CAN} starts with the category-theoretic formalization of CANs (\Cref{subsec:ct_formalization}) and its algebraic descriptors (\Cref{subsec:algebraic_descriptors}), then analyzes the sheaf-theoretic representation and operators for the CK transfer (\Cref{subsec:CAN_as_ns}), and finally investigates the theoretical properties of consistent CANs (\Cref{subsec:CAN_consistency,subsec:global_L,subsec:hyperparams,subsec:smoothness}).
\Cref{sec:learning_CAN} presents our methodological contribution. 
It first formalizes the problem of learning consistent CANs from a collection of probability distributions assumed to be smooth with respect to the latter, and then details the mathematical derivation of our solver (\Cref{subsec:search-proc,subsec:local-prob}).
Additionally, \Cref{sec:empirical_assessment} reports results on synthetic data for both CLCA and CAN learning (\Cref{subsec:clca-learning,subsec:can-learning}).
\Cref{sec:finappl} showcases our applied contribution in the financial domain, starting with the construction of a CAN-based trading system and then addressing CAN learning (\Cref{subsec:net_sheaf_learning_appl}) and a reasoning task (\Cref{subsec:counterfactual_reasoning_appl}).
\Cref{sec:discussion} discusses the salient features of the proposed framework as well as the limitations of our work.
Finally, \Cref{sec:conclusions} draws some conclusions and outlines future research directions.
% !TEX root =  ../main.tex
\section{Related Works}\label{sec:related_works}

We build upon the general language of \emph{category theory} \citep{mac2013categories}, which is axiomatically centered on consistent mappings among objects of a given category; in our case, causal models.
We work in the context of \citep{d2025causal}, specifically without assuming: 
\emph{(i)} complete specification of the SCMs;
\emph{(ii)} knowledge of their causal graphs;
\emph{(iii)} their functional forms;
\emph{(iv)} availability of interventional data; or
\emph{(v)} jointly sampled observational data.
Accordingly, we extend the category-theoretic treatment of SCMs by \citet{d2025causal,d2025relativity}, a functorial representation that separates \emph{syntax} (i.e., structure) from \emph{semantics} (i.e., CK) of the SCM.
Alternative category-theoretic formalizations, inspired by the seminal work on categorical probability theory by \citet{fong2013causal}, leverage copy-discard and Markov categories to provide a string diagrammatic description of causal models \citep{rischel2020category,jacobs2021causal,fritz2023d,lorenz2023causal}.
These works are more tailored to deterministic causal models involving discrete distributions. Conversely, our work focuses on probabilistic causal models with continuous probability distributions.
More recently, \citet{mahadevan2025universal} introduced the topos causal model, which generalizes prior categorical approaches as well as the SCM \citep{pearl2009causality} and the potential outcome frameworks \citep{imbens2015causal}.

As in \citet{d2025relativity}, the causal abstraction framework is key to determine mappings between SCMs.
More in detail, we adopt the \emph{functional} \abst-abstraction framework \citep{rischel2020category,rischel2021compositional}, focusing on \emph{constructive linear causal abstractions} (CLCAs, \citealp{massidda2024learning}) and leverage the \emph{semantic embedding principle} (SEP, \citealp{d2025causal}).
Our choice is made according to the five \enquote{non-assumptions} above, characterizing our setting.
An alternative functional framework is the $(\tau,\omega)$-abstraction \citep{rubenstein2017causal,beckers2019abstracting,massidda2023causal}, which relies on \emph{(i)} a single function $\tau$ mapping the values of the low-level model into those of the high-level and \emph{(ii)} a surjective and order-preserving map $\omega$ between interventions on the two models.
At the \emph{graphical} level instead we find the \emph{cluster DAG} framework \citep{anand2023causal,xia2024neural}, which is based on clustering the low-level causal variables.
The equivalence among the \abst, $(\tau,\omega)$, and cluster DAG frameworks has been recently investigated by \citet{schooltink2025aligning}.
Other approaches investigate causal abstractions from a category-theoretic perspective, modeling them as \emph{natural transformations} \citep{otsuka2022equivalence,lorenz2026causal}.
These approaches build upon the aforementioned string-diagrammatic formalization of causal models, while our work is based upon Def.\nb6 in \citep{d2025relativity}.

The CAN framework synthesizes network sheaves and cosheaves of CK introduced by \citet{d2025relativity} when \emph{(i)} subjective causal models induce observational Gaussian mixtures and, for each node-edge incidence relation, \emph{(ii)} the restriction map is a permutation matrix if the model associated with the node is the finest among the two models at the endpoints of the edge, and otherwise the transpose of a CLCA (cf. \Cref{subsec:CAN_as_ns}).
The former point links the subjective models within a CAN with the recently developed mixtures of causal models (MCMs, \citealp{geiger1996knowledge,thiesson1998learning,kumar2024learning,varambally24a,mamechecausal}), where the generating process is a mixture of different causal functional assignments. 
Notably, in our framework, the ability to handle subjective models that are themselves MCMs not only adds flexibility, making the framework more broadly applicable, but is also a natural requirement, since the diffusion operators defined for the CAN produce Gaussian mixtures, which cannot be captured by canonical SCMs (cf. Rem.~\ref{rem:why-gmm}).

From the theoretical perspective, our work relates to \citet{ghrist2022cellular}, which develops cellular sheaves valued in the category of lattices, a non-Abelian setting akin to convex spaces of probability measures \citep{fritz2009convex}.

At the methodological level, we leverage the work of \citet{cai2022distances} on distances between probability distributions with different dimensions to formulate the problem of learning a consistent CAN from a collection of Gaussian mixtures (cf. \ref{eq:globalProblemGMM}).
Additionally, the alternating optimization approach in our algorithmic solution relates to that of \citet{salmona2024gromovwassersteinlike} which is based on the mixture embedded Wasserstein distance for learning optimal transportation plans between Gaussian mixture models. 

Finally, this article extends our previous work~\citep{dacunto2026learningconsistentcausalabstraction} along theoretical, methodological, and experimental directions. 
While the main contribution of the latter was a computationally efficient learning procedure for CANs in the Gaussian setting, the present work substantially broadens the scope of that approach. 
In particular, as detailed in \cref{sec:contributions}, we 
\emph{(i)} further develop the theoretical apparatus underlying CANs and establish results that do not, in fact, rely on Gaussianity; 
\emph{(ii)} introduce \mcalsep, which is able to handle the practically relevant case of mixtures of Gaussians (cf.\ Rem.~\nb\ref{rem:why-gmm}); 
and \emph{(iii)} provide a real-world case study that is of broader interest to the sheaf-theoretic approach.

% !TEX root =  ../main.tex
\section{Contributions}\label{sec:contributions}
Our contribution is threefold.
From the \emph{theoretical} standpoint, we first unify the main existing definitions of MCMs and provide a categorical formulation thereof.
Subsequently, we introduce CAN, again within the general language of category theory, and investigate its properties.
Specifically,
\emph{(i)} the algebraic descriptors induced by the combinatorial network,
\emph{(ii)} the boundary, coboundary and Laplacian operator determining the diffusion of CK, 
\emph{(iii)} the consistency of CAN, its relation to SEP, and the spectral properties of the connection Laplacian of a consistent CAN,
\emph{(iv)} the global sections of a consistent CAN--informally, consistent assignments of CK to each node of the network which lead to a perfect alignment among the subjective MCMs--tying their existence to the kernel dimensionality of the connection Laplacian, characterizing their statistical structure, and establishing the convergence of the diffusion of CK over the CAN to the space of global sections, 
\emph{(v)} the smoothness of a collection of probability distributions with respect to a consistent CAN, quantified through an information-theoretic metric or $\phi$-divergence depending on the application setting.   

From the \emph{methodological} standpoint, we formulate the problem of learning a consistent CAN given a collection of Gaussian mixtures assumed to be smooth with respect to the latter.
Our formulation is a combinatorial problem over the edge set with edge-specific alignment constraints,  amounting to local CLCA learning tasks.
For each of them, given as input a low-level mixture (from a detailed MCM) and a high-level one (from an abstract MCM) related by a CLCA, we propose the \mcalsep algorithm that jointly learns \emph{(i)} the coupling between the components of the input mixtures and \emph{(ii)} the CLCA between them, based on structural prior knowledge. 
Hence, as a byproduct, our work also contributes to the existing literature on learning causal abstractions from data.
The proposed solution for the CAN learning problem is an efficient search procedure based on  Cor.\nb\ref{cor:spectral_interlacing_gmm}, and the compositionality of CLCAs.
This procedure employs \mcalsep to solve the edge-specific problems.
We tested our methods on a series of experiments on synthetic data, also comparing with the baselines in \cite{d2025causal,dacunto2026learningconsistentcausalabstraction} on the CLCA learning task with varying degrees of abstraction. 
Overall, the proposed methods show competitive performance and good recovery capabilities.

Last but not least, from the \emph{application} standpoint, we show a use-case in finance in which we build a simple trading AI-system with \num{5} agents using the CAN framework. Starting from real industry-portfolio returns at two levels of aggregation, we explicitly construct a ground-truth CAN encoding CLCA relationships across agents.
Then, we show that the CAN can be recovered with high-precision from the induced global section using our learning procedure. 
Additionally, we assess the impact of using learned (rather than prescribed) CLCA maps on a downstream mean-variance portfolio allocation task.
Finally, we consider a counterfactual reasoning task in which, given portfolio allocations observed from a set of investors operating on the same investment universes as the CAN-based trading system, we infer the causal knowledge of the coarsest agent that induces a counterfactual global section under which the observed allocations are optimal in a mean-variance sense. 
Within this setting, we study how the resulting counterfactual global section varies with the risk-aversion parameter of the trading system, and quantify how far the system’s own optimal allocation departs from the investors’ observed allocation as the risk-aversion changes.
% !TEX root =  ../main.tex

\section{Background}\label{sec:background}
This section reviews the main tools underlying our work.
Rather than focusing on technical details, we emphasize the role played by each framework in our work.
Formal definitions, theoretical results, and proofs are deferred to the corresponding cited references.

\spara{Category theory.}
Category theory \citep{mac2013categories} is a branch of pure mathematics investigating relationships among abstract structures by means of objects and morphisms.
Being axiomatically centered on relations, the general language of this theory eases the shift from the single causal model-centric to the relative setting introduced in \Cref{sec:introduction}.
Intuitively, as nicely discussed by \citet[Ch.1]{perrone2024starting}, we can think about a category as structured objects---e.g., vector spaces, measurable spaces---and composable arrows (morphisms) between them preserving their structure---e.g., linear maps, measurable maps---and satisfying certain axioms (unitality and associativity).
In this work, SCMs, their intervened states and abstractions, are treated as objects of a category connected by structure-preserving maps, ensuring that causal semantics is maintained when models are compared, composed, or integrated.
For our work, a key advantage of the categorical viewpoint is compositionality: complex causal systems can be built and analyzed by composing simpler components in a principled way.
This property is essential for reasoning about networks of causal AI agents, where local causal knowledge must be related consistently across agents and scales.

\spara{Structural causal models.}
An SCM \scm{} \citep{pearl2009causality} consists of endogenous variables, exogenous variables capturing latent influences, and structural assignments forming an acyclic system of equations describing how each endogenous variable is generated from its parents and exogenous variable.
SCMs are commonly represented as directed acyclic graphs (DAGs), where nodes correspond to variables and edges encode direct causal relationships.

A formal category-theoretic reformulation of SCMs, as well as of interventions and abstractions, is developed in \citep{d2025relativity,d2025causal}.
Essentially, the category-theoretic SCM is an arrow--formally a \emph{functor}--from a syntactic to a semantic category.
The former is made of two objects and one arrow between them, the latter of the probability spaces of exogenous and endogenous variables related by a a collection mixing measurable maps. 
Then, the category-theoretic SCM maps objects of the syntactic category to the probability spaces and the only arrow to the collection of mixing measurable maps.
This reformulation is also general enough to encompass both Markovian and semi-Markovian SCMs and will be extended to the MCMs case in \Cref{sec:categorical-CMM}, providing the essential categorical language underlying the sheaf-theoretic constructions in \cref{sec:CAN}.

\spara{Causal abstraction.}
In many applications, the same phenomenon can be modeled at multiple levels of granularity.
\emph{Causal abstraction} (CA) formalizes principled mappings between such models, allowing one to relate low-level and high-level SCMs while preserving causal meaning \citep{rubenstein2017causal,beckers2019abstracting,rischel2020category}.
In our work, we leverage the \abst-abstraction framework \citep{rischel2020category}.
 
Intuitively, an \abst-abstraction specifies how relevant endogenous variables in the low-level are aggregated at the high-level via a surjective structural map, and how their values correspond to those of the high-level via modular functions.
\emph{Interventional consistency} (IC) is a common requirement for CA. 
In words, we obtain equivalent outcomes \emph{(i)} either by intervening on the low-level model and then abstracting or, \emph{(ii)} by abstracting to the high-level model and then intervening in an equivalent fashion.
An important class of IC abstractions are the \emph{constructive abstractions}, implying a clustering of the low-level variables into the high-level and a map between exogenous variables as well \citep{beckers2019abstracting}.
In our work, we focus on \emph{constructive linear causal abstractions} (CLCAs, \citealp{massidda2024learning,d2025causal}), where 
\emph{(i)} all low-level variables are relevant, 
\emph{(ii)} $\B^\top \in \{0,1\}^{h,\ell}$ encodes a structural disjoint partitioning of $\ell$ low-level variables $\myendogenous^\ell$ into $h$ high-level $\myendogenous^h$, 
and \emph{(iii)} $\V^\top \in \reall^{h \times \ell}$ specifies mappings between outcomes, whose nonzero entries are determined by ones in $\B^\top$.

Recently, \cite{d2025causal} introduced the \emph{semantic embedding principle} (SEP). 
Intuitively, SEP implies that going from the high-level model \scmhigh to the low-level model \scmlow and then abstracting back to \scmhigh allows for perfect reconstruction of CK. 
Importantly, for CLCA, a principled way to satisfy SEP is via the geometry of the Stiefel manifold:
\begin{equation}\label{eq:stiefel}
    \stiefel{\ell}{h} \coloneqq \{ \V \in \reall^{\ell \times h} \, \mid \, \V^\top\V = \identity_h \}\,.
\end{equation}
SEP induces a necessary condition for the existence of CLCA between two Gaussian distributions \citep{d2025causal}, an application of the Ostrowski's theorem \citep{higham1998modifying}.
Let $\measurelow \sim N(\zeros_\ell, \covlow)$, $\measurehigh \sim N(\zeros_h, \covhigh)$, where $\covlow \in \pd^\ell$ and $\covhigh \in \pd^h$.
Denote by $0<\lambda_1\leq \ldots \leq \lambda_\ell$ the eigenvalues of \covlow, and by $0<\kappa_1 \leq \ldots\leq \kappa_h$ those of \covhigh.
If a CLCA complying with SEP from \measurelow to \measurehigh exists, then
\begin{equation}\label{eq:spectralCA}
    \lambda_i \leq \kappa_i \leq \lambda_{i + \ell -h}, \quad \forall \,i \in [h]\,.
\end{equation}
As detailed in \Cref{subsec:global_L}, a similar spectral interlacing holds for CLCAs between Gaussian mixtures of different dimensions.

\spara{Network sheaves and cosheaves of causal knowledge.}
To formalize the interaction of multiple causal AI agents, we adopt the framework of \emph{network sheaves and cosheaves of causal knowledge} by \citet{d2025relativity}.
Informally, these constructions can be thought of as assignments of CK to nodes and edges of a network, where node-edge incidence relations correspond to maps that preserve the causal meaning of CK.
Intuitively, a network sheaf describes how local causal knowledge held by individual agents at the node level can be ported to agents at the edge level.
Conversely, a network cosheaf captures how shared or aggregated causal knowledge at the edge level can be propagated back to agents at the node level.
Together, they formalize the {\it bidirectional interplay between local perspective and global coherence} that characterizes the relative causal framework.

Our model assigns a probability distribution, typically a Gaussian mixture, to each node and edge in the graph-- either observational, interventional or counterfactual--induced by the local SCMs.
Collectively, the assignments over the nodes form a $0$-cochain, those on the edges a $1$-cochain.
A central concept in this setting is that of a \emph{global section}, representing a consistent assignment of probability distributions across the entire network that does not break local rules.
When such consistency holds true, {\it global causal knowledge} is said to emerge from the interaction and alignment of {\it subjective and local causal knowledge} among different agents.
The precise categorical constructions and their properties are developed in\nb\citep{d2025relativity} and are recalled as needed below.

\spara{Gaussian mixtures.}
Gaussian mixtures (GMs) provide a flexible class of probability distributions obtained as convex combinations of Gaussian components, relevant to the methodological and empirical parts of our work.
We denote by \GMM{S}{d} the space of Gaussian mixtures with at most $S$ components in $\reall^d$.
A generic Gaussian mixture in \GMM{S}{d} is given by
\begin{equation}\label{eq:gmm}
    \measure = \sum_{s \in [S^\prime]}  w_s \measure_s\,;    
\end{equation}
where $S^\prime \leq S$, $\w = [w_1, \ldots, w_{S^\prime}]^\top \in \simplex{S}$ lies in the probability simplex, and each component is a Gaussian distribution $\measure_s \sim N(\bm\mu_s, \bm\Sigma_s)$.
The mean and covariance of the mixture admit closed-form expressions,
\begin{equation}\label{eq:gmm_mu}
    \bm\mu = \sum_{s \in [S^\prime]} w_s \bm\mu_s\,,
\end{equation}
and
\begin{equation}\label{eq:gmm_sigma}
    \bm\Sigma=\sum_{s\in [S^\prime]}w_s \left(\bm\Sigma_s + (\bm\mu_s - \bm\mu)(\bm\mu_s - \bm\mu)^\top\right)\,.
\end{equation}
Importantly, Gaussian mixtures are identifiable up to permutation of their components \citep{yakowitz1968identifiability}.
Additionally, they are well known universal approximators of probability distributions under mild regularity conditions.
To compare Gaussian mixtures induced by the causal models, we rely on the mixture-Wasserstein distance, which is a metric for Gaussian mixtures \citep{delon2020wasserstein,chen2018optimal}.
The technical details are deferred to \Cref{sec:learning_CAN}.
% !TEX root =  ../main.tex
\section{Category-theoretic Mixture Causal Model}\label{sec:categorical-CMM}

As stated in \cref{sec:contributions} and further clarified in Rem.\nb\ref{rem:why-gmm}, in our framework, the ability to handle subjective causal models that are themselves mixtures of causal models not only provides more flexibility, but is also a natural requirement, since the diffusion operators defined in \Cref{subsec:CAN_as_ns} produce mixtures of probability distributions, which cannot be captured by SCMs.

A mixture of causal models can be thought of as a generative model where the joint distribution of the endogenous variables is not governed by a single causal mechanism, rather by a convex combination of independent SCMs, activated by categorical latent variables encoding for instance context- or subpopulation-specific conditional independence.
Hereinafter, we refer to the categorical latent variables as \emph{activation variables}.

Over the years, a mixture of causal models has been called in many ways and presented with different nuances and in diverse application contexts. 
Among the others, \emph{Bayesian multinet} \citep{geiger1996knowledge}, \emph{multi-DAG} \citep{thiesson1998learning}, \emph{mixture of DAGs} \citep{kumar2024learning}, and \emph{causal mixture model} (CMM, \citealp{mamechecausal}).
Although the previous models tend to generalize the canonical SCM, there are some differences between CMM and others.
On one hand, while other models assume a single global activation variable for the data--in other words, an observation of the endogenous corresponds to a specific value of a latent global variable activating a certain SCM--CMM allows different local activation variables to simultaneously influence distinct sets of endogenous variables, modifying the strength of the causal relationships in different ways, but not the parent sets.
On the other hand, CMM is more stringent from a structural standpoint, as the parent sets of the endogenous variables do not vary along the values of the local activation variables, which only modify the strength of the causal relationships.
Conversely, in Bayesian multinets and multi-DAGs, the causal structure varies according to the value of the global activation variable.

However, in all models a single global causal \emph{structure} is active at a specific sample of the endogenous variables.
Leveraging this observation, and adopting hereinafter the term \emph{mixture causal model} (MCM), herebelow we formalize a unifying definition.
    
\begin{definition}[Mixture causal model]\label{def:mcm}
    A mixture causal model is a tuple 
    $$\mcm{}\coloneqq\langle \myendogenous, \myexogenous, \zeta, W, \w, \mathcal{L}, \Upsilon, \mathcal{Q} \rangle\,,$$ 
    where:
    \begin{squishlist}
        \item A collection \myendogenous of $n$ observed endogenous variables $X_i$, arranged in a random vector $X=(X_1,\ldots,X_n)$ ; 
        \item A collection \myexogenous of $n$ latent exogenous variables $Z_i$, one for each $X_i$, arranged in a random vector $Z=(Z_1,\ldots,Z_n)$;
        \item A product probability distribution $\zeta$ over jointly independent exogenous variables;
        \item A global, categorical, DAG activation latent variable $W \in [w]$;
        \item A discrete probability distribution $\w = [w_1,\ldots, w_w]^\top \in \Delta_w$, yielding $W\sim \mathrm{Categorical}(\w)$;
        \item A collection $\mathcal{L}$ of $u \leq n$ $k_j$-dimensional, local, categorical, assignment activation latent variables $L_j$, taking values in $[k_j]$, $j \in [u]$ and relating to the endogenous via a surjective map $\mathrm{La}:\myendogenous \to \mathcal{L}$;
        \item A collection $\Upsilon=\{\bm\upsilon_1,\ldots,\bm\upsilon_u\}$ of discrete probability distributions $\bm\upsilon_j$, one for each $L_j$, lying on the $k_j$-dimensional probability simplex $\Delta_{k_j}$ and yielding $L_j \sim \mathrm{Categorical}(\bm\upsilon_j)$ 
        \item A collection of $w$ pairs $\mathcal{Q}\coloneqq\{(\mathcal{D}_1,\myfunctional_1),\ldots,(\mathcal{D}_w,\myfunctional_w)\}$, such that, for each $X_i$, 
        \begin{enumerate}
           \item[(i)] the parent set $\parents_i \subseteq \myendogenous \setminus \{X_i\}$ is determined by the causal DAG $\mathcal{D}_r$ with $r$ corresponding to the value of $W$;
           \item[(ii)] the value is determined by those of the corresponding $Z_i$ and the parents $\parents_i$, according to the structural assignment $f_{r,i,h_j}$ where $h_j \in [k_j]$ is the value of $L_j=\mathrm{La}(X_i)$,
           \begin{equation}
               X_i = f_{r,i,h_j}(\parents_i,Z_i)\,.
           \end{equation}
        \end{enumerate} 
        The activation variables $W, L_1, \ldots, L_u$ are assumed to be jointly independent.
    \end{squishlist}
\end{definition}

From Def.\nb\ref{def:mcm}, MCM coincides with CMM when \emph{(i)} $w=1$ and thus $\mathcal{Q}=\{(\mathcal{D},\myfunctional)\}$, \emph{(ii)} the structural assignments $f_{i,h_j}$ are linear with additive noise, and \emph{(iii)} $\myexogenous \sim N(\zeros_n,\sigma^2\identity)$.
When $u=n$ and $k_j=1$, we get the rest of the models above, also by specifying the corresponding functional forms and noise distributions, which in the continuous case are chosen linear with additive Gaussian noise.
Notably, in the general case, the MCM in Def.\nb\ref{def:mcm} is a more general model, combining both the global and local mixing properties of the above models, and at the same time encompassing more general functional forms for the structural assignments.

\begin{figure}[h!]
    \centering
    \begin{tikzpicture}[
    observed/.style={draw, circle, minimum size=6mm, inner sep=1pt},
    latent/.style={observed, dashed}
]
        \node[latent] (W) at (2,2.5) {$W$};
        
        \node[observed] (x1) at (-1,0) {$X_1$};
        \node[observed] (x2) at (0,1) {$X_2$};
        \node[observed] (x3) at (0,-1) {$X_3$};
        \node[observed] (x4) at (1,0) {$X_4$};
        \node[latent] (L1) at (-1,-1) {$L_1$};
        \node[latent] (L2) at (1,1) {$L_2$};
        \draw[rounded corners] (-1.5,1.5) rectangle (1.5,-1.5);
        \draw[->] (x1) to (x3);
        \draw[->] (x2) to (x1);
        \draw[->] (x2) to (x4);
        \draw[->] (L1) to (x1);
        \draw[->] (L1) to (x3);
        \draw[->] (L2) to (x4);

        \node[observed] (x11) at (3,0) {$X_1$};
        \node[observed] (x21) at (4,1) {$X_2$};
        \node[observed] (x31) at (4,-1) {$X_3$};
        \node[observed] (x41) at (5,0) {$X_4$};
        \node[latent] (L11) at (3,-1) {$L_1$};
        \node[latent] (L21) at (5,1) {$L_2$};
        \draw[rounded corners] (2.5,1.5) rectangle (5.5,-1.5);

        \coordinate (rectA-top) at (0,1.5);
        \coordinate (rectB-top) at (4,1.5);
        \node at (0,-1.8){$(\mathcal{D}_1, \mathcal{F}_1)$};
        \node at (4,-1.8){$(\mathcal{D}_2, \mathcal{F}_2)$};
        
        \draw[->] (x11) to (x31);
        \draw[->] (x21) to (x11);
        \draw[->] (x21) to (x41);
        \draw[->] (x41) to (x31);
        \draw[->] (L11) to (x11);
        \draw[->] (L11) to (x31);
        \draw[->] (L21) to (x41);
        \draw[->] (W) to (rectA-top);
        \draw[->] (W) to (rectB-top);
    \end{tikzpicture}
    \caption{Example of MCM in Def.\nb\ref{def:mcm}. 
    The observed variables are enclosed within solid circles, the latent variables in dashed circles.}
    \label{fig:MCM}
\end{figure}

\begin{example}\label{ex:MCM}
   \Cref{fig:MCM} shows an example of an MCM as defined in Def.\nb\ref{def:mcm}.
    We have $n=4$ endogenous variables $\mathcal{X}=\{X_1,\ldots,X_4\}$, each associated with an exogenous variable $Z_i$, $i \in [4]$, which is not shown for clarity.
    Further, we have a single global activation variable $W\sim \mathrm{Categorical}(\w)$, with $\w \in \Delta_2$, and $u=2$ local activation variables: $L_1 \sim \mathrm{Categorical}(\bm\upsilon_1)$ and $L_2 \sim \mathrm{Categorical}(\bm\upsilon_2)$, with $\bm\upsilon_1 \in \Delta_{k_1}$ and $\bm\upsilon_2 \in \Delta_{k_2}$.
    Depending on the value of $W \in [2]$, either $(\mathcal{D}_1, \mathcal{F}_1)$ or $(\mathcal{D}_2, \mathcal{F}_2)$ is active.
    Locally, $L_1 \in [k_1]$ modifies the causal mechanisms of $X_1$ and $X_3$, that is, the strength of the causal relationships $X_2 \rightarrow X_1$, $X_1 \rightarrow X_3$, and $X_4 \rightarrow X_3$ (when $W=2$).
    Similarly, $L_2 \in [k_2]$ affects the causal mechanism of $X_4$.
\end{example}

As discussed in \Cref{sec:background}, an SCM can be interpreted as a generative mechanism that maps a probability distribution over exogenous variables to a probability distribution over endogenous variables through a measurable \emph{deterministic} map obtained by recursively composing the structural assignments according to the causal ordering induced by the DAG. 
In categorical terms, this correspondence is captured by a functor from a simple index category \Index (made only of a source node and a target node, plus a unique arrow from the source to the target) to the category of probability spaces \Prob (cf. \citealp{d2025relativity}).
The functorial representation follows by mapping \emph{(i)} the source node to the probability space of the exogenous variables, \emph{(ii)} the target node to the probability space of the endogenous variables, and \emph{(iii)} the unique arrow to the measurable map.

Now, an MCM extends this picture by allowing the generative process to be governed not by a single causal mechanism, but by a finite collection of alternative causal mechanisms, selected by latent activation variables and combined according to a mixture distribution. 
As a consequence, the probability distribution induced on the endogenous variables is no longer the pushforward of the exogenous distribution along a single measurable map, but rather a convex combination of pushforward distributions corresponding to different causal mechanisms.

This observation motivates a categorical formulation of MCMs in which the syntax encoded by \Index is preserved, while \Prob is replaced by a category capable of internalizing convex combinations of probability distributions, that is, \CSprob \citep{fritz2009convex}. 
In the latter, objects are convex spaces of probability distributions, that is, sets of probability distributions over the sample space \myendogenousvals, denoted by $\mathcal{P}(\myendogenousvals)$, equipped with a convex combination operation 
\begin{equation}\label{eq:ccl}
    cc_{\lambda}(P_1, P_2)(\mathcal{O}) \coloneqq \lambda P_1(\mathcal{O}) + \bar{\lambda} P_2(\mathcal{O})\,, \quad \text{with}\; \lambda \in [0,1] \text{ and } \bar{\lambda}=1-\lambda\,; 
\end{equation}
for all measurable events $\mathcal{O}\subseteq\myendogenousvals$.
In our setting, each $P_i \in \mathcal{P}(\myendogenousvals)$ corresponds to the probability distribution $\measure_i$ of endogenous outcomes induced by a distinct causal mechanism in the MCM, while $\lambda$ represents the weight of the corresponding mixture. 
Morphisms are affine measurable maps commuting with $cc_\lambda$.
Using \Cref{eq:ccl}, throughout we denote by $cc_{\w}(\measure^{[K]})$ a mixture of probability distributions $\measure^k$, $k \!\in\![K]$, with weights $\w \!\in\! \Delta_{K}$.

As in the case of canonical SCMs, the categorical formulation of MCMs does not rely directly on the structural assignments. 
Instead, it relies on the induced mixing functions that map exogenous variables to endogenous ones.

Consider an MCM as in Def.\nb\ref{def:mcm}. 
For each value $r\in[w]$ of the global DAG activation variable $W$, the associated DAG $\mathcal{D}_r$ specifies a causal ordering over the endogenous variables. 
Moreover, for each configuration $\bm h=(h_1,\ldots,h_u) $ of the local activation variables, the corresponding collection of structural assignments 
$\{f_{r,i,h_j}\}_i$ determines the functional form of the causal mechanisms.
By acyclicity of $\mathcal{D}_r$, these structural assignments can be recursively composed, yielding a measurable mixing function $\mathcal{M}_{r,\bm h}$ that expresses the endogenous variables as functions of the exogenous ones alone, conditional on the activation configuration $(r,\bm h)$.
Each such mixing function induces a probability distribution on the endogenous space via pushforward of the exogenous distribution $\zeta$, namely $\measure^{r,\bm h}\coloneqq P_{\mathcal{M}_{r,\bm h}(Z)}$, where the pushforward distribution is defined by $P_{\mathcal{M}_{r,\bm h}(Z)}(\mathcal{O})=P_Z\left( \mathcal{M}_{r,\bm h}^{-1}(\mathcal{O})\right)$ for every measurable event $\mathcal{O}\subseteq\myendogenousvals$.

The latent activation variables do not select a single mixing function deterministically, but rather define a probability distribution over the family $\{\mathcal{M}_{r,\bm h}\}$. 
Consequently, the overall distribution of the endogenous variables is obtained as a convex combination of the distributions $\measure^{r,\bm h}$, with weights determined by the mixing distributions of the activation variables.
This perspective makes explicit that an MCM induces a convex family of pushforward distributions on the endogenous space.
Thus, we exploit this observation to formalize MCMs as functors from \Index to \CSprob, where convex combinations of probability distributions are internalized at the level of objects and morphisms.

\begin{definition}[Category-theoretic MCM]\label{def:mcm-fun}
    An MCM with a global and $u$ local activation variables, $W \in [w]$ and $L_j \in [k_j]$ with $j \in [u]$, respectively, is a functor $\mcm{}:\Index \to \CSprob$ defined as follows
    \begin{equation}
        \centering
        \begin{tikzpicture}[]
    
        \node (I) at (0, 2.25) {\Index};
        \node (P) at (3, 2.25) {\CSprob};
        
        \node[circle, draw, fill,inner sep=1pt] (A) at (0, 1.5) {};
        \node[circle, draw, fill,inner sep=1pt] (B) at (0, 0) {};
        \node (A1) at (-0.3, 1.5) {I};
        \node (B1) at (-0.3, 0) {$I^\prime$};
        \node (C) at (3, 1.5) {$(\mathcal{P}(\myexogenousvals), cc_\lambda)$};
        \node (D) at (3, 0) {$(\mathcal{P}(\myendogenousvals), cc_\lambda)$};
        \coordinate (AB) at (0,0.75);
        \coordinate (CD) at (3,0.75);
    
        \coordinate (A1shift) at ([yshift=-5pt]A);
        \draw[->,shorten >=2pt] (A1shift) -- node[left] {f\,} (B);
    
        \coordinate (Ishift) at ([xshift=10pt]I);
        \coordinate (Pshift) at ([xshift=-20pt]P);
        \draw[->, dashed, shorten <=2pt, shorten >=0pt] (A) to[out=45, in=160] (C);
        \draw[->, dashed, shorten <=2pt, shorten >=0pt] (B) to[out=45, in=160] (D);
        \draw[->, dashed, shorten <=2pt, shorten >=2pt] (AB) to[out=45, in=160] (CD);
        \draw[->] (Ishift) -- node[above] {\mcm{}} (Pshift);
        \draw[->] (C) -- node[right] {\mcm{f}} (D);
        \end{tikzpicture}
        \label{fig:mcmfunctor}
    \end{equation}
    Specifically, the sample space \myexogenousvals of the exogenous is endowed with the product probability distribution $\zeta \in \mathcal{P}(\mathcal{U})$, and the affine measurable map \mcm{f} yields a pushforward mixture probability distribution $\measure \in \mathcal{P}(\mathcal{V})$ over the endogenous, viz.
    \begin{equation}\label{eq:morphism-mcm}
        P^{\mcm{f}(Z)} = cc_{\mathbf{p}}\left(\measure^{r\in[w],\bm h \in [k_1]\times \ldots\times[k_j]}\right)\,,
    \end{equation}
\end{definition}
where the entries of $\bf p \in \Delta_{w\prod_jk_j}$ are $p_{r, \bm h}=w_r\prod_j^u \upsilon_{j,h_j}$ and $\measure^{r, \bm h}=P_{\mathcal{M}_{r,\bm h}(Z)}$.

\begin{remark}
In the categorical Def.\nb\ref{def:mcm-fun}, the latent activation variables $W$ and $L_j$, $j \in [u]$, index a finite family of mixing functions whose induced pushforward distributions are combined convexly in the codomain \CSprob.
\end{remark}

Similarly to the SCM case, we can organize MCMs in Def.\nb\ref{def:mcm-fun} in a functor category, termed \MCMcat, where morphisms are natural transformations.

\begin{definition}[Category of MCMs]\label{def:mcm-cat}
    The category of MCMs, namely \MCMcat, consists of (i) functors $\mcm{}:\Index \to \CSprob$ as objects; and (ii) natural transformations $\eta:\mcm{}\to \mcm{\prime}$ as morphisms, such that:
    \begin{squishlist}
        \item for each $I$ in \Index, an affine measurable map $\eta_I: \mcm{}(I)\to \mcm{\prime}(I)$ in \CSprob, called component at $I$;
        \item for the unique morphism $f: I \to I^\prime$ in \Index, the following diagram commutes:
            \begin{equation}\label{eq:nat_transf_MCMcat}
                \begin{tikzcd}[row sep=1.5cm, column sep=1.5cm]
                    \mcm{}(I) \arrow[r, "\mcm{f}"] \arrow[d, "\eta_I"'] & \mcm{}(I^\prime) \arrow[d, "\eta_{I^\prime}"] \\
                    \mcm{\prime}(I) \arrow[r, "\mcm{\prime^f}"'] & \mcm{\prime}({I^\prime)}
                \end{tikzcd}
                \begin{tikzpicture}[overlay]
                    \draw[dashed, rounded corners] (-5.3,-1.5) rectangle (-3.5,1.5);
                    \draw[dashed, rounded corners] (-1.9,-1.5) rectangle (0,1.5);
                    \node at (-4.4,-1.8){Exogenous };
                    \node at (-1,-1.8){Endogenous };
                \end{tikzpicture}
            \end{equation} 
    \end{squishlist}
    \medskip
\end{definition}
Additionally, in \MCMcat in Def.\nb\ref{def:mcm-cat}, the categorical identity is defined as $\catidentity{\mcm{}}\coloneqq\langle \catidentity{(\mathcal{P}(\myexogenousvals), cc_\lambda)}, \catidentity{(\mathcal{P}(\myendogenousvals), cc_\lambda)}\rangle$; while the composition is defined component-wise, that is, provided $\eta: \mcm{} \to \mcm{\prime}$ and $\eta^\prime: \mcm{\prime} \to \mcm{''}$, we have $\eta^{''} = \eta^\prime \circ \eta \coloneqq \langle \eta_{I}^\prime \circ \eta_{I}, \eta_{I^\prime}^\prime \circ \eta_{I^\prime} \rangle$ from \mcm{} to \mcm{''}.

It is worth noticing that \scm{}, in Def.\nb2 in \citep{d2025relativity}, is a special case of \mcm{} in Def.\nb\ref{def:mcm-fun} where we only have a single index $(r, \bm h)$, thus implying a deterministic morphism $\mcm{f}=\mathcal{M}_{r,\bm h}=\mathcal{M}$.
This highlights that the category \SCMcat in \citep{d2025relativity} is a \emph{full subcategory} of \MCMcat, that is, \emph{(i)} all objects in \SCMcat are also in \MCMcat (while the converse does not hold), and \emph{(ii)} given \scm{} and \scm{\prime} in \SCMcat, all the arrows between them in \MCMcat are also in \SCMcat.

The network sheaf and cosheaf of CK remain well defined when canonical SCMs are replaced by MCMs.
Indeed, MCMs admit a categorical semantics in \CSprob via affine measurable generative maps, and the morphisms between MCMs are represented by commutative diagrams in the same category.
Since the sheaf and cosheaf constructions depend solely on this functorial encoding, no modification is required when passing from SCMs to MCMs.
% !TEX root =  ../main.tex
\section{The Causal Abstraction Network}\label{sec:CAN}

This section introduces CAN.
For clarity, we outline below the main steps:

\begin{squishlist}
    \item \Cref{subsec:ct_formalization} formalizes CAN using category theory, starting from the abstraction and embedding functors, and defines its graphical representation \CAN;
    \item \Cref{subsec:algebraic_descriptors} derives the algebraic descriptors of \CAN;
    \item \Cref{subsec:CAN_as_ns} reinterprets \CAN as a network sheaf of CK, enabling the analysis of CK flow;
    \item \Cref{subsec:CAN_consistency} investigates when \CAN is \emph{consistent}, linking consistency to SEP and analyzing spectral properties;
    \item \Cref{subsec:global_L} characterizes the diffusion of CK over CANs, including the existence and statistical properties of its global sections; 
    \item \Cref{subsec:convergence} proves convergence of the diffusion process to the space of global sections of a consistent CAN;
    \item \Cref{subsec:hyperparams} discusses the role of the hyper-parameters that shape the CK diffusion;
    \item \Cref{subsec:smoothness} introduces a notion of smoothness w.r.t. a CAN, quantifying it via information-theoretic metrics or $\phi$-divergences.
\end{squishlist}

To facilitate the understanding of concepts that may initially seem very abstract, we provide worked examples throughout the section.

\begin{figure*}[t]
    \centering
    \subfloat[Causal abstraction network\label{subfig:CAN}]{
        \begin{tikzpicture}[scale=0.97, transform shape]
            \node[draw, align=left] (R) at (3.5,3) {Causal abstraction relations:\\
            $1 \leqslant 2$, $2 \leqslant 3$, $2 \leqslant 4$}; 
            \node[draw, circle] (A) at (0,2) {$\mathsf{\scriptstyle  MCM}_1$};
            \node[draw, circle] (B) at (1,0) {$\mathsf{\scriptstyle  MCM}_2$};
            \node[draw, circle] (C) at (3,0) {$\mathsf{\scriptstyle  MCM}_3$};
            \node[draw, circle] (D) at (-1,-1) {$\mathsf{\scriptstyle  MCM}_4$};

            \draw[->,Mulberry] (A) to[bend right] node[midway, left]{$\V_{12}^\top$} (B);
            \draw[->,Periwinkle] (B) to[bend right] node[midway, right]{$\V_{12}$} (A);
            \draw[-] (A) to (B);
            \draw[->,Mulberry] (B) to[bend right] node[midway, below]{$\V_{23}^\top$} (C);
            \draw[->,Periwinkle] (C) to[bend right] node[midway, above]{$\V_{23}$} (B);
            \draw[-] (B) to (C);
            \draw[->,Mulberry] (B) to[bend right] node[midway, left]{$\V_{24}^\top\quad$} (D);
            \draw[->,Periwinkle] (D) to[bend right] node[midway, right]{$\quad \V_{24}$} (B);
            \draw[-] (B) to (D);
        \end{tikzpicture}
    }
    \hfill
    \subfloat[Network (co)sheaf representation\label{subfig:CANns}]{
        \begin{tikzpicture}
            \node[draw, circle] (A) at (0,3) {$\mathsf{\scriptstyle  MCM}_1$};
            \node[draw, circle] (B) at (1,0) {$\mathsf{\scriptstyle  MCM}_2$};
            \node[draw, rectangle] (AB) at (.5,1.5) {$\mathsf{\scriptstyle  MCM}_{12}$};
            \node[draw, circle] (C) at (4,0) {$\mathsf{\scriptstyle  MCM}_3$};
            \node[draw, rectangle] (BC) at (2.5,0) {$\mathsf{\scriptstyle  MCM}_{23}$};
            \node[draw, circle] (D) at (-2,-1) {$\mathsf{\scriptstyle  MCM}_4$};
            \node[draw, rectangle] (BD) at (-.5,-.5) {$\mathsf{\scriptstyle  MCM}_{24}$};

            \draw[->,Mulberry] (AB) to[bend right] node[near start, left]{$\V_{12}^\top$} (B);
            \draw[->,Periwinkle] (B) to[bend right] node[near end, right]{$\V_{12}$} (AB);
            \draw[->,Mulberry] (AB) to[bend right] node[midway, right]{$\myP_{d_1}$} (A);
            \draw[->,Periwinkle] (A) to[bend right] node[midway, left]{$\myP_{d_1}^\top$} (AB);
            \draw[-] (AB) to (A);
            \draw[-] (AB) to (B);
            \draw[->,Mulberry] (BC) to[bend right] node[midway, below]{$\V_{23}^\top$} (C);
            \draw[->,Periwinkle] (C) to[bend right] node[midway, above]{$\V_{23}$} (BC);
            \draw[->,Periwinkle] (B) to[bend right] node[midway, below]{$\myP_{d_2}^\top$} (BC);
            \draw[->,Mulberry] (BC) to[bend right] node[midway, above]{$\myP_{d_2}$} (B);
            \draw[-] (BC) to (C);
            \draw[-] (BC) to (B);
            \draw[->,Mulberry] (BD) to[bend right] node[midway, above]{$\V_{24}^\top\quad$} (D);
            \draw[->,Periwinkle] (D) to[bend right] node[midway, below]{$\quad \V_{24}$} (BD);
            \draw[-] (BD) to (D);
            \draw[-] (BD) to (B);
            \draw[->,Periwinkle] (B) to[bend right] node[near end,above]{$\myP_{d_2}^\top$} (BD);
            \draw[->,Mulberry] (BD) to[bend right] node[midway,below]{$\myP_{d_2}$} (B);
        \end{tikzpicture}
    }
    \caption{(a) Causal abstraction network \CAN made of $4$ nodes and $3$ undirected edges given in black. 
        Blue arcs follow the network orientation, corresponding to the embedding direction, that is, the action of the functor $E$.
        Purple arcs follow the abstraction direction, that is, the action of the functor $A$.
        (b) Network sheaf (blue arcs) and cosheaf (purple arcs) representations corresponding to \CAN. 
        Each edge (co)stalk coincides--up to a permutation of causal variables--with the node (co)stalk of the finer causal model.}
    \label{fig:CAN}
\end{figure*}

\subsection{Category-theoretic formalization}\label{subsec:ct_formalization}
Let us consider a collection of $N$ subjective MCMs, denoted by \mcmcollection, modeling the same phenomenon at different levels of granularity.
For instance, different teams of neuroscientists might use different brain atlases to study effective connectivity \citep{friston2011functional}.
Each $\mcm{}_i$, $i \in [N]$, has a number $d_i$ of causal variables $X_i=(X_1, \ldots, X_{d_i})$--the regions of interest (ROIs) individuated by the atlas--and we assume it induces an observational Gaussian mixture with $I>0$ components in \GMM{I}{d_i}, viz. $\measure_i=P_{X_i} = cc_{\w_i}\left(P^{s\in[I]}_{X_i}\right)$ and $\measure_{i}^s=P^s_{X_i} \sim N(\bm\mu_{i}^s,\bm\Sigma_{i}^s)$--for example, the joint probability distribution of the signals from ROIs recorded from $I$ brain states.
We denote by $\x_i \in \reall^{d_i}$ a sample from $\measure_i$, $\forall \, i \in [N]$.
Although we focus on Gaussian mixtures, the theory developed in this section is valid for general mixtures of probability distributions.
From \mcmcollection, we build a poset \Poset{[N]} defined by 
\begin{equation}\label{eq:poset_rel}
    i \leqslant j \;\iff\; \text{$\mcm{}_j$ is a CLCA of $\mcm{}_i$.}
\end{equation}
\Cref{eq:poset_rel} can be read as \emph{\enquote{the CK of the high-level $\mcm{}_j$ can be reconstructed by abstracting that of the low-level $\mcm{}_i$ through a CLCA.}}
More precisely, representing the CLCA with a matrix $\V_{ji}$, \Cref{eq:poset_rel} implies that $\measure_j=cc_{\w_j}(\measure_{j}^{[J]})$ is a Gaussian mixture with \emph{(i)} the same weights--and thus number of components $J=I$--as in $\measure_i$, $\w_j=\w_i$; and \emph{(ii)} linearly transformed components $\measure_{j}^s\sim N(\V_{ji}\bm\mu_{i}^s, \V_{ji} \bm\Sigma_{i}^s \V_{ji}^\top)$.
\begin{remark}\label{rem:general-clca}
    In principle, more flexible forms of linear causal abstraction between Gaussian mixtures could be formalized.
    For example, it would be interesting to investigate the case in which, in addition to the causal variables of the MCMs, the components of the low-level mixture are also combined during the abstraction process, thus allowing the high-level MCM to have a different number of components and weights from those of the low-level MCM. 
    This would lead to a specific definition of causal abstraction for mixtures where, in addition to the structural and functional maps between the causal variables of the MCMs, there would be an $I \times J$ coupling matrix $\bm\Omega$ between components, with $I\neq J$ in general.
    Specifically, by viewing a finite Gaussian mixture in \GMM{I}{d_i} as a discrete probability distribution $\w_i$ over the space of Gaussians $\mathcal{N}(\reall^{d_i})$, thus $\bm\Omega$ is equivalent to a transport plan between discrete distributions $\w_i \in \simplex{I}$ and $\w_j \in \simplex{J}$ belonging to the set of nonnegative matrices--joint distributions--with marginals $\w_i$ and $\w_j$, that is
    \begin{equation}\label{eq:coupling-matrix}
        \bm\Omega \in \Pi(\w_i,\w_j) \coloneqq\{\bm\Omega \in \Delta_{I\times J} \mid \bm\Omega \ones_J=\w_i \text{ and } \bm\Omega^\top \ones_I=\w_j\}\,.
    \end{equation}
\end{remark}
Notably, in the proposed sheaf-theoretic framework, CA relationships can be either inferred from given causal models (\emph{deductive approach}) or enforced through the network sheaf--i.e., its topology and associated maps--to construct causal AI systems with desired properties (\emph{inductive approach}).
Additionally, high-level CK can also be embedded into the low-level. 
Thus, to model the dynamics of CK over the network sheaf, we need a single \emph{functorial} object that captures both abstraction and embedding.
Here, the term \enquote{functorial} is used in its categorical meaning: the investigated object is an \emph{arrow from the source to the target categories}.

Toward our goal, we adopt the general setting in \cite{d2025causal}, specifically without assuming: 
\emph{(i)} complete specification of the causal models;
\emph{(ii)} knowledge of their causal DAGs;
\emph{(iii)} their functional forms;
\emph{(iv)} availability of interventional data; or
\emph{(v)} jointly sampled observational data. 
Thus, we work at the CK level.

Let us start from the source category.
Each poset can be read as a category having at most one arrow between any two distinct objects.
Furthermore, the properties of \Poset{[N]} above naturally follow: 
\emph{reflexivity} from the identity morphism \catidentity{i};
\emph{antisymmetry} from the fact that CLCA is oriented from low- to high-level;
and \emph{transitivity} from the compositionality of CLCAs, as formalized in Lem.\nb\ref{lemma:composedCA}.

\begin{restatable}{lemma}{ComposedCALemma}\label{lemma:composedCA}
    Given two CLCAs from $\mcm{}_i$ to $\mcm{}_j$ and from $\mcm{}_j$ to $\mcm{}_k$, their composition is a valid CLCA.    
\end{restatable}
\begin{proof}
    See App. \ref{app:proofs}.
\end{proof}

For parsimony, we consider the \emph{transitive reduction} of $\Poset{[N]}$, keeping only \emph{irreducible} relations, i.e., $i \leqslant j$ such that no $m$ exists with $i \leqslant m \leqslant j$.
For simplicity, we denote the transitive reduction again by $\Poset{[N]}$.

Viewing \Poset{[N]} as a category, its diagrammatic representation is a directed graph with vertices corresponding to $\mcm{}_i \in \mcmcollection$ and arcs $i \rightarrow j$ representing irreducible CLCA relations $i \leqslant j$.
The dual construction $\Poset{[N]}^{\mathrm{op}}$ has the same vertices with reversed arcs. 
The relations $i \geqslant j$ in the dual can be read as \emph{\enquote{~$\mcm{}_i$ is a fine-grained representation of $\mcm{}_j$}}.
Thus, we let \Poset{[N]} and $\Poset{[N]}^{\mathrm{op}}$ be the source categories for abstraction and embedding dynamics, respectively.
We remark that duality between these categories implies that a relation in $\Poset{[N]}^{\mathrm{op}}$ exists iff a corresponding CLCA exists in \Poset{[N]}.

Moving to the target category, according to \Cref{sec:categorical-CMM}, we model CK as an object in \CSprob.
Since we deal with CLCAs, relevant morphisms in \CSprob are linear maps, here represented as matrices.
We then attach the CK of the MCMs and CLCAs to the poset using the following \emph{abstraction functor}.

\begin{definition}[Abstraction functor]\label{def:abstraction_functor}
    The abstraction functor $A: \Poset{[N]} \rightarrow \CSprob$ is as follows:
    (i) applied to objects, $A(\mcm{}_i) = \text{CK}(\mcm{}_i)$ for each $\mcm{}_i \in \mcmcollection$; 
    (ii) applied to arrows, $A(i \leqslant j) = \V_{j i}$ for each $i \leqslant j$, the CLCA from $\mcm{}_i$ to $\mcm{}_j$.
\end{definition}

The embedding action is formalized via the dual functor $E$.

\begin{definition}[Embedding functor]\label{def:embedding_functor}
    The embedding functor $E: \Poset{[N]}^{\mathrm{op}} \rightarrow \CSprob$ is as follows:
    (i) on objects, $E(\mcm{}_i) = \text{CK}(\mcm{}_i) $ for each $\mcm{}_i \in \mcmcollection$; 
    (ii) on arrows, $E(i \geqslant j) = \V_{i j} = \V_{j i}^\top$ for each $i \geqslant j$.
\end{definition}

We refer to $E$ as the dual of $A$ as it acts on the opposite category $\Poset{[N]}^{\mathrm{op}}$.
In the linear setting, this coincides with taking the transposes of CLCA matrices, giving the dual action on the associated convex spaces of probability distributions.
The transfer of CK among MCMs in \mcmcollection is determined by the complementary actions of $A$ and $E$.
Exploiting the above source categories, we then introduce the CAN for representing both functors simultaneously. 

\begin{definition}[Causal Abstraction Network]\label{def:CAN}
    Let \Poset{[N]} be the poset indexing a collection of mixture causal models $\mcmcollection = \{\mcm{}_1,\dots,\mcm{}_N\}$.
    The Causal Abstraction Network (CAN) is the pair
    $$\mathrm{CAN} \coloneqq (A,E)\,;$$
    where $A:\Poset{[N]}\to\CSprob$ is the abstraction functor in Def.\nb\ref{def:abstraction_functor} and $E:\Poset{[N]}^{\mathrm{op}}\to\CSprob$ is the embedding functor
    in Def.\nb\ref{def:embedding_functor},
    both acting on the same family of objects.
\end{definition}

\begin{definition}[Graphical representation of a CAN]\label{def:CAN_graph}
    Let $\mathrm{CAN}=(A,E)$ be a CAN indexed by the poset
    \Poset{[N]}.
    The CAN graphical representation \CAN consists of:
    \begin{squishlist}    
        \item an undirected graph $\graph=(\vertexset,\edgeset)$, where \emph{(i)}
        $\vertexset$ is in bijection with the objects in \mcmcollection and (ii) $e_{ij}: i \sim j \in\edgeset$
        whenever $i\leqslant j$ in $\Poset{[N]}$;
        \item a choice of orientation on each edge of \graph induced by the
        embedding functor $E$, assigning the direction $j \rightarrow i$ whenever $i\leqslant j$.
    \end{squishlist}
\end{definition}

\begin{remark}
    The combinatorial graph \graph of \CAN is undirected; the orientation is not part of the graph structure but is chosen to encode the direction of causal enrichment given by $E$.
    This makes the framework suitable for building more complex CK from simpler causal models. 
    Abstraction maps correspond to traversing the same edges with opposite orientation.
\end{remark}

\subsection{Algebraic descriptors}\label{subsec:algebraic_descriptors}
We then provide an algebraic representation of \CAN, enabled by the CLCAs.  
Let $n = \sum_i d_i$ and the MCMs in \mcmcollection be ordered by decreasing dimensionality.  
Hereinafter, we will assume w.l.o.g. the MCMs having different dimensionalities.

The adjacency matrix of \CAN is a block matrix \CANA, with blocks
\begin{equation}\label{eq:CAN_A}
    \CANA(i,j) \coloneqq \begin{cases}
        \V_{ij} \quad & \text{if } i \sim j \in \edgeset \text{ and } i<j,\\
        \V_{ji}^\top \quad & \text{if } i \sim j \in \edgeset \text{ and } i>j,\\
        \zeros_{d_i \times d_j} \quad & \text{otherwise.}
    \end{cases}
\end{equation}

By construction, \CANA is symmetric. 
The strictly upper triangular blocks represent embedding morphisms, while the strictly lower triangular blocks represent CLCAs.

The degree matrix \CAND of \CAN is block-diagonal with blocks 
\begin{equation}\label{eq:CAN_D}
\CAND(i) \coloneqq \underbrace{\sum_{i \sim j, i< j} \identity_{d_i}}_{\text{embedding edges}} + \underbrace{\sum_{i\sim j, i>j} \V_{ij}^\top\V_{ij}}_{\text{abstraction edges}}\,. 
\end{equation}

Given the set of oriented edges $e:i\sim j$, the node-edge incidence matrix \CANB has blocks
\begin{equation}\label{eq:CAN_B}
    \CANB(i,e) \coloneqq \begin{cases}
        \identity_{d_i} \quad & \text{if $i$ is the head of $e$ ,}\\
        -\V_{ij}^\top \quad & \text{if $i$ is the tail of $e$ ,}\\
        \zeros_{d_j \times d_i} \quad & \text{otherwise.}
    \end{cases}
\end{equation}

Finally, the connection Laplacian of \CAN reads
\begin{equation}\label{eq:CAN_L}
    \CANL= \CAND - \CANA = \CANB \CANB^\top\,.
\end{equation}

\begin{example}\label{ex:alginv}
    Consider \CAN in \Cref{subfig:CAN}.
    Its algebraic descriptors are
    \begin{equation}
        \begin{aligned}
            \CANA &= 
            \begin{pmatrix}
            \zeros & \V_{12} & \zeros & \zeros \\ 
            \V_{12}^\top & \zeros & \V_{23} & \V_{24} \\
            \zeros & \V_{23}^\top & \zeros & \zeros \\
            \zeros & \V_{24}^\top & \zeros & \zeros
            \end{pmatrix}\, , \\
            \CAND &=
            \begin{pmatrix}
                \identity_{d_1} & \zeros & \zeros & \zeros\\
                \zeros & 2\identity_{d_2} + \V_{12}^\top \V_{12} & \zeros & \zeros \\
                \zeros & \zeros & \V_{23}^\top\V_{23} & \zeros\\
                \zeros & \zeros & \zeros & \V_{24}^\top\V_{24}
            \end{pmatrix}\,,\\
            \CANB &= 
            \begin{pmatrix}
                \identity_{d_1} & \zeros & \zeros \\
                - \V_{12}^\top & \identity_{d_2} & \identity_{d_2}\\
                \zeros & -\V_{23}^\top & \zeros \\
                \zeros & \zeros & -\V_{24}^\top 
            \end{pmatrix}\,,\\
            \CANL &= 
            \begin{pmatrix}
            \identity_{d_1} & -\V_{12} & \zeros & \zeros \\ 
            -\V_{12}^\top & 2\identity_{d_2} + \V_{12}^\top \V_{12} & -\V_{23} & -\V_{24} \\
            \zeros & -\V_{23}^\top & \V_{23}^\top\V_{23} & \zeros \\
            \zeros & -\V_{24 }^\top & \zeros & \V_{24}^\top\V_{24}
            \end{pmatrix}\, . \\
        \end{aligned}
    \end{equation}
\end{example}

\subsection{Relation with network sheaf and cosheaf of CK}\label{subsec:CAN_as_ns}
In the previous section we showed that the abstraction and embedding functors $A$ and $E$ in
Defs.~\ref{def:abstraction_functor} and \ref{def:embedding_functor}
define a CAN indexed by the poset $\Poset{[N]}$.
This structure naturally induces a specific family of network sheaves and cosheaves of CK.
Let us recall their formal definitions from \cite{d2025relativity}.

\begin{definition}[Network sheaf of CK]\label{def:net_sheaf_csprob}
    Given a network $G\coloneqq(\vertexset, \mathcal{E})$ with face incidence poset $\faceincidenceposet$, a network sheaf valued in \CSprob is a functor $F: \faceincidenceposet \rightarrow \CSprob$ assigning (i) to each node $\rho$ and edge $\tau$ in \faceincidenceposet stalks consisting of convex spaces of probability distributions, $\text{CK}(\scm{}_\rho)$ and $\text{CK}(\scm{}_\tau)$, respectively; (ii) to each node-edge incidence relation an affine restriction map $F^{\rho\trianglelefteq\tau}: \text{CK}(\scm{}_\rho) \rightarrow \text{CK}(\scm{}_\tau)$.% being affine.  
\end{definition}

\begin{definition}[Network cosheaf of CK]\label{def:net_cosheaf_csprob}
    Given a network $G\coloneqq(\vertexset, \mathcal{E})$ with face incidence poset $\faceincidenceposet$, a network cosheaf valued in \CSprob is a functor $\widehat{F}: \faceincidenceposet^{\mathrm{op}} \rightarrow \CSprob$ assigning (i) to each node $\rho$ and edge $\tau$ in $\faceincidenceposet^{\mathrm{op}}$ costalks consisting of convex spaces of probability distributions, $\text{CK}(\scm{}_\rho)$ and $\text{CK}(\scm{}_\tau)$, respectively; (ii) to each node-edge incidence relation an affine extension map $\widehat{F}^{\rho\trianglelefteq\tau}: \text{CK}(\scm{}_\tau) \rightarrow \text{CK}(\scm{}_\rho)$. 
\end{definition}

Consider the combinatorial graph \graph associated with \CAN, and let \faceincidenceposet denote its face incidence poset.
For each edge $e_{ij}\in\edgeset$ with $i\leqslant j$, define the restriction maps by
\begin{equation}\label{eq:ns_restriction_can}
    F^{i \trianglelefteq e_{ij}}=\identity_{d_i},
    \qquad
    F^{j \trianglelefteq e_{ij}}=\V_{ij}\,;    
\end{equation}
and the corresponding extension maps as their transposes.
With this choice, the edge (co)stalk $F(e_{ij})$ can be canonically identified, up to a permutation of causal variables, with the node (co)stalk $F(i)$ associated with the lower-level MCM, viz. $\mcm{}_i$.
\Cref{subfig:CANns} illustrates the resulting network sheaf and cosheaf
representations corresponding to the CAN in \Cref{subfig:CAN}.
With a slight abuse of notation, hereafter we will denote by $E$ and $A$ the network sheaf and cosheaf induced by the embedding and abstraction functors. 

The latter representations are convenient for studying the \emph{nonlinear} action of the restriction maps in the target category \CSprob.
For each edge $e:i \sim j \in \edgeset$, denote by $\alpha_e \in [0,1]$ an \emph{aggregation parameter}.
Furthermore, for each node $i \in \vertexset$, denote by $\bm\beta_i \in \simplex{|\starr{i}|}$ an \emph{assimilation parameter vector}. 
Here, \starr{i} denotes the star of node $i$, that is, the set of edges incident to $i$.
The role of $\alpha_e$ and $\bm\beta_i$ hyper-parameters is investigated in \Cref{subsec:hyperparams}.
Let \cochainspace{k}{\graph}{E} be the space of $k$-th order cochains for the network sheaf $E$.
For a $0$-cochain $\chi^0 \in \cochainspace{0}{\graph}{E}$, we define the \emph{coboundary operator} $\delta: \cochainspace{0}{\graph}{E} \rightarrow \cochainspace{1}{\graph}{E}$ as the operator that transports probability distributions from nodes to edges of \graph, in terms of the rows of $\CANB^\top$:
\begin{equation}\label{eq:cobounding_op}
    \begin{aligned}
    \delta(\chi^0)\at{e} &\coloneqq \cc{\alpha_e}{cc_{\w_i}\left(P^{[I]}_{\CANB^\top_{e,i}(X_{i})}\right)}{cc_{\w_j}\left(P^{[J]}_{-\CANB^\top_{e,j}(X_{j})}\right)}\\
    &= \cc{\alpha_e}{cc_{\w_i}\left(P^{[I]}_{X_{i}}\right)}{cc_{\w_j}\left(P^{[J]}_{\V_{ij}(X_{j})}\right)}= P_{X_i\at{e}} = \measure_e\,;
    \end{aligned}
\end{equation}
for all $e: i \sim j$ in \edgeset, $d_i > d_j$.
Thus, on the edge stalk the incoming probability distributions from the node stalks are \emph{aggregated} through the convex combination operation in \Cref{eq:ccl} with parameter $\alpha_e$, resulting in a mixture of Gaussians $\chi_e \in \cochainspace{1}{\graph}{E}$ over $X_i$.

The dual \emph{boundary operator} $\delta^\star: \cochainspace{1}{\graph}{E} \rightarrow \cochainspace{0}{\graph}{E}$, which transports probability distributions from edges to nodes of \graph, is defined via the rows of \CANB:
\begin{equation}\label{eq:bounding_op}
    \begin{aligned}
        \delta^\star(\measure^1)\at{i}  &= \underbrace{\sum_{e \in \starr{i}^{-}} \beta_{i,e}P_{\CANB_{i,e}(X_i\at{e})}}_{\text{embedding: $i$ is head of $e$}} + \underbrace{\sum_{e \in \starr{i}^{+}} \beta_{i,e}P_{-\CANB_{i,e}(X_j\at{e})}}_{\text{abstraction: $i$ is tail of $e$}}\\
        &=\sum_{e \in \starr{i}^{-}} \beta_{i,e}P_{X_i\at{e}} + \sum_{e \in \starr{i}^{+}} \beta_{i,e}P_{\V_{ij}^\top(X_j\at{e})} = \hat{\chi}_i
    \end{aligned}
\end{equation}
for each $i \in \vertexset$.
In this case, each node $i$ \emph{assimilates} the incoming probability Gaussian mixtures from its star \starr{i}, possibly giving anisotropic weights to the contributions.
Again, $\hat{\chi}_i$ is a mixture of Gaussians. 
Please note that the convex combination is iterated and order-independent due to the parametric associativity of $cc_\lambda$ (see Sec.\nb3 in \cite{fritz2009convex}).

Now, starting from \Cref{eq:cobounding_op,eq:bounding_op}, given $\chi =\{\measure_1, \ldots, \measure_n\}$ in $\cochainspace{0}{\graph}{E}$, the CAN Laplacian operator $\mathcal{L}: \cochainspace{0}{\graph}{E} \rightarrow \cochainspace{0}{\graph}{E}$ is
\begin{equation}\label{eq:Laplacian_op}
    \mathcal{L}(\chi) \coloneqq \delta^\star \circ \delta(\chi).
\end{equation}
This operator produces a diffusion of CK over the network, where the local contribution at node $i$ is equal to
\begin{equation}\label{eq:Laplacian_op_local}
    \begin{aligned}
        \mathcal{L}(\chi)\at{i} &= \underbrace{\sum_{e \in \starr{i}^{-}} \beta_{i,e} cc_{\alpha_e}\left(cc_{\w_i}\left(P_{X_i}^{[I]}\right), cc_{\w_j}\left(P_{\V_{ij}(X_j)}^{[J]}\right)\right)}_{\text{embedding: $i$ is head of $e$}} + \\
        &+ \underbrace{\sum_{e \in \starr{i}^{+}} \beta_{i,e} cc_{\alpha_e}\left(cc_{\w_j}\left(P_{\V_{ji}^\top(X_j)}^{[J]}\right),cc_{\w_i}\left(P_{\V_{ji}^\top\circ \V_{ji}(X_i)}^{[I]}\right)\right)}_{\text{abstraction: $i$ is tail of $e$}} \,; \quad \forall \, i \in \vertexset. 
    \end{aligned}
\end{equation}

Hence, from the node perspective, the diffusion of CK via the CAN Laplacian corresponds to mixtures of Gaussians resulting from the combined action of $\delta$ and $\delta^\star$.
Thus, starting from local Gaussian mixtures $\measure_i=cc_{\w_i}\left(P_{X_i}^{[I]}\right)$ at each node $i \in [N]$, representing subjective CK, the application of \Cref{eq:Laplacian_op} yields an updated CK obtained through diffusion and combination of the initial CK across the structure of \CAN.

\begin{remark}
    The above operators are also valid for general network sheaves of CK, i.e., when the restriction map $F^{i \trianglelefteq e_{ij}}$ in \Cref{eq:ns_restriction_can} is an embedding matrix different from the identity.
    Indeed, the operators are formulated in terms of the blocks of \CANB and their transposes, determined by the restriction maps.
    Thus, in the general case, the pushforward identity in \Cref{eq:cobounding_op,eq:bounding_op} is replaced with that induced by the embedding matrix $F^{i \trianglelefteq e_{ij}}$.
\end{remark}

\begin{remark}\label{rem:why-gmm}
    As repeatedly noted above, the action of diffusion operators produces Gaussian mixtures. 
    Importantly, this occurs even when the initial probability distributions at each node are themselves Gaussian, as in standard Gaussian SCMs. 
    Therefore, to develop a framework that is both general and widely applicable, it is crucial to consider the broader case of Gaussian mixtures arising from mixtures of SCMs.
\end{remark}

\begin{example}\label{ex:laplacian}
    Consider \CAN in \Cref{subfig:CAN} with algebraic descriptors in \Cref{ex:alginv}.
    Let the $0$-cochain be $\chi=\{\chi_1,\chi_2,\chi_3,\chi_4\}$ with $S_1, \ldots, S_4$ components and $\alpha_e \in [0,1]$, for each $e \in \edgeset$.
    The coboundary operator gives the $1$-cochain $\{\chi_{12}, \chi_{23}, \chi_{24}\}$, where
    \begin{equation}\label{eq:1cochainex}
        \delta(\chi)=\begin{cases}
            \chi_{12}=P_{X_1\at{12}}=cc_{\alpha_{12}}\left(cc_{\w_1}\left(P^{[S_1]}_{X_1}\right), cc_{\w_2}\left(P^{[S_2]}_{\V_{12}(X_2)}\right)\right)\,,\\
            \chi_{23}=P_{X_2\at{23}}=cc_{\alpha_{23}}\left(cc_{\w_2}\left(P^{[S_2]}_{X_2}\right),cc_{\w_3}\left(P^{[S_3]}_{\V_{23}(X_3)}\right)\right)\,,\\
            \chi_{24}=P_{X_2\at{24}}=cc_{\alpha_{24}}\left(cc_{\w_2}\left(P^{[S_2]}_{X_2}\right),cc_{\w_4}\left(P^{[S_4]}_{\V_{24}(X_4)}\right)\right)\,.
        \end{cases}    
    \end{equation}
    Let $\bm\beta_i \in \simplex{|\starr{i}|}$, for $i=1,\ldots,4$.
    Then, the application of \Cref{eq:Laplacian_op_local} gives the local contributions of the diffusion process based on the CAN Laplacian in \Cref{eq:Laplacian_op},
    \begin{equation}\label{eq:Laplacianex}
        \mathcal{L}(\chi) = \begin{cases}
            \underbrace{cc_{\alpha_{12}}\left(cc_{\w_1}\left(P^{[S_1]}_{X_1}\right), cc_{\w_2}\left(P^{[S_2]}_{\V_{12}(X_2)}\right)\right)}_{\text{embedding edge $e_{12}$}}\,,\\
            \underbrace{\beta_{2,1}cc_{\alpha_{12}}\left(cc_{\w_1}\left(P^{[S_1]}_{\V_{12}^\top(X_1)}\right), cc_{\w_2}\left(P^{[S_2]}_{\V_{12}^\top\circ\V_{12}(X_2)}\right)\right)}_{\text{abstraction edge $e_{12}$}} +\\
            \qquad\underbrace{\beta_{2,2} cc_{\alpha_{23}}\left(cc_{\w_2}\left(P^{[S_2]}_{X_2}\right),cc_{\w_3}\left(P^{[S_3]}_{\V_{23}(X_3)}\right)\right)}_{\text{embedding edge $e_{23}$}} +\\
            \qquad\underbrace{\beta_{2,3} cc_{\alpha_{24}}\left(cc_{\w_2}\left(P^{[S_2]}_{X_2}\right),cc_{\w_4}\left(P^{[S_4]}_{\V_{24}(X_4)}\right)\right)}_{\text{embedding edge $e_{24}$}}\,,\\
            \underbrace{cc_{\alpha_{23}}\left(cc_{\w_2}\left(P^{[S_2]}_{\V_{23}^\top(X_2)}\right),cc_{\w_3}\left(P^{[S_3]}_{\V_{23}^\top\circ\V_{23}(X_3)}\right)\right)}_{\text{abstraction edge $e_{23}$}}\,,\\
            \underbrace{cc_{\alpha_{24}}\left(cc_{\w_2}\left(P^{[S_2]}_{\V_{24}^\top(X_2)}\right),cc_{\w_4}\left(P^{[S_4]}_{\V_{24}\circ\V_{24}(X_4)}\right)\right)}_{\text{abstraction edge $e_{24}$}}\,.
        \end{cases}
    \end{equation}
    Looking at \cref{eq:Laplacianex}, we observe
    \begin{squishlist}
        \item Node $1$ (finest) has only the embedding term, a convex combination of $\chi_1$ and the pushforward of $\chi_2$ via $\V_{12}$.
        \item Nodes $3$ and $4$ (coarser) have only abstraction terms, convex combinations of the abstractions of $\measure_2$ with self-contributions.
        \item Node $2$ (middle) combines embedding and abstraction, yielding three terms: embedding followed by abstraction for $e_{12}$, and embedding for $e_{23}$ and $e_{24}$.
    \end{squishlist}
\end{example}

\subsection{Consistency of \CAN}\label{subsec:CAN_consistency}
A key feature enabling the study of the spectral properties of \CAN, as well as those of the associated space of global sections, is \emph{consistency}. 
We say that \CAN is consistent if the composition of the restriction maps along every oriented loop $\mcm{}_j\stackrel{\text{embedding}}{\longrightarrow}\mcm{}_i\stackrel{\text{abstraction}}{\longrightarrow}\mcm{}_j$ commutes.
The next result links consistency to SEP, constraining CLCA transpose to  Stiefel manifold.

\begin{theorem}\label{th:consistencyCAN}
    Let \CAN be a connected CAN.
    Then, \CAN is consistent $\iff$ for every oriented edge $e_{ij} \in \edgeset$, $\V_{ij}$ is the right-inverse of a CLCA adhering to SEP.
\end{theorem}
\begin{proof}
    See App.\nb\ref{app:proofs}.
\end{proof}
We can further characterize the spectral properties of \CANL.
Recall that MCMs are sorted in decreasing dimensionality.

\begin{theorem}\label{th:kernelCAN}
    Consider a consistent and connected \CAN whose coarsest $\mcm{}_N$ has dimension $h$.
    Then, \CANL has null eigenvalue $\mu_0$ with multiplicity $m(\mu_0)=h$ $\iff$ every $\mcm{}_i$ is reachable from $\mcm{}_{N}$ via an oriented path.
\end{theorem}
\begin{proof}
    See App.\nb\ref{app:proofs}.
\end{proof}

The proof of \Cref{th:kernelCAN} relies on the observation that the presence of nodes not reachable from the coarsest one inevitably reduces the maximum possible dimension of the kernel of \CANL, which is $h$.
Consider, for instance, a path $u \sim m \sim N$, with $d_u > h$.
If the intersection $\spn{\V_{mu}} \cap \spn{\V_{mN}}$ had dimension $h = \dim\left(\spn{\V_{mN}}\right)$, then $\spn{\V_{mN}} \subseteq \spn{\V_{mu}}$.
By Lem.\nb\ref{lemma:intersection_maps}, this would imply the existence of a matrix $\V_{uN} \in \stiefel{d_u}{h}$ such that $\V_{mN} = \V_{mu} \V_{uN}$, thus inducing a morphism from $N$ to $u$, which contradicts the assumption that $u$ is not reachable from $N$.
Therefore, the intersection must have a dimension strictly lower than $h$.

Finally, we remark that $\ker(\CANL)$ may still be nonempty even in the presence of nodes not reachable from $N$.
Indeed, its dimension depends on the intersection of the span of the morphisms associated with different intersecting paths.

\begin{example}\label{ex:consistent}
    Let \CAN in \Cref{subfig:CAN} be consistent, that is, $\V_{ij}\in \stiefel{d_i}{d_j}$ for all $e_{ij} \in \edgeset$.
    The connection Laplacian \CANL becomes
    \begin{equation}
        \CANL = 
            \begin{pmatrix}
            \identity_{d_1} & -\V_{12} & \zeros & \zeros \\ 
            -\V_{12}^\top & 3\identity_{d_2} & -\V_{23} & -\V_{24} \\
            \zeros & -\V_{23}^\top & \identity_{d_3} & \zeros \\
            \zeros & -\V_{24 }^\top & \zeros & \identity_{d_4}
            \end{pmatrix}\, .
    \end{equation}
    Then, \Cref{eq:Laplacianex} simplifies to
    \begin{equation}\label{eq:Laplacian_consistent_ex}
        \mathcal{L}(\chi) = \begin{cases}
            cc_{\alpha_{12}}\left(cc_{\w_1}\left(P^{[S_1]}_{X_1}\right), cc_{\w_2}\left(P^{[S_2]}_{\V_{12}(X_2)}\right)\right)\,,\\
            \beta_{2,1}cc_{\alpha_{12}}\left(cc_{\w_1}\left(P^{[S_1]}_{\V_{12}^\top(X_1)}\right), cc_{\w_2}\left(P^{[S_2]}_{X_2}\right)\right) +\\
            \qquad\beta_{2,2} cc_{\alpha_{23}}\left(cc_{\w_2}\left(P^{[S_2]}_{X_2}\right),cc_{\w_3}\left(P^{[S_3]}_{\V_{23}(X_3)}\right)\right) +\\
            \qquad\beta_{2,3} cc_{\alpha_{24}}\left(cc_{\w_2}\left(P^{[S_2]}_{X_2}\right),cc_{\w_4}\left(P^{[S_4]}_{\V_{24}(X_4)}\right)\right)\,,\\
            cc_{\alpha_{23}}\left(cc_{\w_2}\left(P^{[S_2]}_{\V_{23}^\top(X_2)}\right),cc_{\w_3}\left(P^{[S_3]}_{X_3}\right)\right)\,,\\
            cc_{\alpha_{24}}\left(cc_{\w_2}\left(P^{[S_2]}_{\V_{24}^\top(X_2)}\right),cc_{\w_4}\left(P^{[S_4]}_{X_4}\right)\right)\,.
        \end{cases}
    \end{equation}
    Additionally, \CANL has not null eigenvalue with multiplicity $d_4$ as node $3$ is not reachable from node $4$ with an oriented path.
    Furthermore, $\ker{\CANL}\neq \emptyset$ if $\im{\Pi_{23}} \cap \im{\Pi_{24}}\neq \emptyset$, with $\Pi_{23}=\V_{23}\V_{23}^\top$ and $\Pi_{24}=\V_{24}\V_{24}^\top$ being the projectors induced by $\V_{23}$ and $\V_{24}$.
\end{example}

\subsection{Diffusion of causal knowledge and global section of $\mathcal{L}$}\label{subsec:global_L}
A key property of a consistent \CAN is the existence of a well-defined global section for the induced network sheaf. 
Formally, a global section is such that 
\begin{equation}\label{eq:global_sec_conditions}
    \measure_i = P_{\V_{ij}(X_j)} = cc_{\w_j}\left(P_{\V_{ij}(X_j)}^{[J]}\right)\, \quad \text{for all } e_{i j} \in \edgeset\,.
\end{equation}
Here, we generalize this notion using the CAN Laplacian operator $\mathcal{L}$ of \CAN and $\cc{\lambda}{\cdot}{\cdot}$ in \Cref{eq:ccl}, following analogous constructions for network sheaves of vector spaces \citep{hansen2019toward} and lattices \citep{ghrist2022cellular}. 
In those cases, the global section arises as the asymptotic limit of a suitable discrete-time dynamics on $0$-cochains.

Since \CSprob is non-Abelian, we define a discrete-time system using the convex combination operator, 
\begin{equation}\label{eq:dynsistCSprob}
    \chi(t+1) = \cc{\bm\lambda}{\chi(t)}{\mathcal{L}(\chi(t))}\,.
\end{equation}
Here, $\bm\lambda \in [0,1]^N$ is an \emph{update parameter vector} that plays a role analogous to $\alpha_e$ and $\bm\beta_i$ but between consecutive time steps.
Indeed, for each node $i \in [N]$, the parameter $\lambda_i \in [0,1]$ shapes the update process of the subjective CK $\measure_i$, since it determines the weight given to the incoming CK after a diffusion step, that is, $\mathcal{L}(\measure)\at{i}$ in \Cref{eq:Laplacian_op_local} (cf. \cref{subsec:hyperparams}).

As shown in \Cref{th:globalsectionCAN}, the space \gsspace{\graph}{E} of global sections of \CAN coincides with the fixed points of the operator in \Cref{eq:dynsistCSprob}, denoted by $\mathrm{Fix}\left(\cc{\bm\lambda}{\catidentity{}}{\mathcal{L}}\right)$.

\begin{lemma}\label{lemma:fixequiv}
    Let $\chi  \in \mathrm{Fix}\left(\cc{\bm\lambda}{\catidentity{}}{\mathcal{L}}\right)$, with $\bm\lambda \in [0,1]^N$. 
    Then, $\chi=\mathcal{L}(\chi)$.
\end{lemma}
\begin{proof}
    See App.\nb\ref{app:proofs}. 
\end{proof}

\begin{theorem}\label{th:globalsectionCAN}
    Let \CAN be a consistent CAN and $\mathcal{L}: \cochainspace{0}{\graph}{E}\rightarrow \cochainspace{0}{\graph}{E}$ the corresponding Laplacian operator.
    Then,
    \begin{equation}
        \mathrm{Fix}\left(\cc{\bm\lambda}{\catidentity{}}{\mathcal{L}}\right) = \gsspace{\graph}{E} \,, \quad \text{with } \bm\lambda \in [0,1]^N\,.
    \end{equation}
\end{theorem}
\begin{proof} 
    See App.\nb\ref{app:proofs}.
\end{proof}

Within the proof of \Cref{th:globalsectionCAN}, \Cref{eq:global_to_fix} emphasizes that the consistency of a CAN is central. 
However, consistency alone does not suffice to guarantee the existence of a global section; a nonempty kernel of $\CANL$ is also required, as formalized below.

\begin{theorem}\label{th:globalsec_existence}
    Consider a consistent and connected \CAN, and let $h$ be the dimension of the coarsest $\mcm{}_N$.
    Then,
    \begin{equation}
        \mathrm{Fix}\left(\cc{\bm\lambda}{\catidentity{}}{\mathcal{L}}\right) \neq \emptyset \iff 0< \dim \left( \ker{\CANL} \right)\leq h\,.
    \end{equation}
\end{theorem}
\begin{proof}
    See App.\nb\ref{app:proofs}.
\end{proof}

\begin{remark}
 Consider that the combinatorial \graph has $k$ connected components, and let each component $c \in [k]$ have a coarser probability distribution with dimensionality $h_c$.
    Then, applying \cref{th:globalsec_existence} component-wise, the existence of a global section implies that $k\leq \dim \left( \ker{\CANL} \right)\leq \sum_c h_c$, but not vice versa.
\end{remark}

\Cref{th:globalsec_existence} has immediate implications for the statistical properties of the elements of the global sections. 
Indeed, the proof in App.\nb\ref{app:proofs} strictly relates the existence of a global section with that of a $k$-dimensional latent vector space--with $0<k\leq h$--coherently embedded in each node stalk via the Stiefel matrices, such that for each node $i \in [N]$ the support of $\measure_i$ lies on it.  
Thus, the spectral features of the mixture covariances are preserved along the edges. 
The following lemma and corollary formalize this observation.

\begin{lemma}\label{lemma:sameeigenvalues}
    Consider $\covM_i$ and $\covM_j$ having sizes $d_i \!\times\! d_i$ and $d_j \!\times\! d_j$, respectively, with $\covM_i\!=\!\V_{ij}\covM_j \V_{ij}^\top$ and $\V_{ij} \!\in\! \stiefel{d_i}{d_j}$.
    Then, $\covM_i$ and $\covM_j$ have the same nonzero eigenvalues.
\end{lemma}
\begin{proof}
    See App.\nb\ref{app:proofs}
\end{proof}

\begin{corollary}\label{cor:same_spectrum}
    Consider $\chi \in \mathrm{Fix}\left(\cc{\bm\lambda}{\catidentity{}}{\mathcal{L}}\right)$.
    Then, component-wise, the covariance matrices of its elements $\measure_i \in \GMM{I}{d_i}$, $i \in [N]$, have the same $k\leq \min_i d_i=h$ nonzero eigenvalues, with the equality holding in case $m(\mu_0)=h$. 
\end{corollary}
\begin{proof}
    See App.\nb\ref{app:proofs}.
\end{proof}

Additionally, the spectral interlacing necessary condition in \Cref{eq:spectralCA} for the existence of CLCA can be generalized to the case where the low- and high-level distributions are Gaussian mixtures.
Specifically, it turns out that the spectral interlacing has to be checked on the mixture covariance.

\begin{corollary}\label{cor:spectral_interlacing_gmm}
    Let $\measurelow \in \GMM{S}{\ell}$, $\measurehigh \in \GMM{S}{h}$, and let \covlow and \covhigh be their mixture covariances (cf. \cref{eq:gmm_sigma}).
    Denote by $0\leq\lambda_1\leq \ldots \leq \lambda_\ell$ the eigenvalues of \covlow, and by $0\leq\kappa_1 \leq \ldots\leq \kappa_h$ those of \covhigh.
    If a CLCA complying with SEP from \measurelow to \measurehigh exists, then
    \begin{equation}
        \lambda_i \leq \kappa_i \leq \lambda_{i + \ell -h}, \quad \forall \,i \in [h]\,.
    \end{equation}
\end{corollary}
\begin{proof}
    See App.\nb\ref{app:proofs}.
\end{proof}

\subsection{Convergence to global sections}\label{subsec:convergence}
\cref{th:globalsec_existence} provides a necessary and sufficient condition for the existence of a global section.
However, it does not guarantee or provide conditions for the convergence of \cref{eq:dynsistCSprob} to $\mathrm{Fix}\left(\cc{\bm\lambda}{\catidentity{}}{\mathcal{L}}\right)$.
To investigate the convergence, let us rewrite \Cref{eq:dynsistCSprob} in a convenient way.
Specifically, denote by $\measure^{\uparrow}(t)$ a collection of $|\edgeset|$ \emph{embedding} pushforward distributions at time $t$, one for each edge $e \in \edgeset$.
In detail, for each $i\sim j$, corresponding to a CLCA $i\leqslant j$ and embedding from $j$ to $i$, the corresponding element in $\measure^{\uparrow}(t)$ is the pushforward distribution induced by the linear map $\V_{ij} \in \stiefel{d_i}{d_j}$, viz. $P_{\V_{ij}X_j}(t)$.
Similarly, denote by $\measure^{\downarrow}(t)$ a collection of $|\edgeset|$ \emph{abstraction} pushforward distributions at time $t$, one for each edge $e \in \edgeset$.
Here, the element of $\measure^{\downarrow}(t)$ corresponding to $i\sim j$ is $P_{\V_{ji}^\top X_i}(t)$.

Recall the aggregation, assimilation, and update hyper-parameters above.
For convenience, for each $e \in \starr{i}$ with either $e:i\sim j$ or $e:j\sim i$, let $\beta_{i,e}=\beta_{ij}$.
Starting from \Cref{eq:dynsistCSprob}, for the node $i$ with $|\starr{i}^-|$ embedding edges and $|\starr{i}^+|$ abstraction edges, we have
\begin{equation}\label{eq:nodedynamics}
    P_{X_i}(t+1) = a_{ii} P_{X_i}(t) + \sum_{e \in \starr{i}^{-}} b_{ij}P_{\V_{ij}X_j}(t) + \sum_{e \in \starr{i}^{+}} c_{ij}P_{\V_{ji}^\top X_j}(t)\,;
\end{equation}
where
\begin{equation}\label{eq:abc}
    \begin{aligned}
         a_{ii}&=\lambda_i + (1-\lambda_i)(\sum_{e \in \starr{i}^{-}} \beta_{ij}\alpha_{ij} + \sum_{e \in \star{i}^{-}}\beta_{ij}(1-\alpha_{ji}))\,;\\
         b_{ij}&=(1-\lambda_i)\beta_{ij}(1-\alpha_{ij})\,;\\
         c_{ij}&=(1-\lambda_i)\beta_{ij}(1-\alpha_{ji})\,.
    \end{aligned}
\end{equation}
Please note that, $b_{ij}$ and $c_{ij}$ are mutually exclusive, and the hyper-parameters in \Cref{eq:abc} satisfy
\begin{equation}\label{eq:cond-abc}
    a_{ii} + \sum_{e \in \starr{i}} (b_{ij}+c_{ij})=1, \quad \text{for each } i \in [N]. 
\end{equation}

Thus, using \Cref{eq:nodedynamics}, \Cref{eq:dynsistCSprob} can be rewritten as
\begin{equation}\label{eq:rewritten-dynamics}
    \measure(t+1) = \A \measure(t) + \B \measure^{\uparrow}(t) + \C \measure^{\downarrow}(t)\,;
\end{equation}
where \emph{(i)} $\A \in [0,1]^{N\times N}$ is diagonal with entries $a_{ii}$ as in \Cref{eq:abc}, \emph{(ii)} $\B \in [0,1]^{N\times |\edgeset|}$ strictly upper triangular with entries $b_{ij}$ as in \Cref{eq:abc} corresponding to embedding edges, and \emph{(iii)} $\C \in [0,1]^{N\times |\edgeset|}$  strictly lower triangular with entries $b_{ij}$ as in \Cref{eq:abc} corresponding to abstraction edges.
In addition, the matrix $\A+\B+\C$ is row-stochastic.

Let $\mathcal{M}(t)\!\coloneqq\!\{\bm\mu_1(t),\ldots,\bm\mu_N(t)\}$ be the collection of mixture mean vectors $\bm \mu_i(t) \in \reall^{d_i}$ (cf. \cref{eq:gmm_mu}), one for each Gaussian mixture $\measure_i$.
Furthermore, let $\mathcal{S}(t)\!\coloneqq\!\{\covM_1(t),\!\ldots\!,\covM_N(t)\}$ be the collection of mixture covariance matrices $\covM_i(t) \!\in\! \reall^{d_i \times d_i}$ (cf. \cref{eq:gmm_sigma}), one for each Gaussian mixture $\measure_i$.
Starting from \Cref{eq:rewritten-dynamics}, we can formalize the convergence of CAN dynamics.
\begin{theorem}\label{thm:convergence}
    Consider a consistent and connected \CAN and let $\mathrm{Fix}\left(\cc{\bm\lambda}{\catidentity{}}{\mathcal{L}}\right) \neq \emptyset$.
    Then, if $a_{ii}\neq0$ for each $i \in [N]$, \Cref{eq:rewritten-dynamics} converges to a global section $\measure^\star$ where:
    \begin{itemize}
        \item $\mathcal{M}^\star$ has elements 
        \begin{equation}\label{eq:mu-i-star}
            \bm\mu_i^\star=\vecu_i^i\,; 
        \end{equation}
        where $\vecu^i=[\vecu_1^{i^\top},\ldots,\vecu_N^{i^\top}]^\top \in \reall^d$, $d =\sum_i d_i$, is an eigenvector in the $k$-dimensional kernel of the normalized Laplacian $\ker{\widetilde{\CANL}}$, i.e.
    \begin{equation}\label{eq:normalized-Laplacian}
            \widetilde{\CANL}= \CAND^{-\frac{1}{2}}\CANL \CAND^{-\frac{1}{2}}\,;
        \end{equation}
        \item $\mathcal{S}^\star$ has elements 
        \begin{equation}\label{eq:sigma-i-star}
            \covM^\star_n \in \{\covM \in \reall^{d_n \times d_n} \mid \covM=\frac{1}{2}(\vecu^i_n\vecu_n^{j^\top}+\vecu_n^j\vecu_n^{i^\top});\, (i,j) \in [k]\times[k] \text{ and } i\leq j \}\,.
        \end{equation}
    \end{itemize}
\end{theorem}
\begin{proof}
    See App.\nb\ref{app:proofs}.
\end{proof}

\Cref{thm:convergence} can be thought of as a generalization of the classical linear consensus to probability distributions coupled through a sheaf-like structure.
Although the dynamics in \cref{eq:dynsistCSprob} is nonlinear at the distribution level, the moments evolution is governed by the spectrum of the normalized Laplacian $\widetilde{\CANL}$, in close analogy with averaging processes on graphs.

However, in our case the limit does not correspond to scalar consensus, but rather to a sort of \emph{subspace consensus}: the means $\mathcal{M}^\star$ in \Cref{eq:mu-i-star} lie in the subspace spanned by the eigenvectors associated with $\ker{\widetilde{\CANL}}$, whose dimension is $k = \dim(\ker{\CANL})$.
This latent subspace is shared across all nodes and is consistent with the structure of the CAN.

The covariance structure in \cref{eq:sigma-i-star} further reflects this property.
The $\frac{k(k+1)}{2}$ covariance matrices obtained are all supported on the same $k$-dimensional latent subspace. 
In particular: 
\emph{(i)} the $k$ covariances corresponding to $i=j$ are of the form $\vecu_n^i \vecu_n^{i^\top}$ and have rank $1$;
\emph{(ii)} the $\frac{k(k-1)}{2}$ covariances corresponding to $i \neq j$ are symmetric combinations of eigenvectors and, in general, have rank $k$, since each block $\vecu_n^i$ is obtained via embedding maps from a common $k$-dimensional latent space.
Consequently, all limiting covariances are supported on a common latent subspace of dimension $k$, providing a strong structural characterization of the consensus solution.

We finally conclude the section with a worked example showing the dynamics of CK over a consistent CAN made of three MCMs, and the resulting global section.

\begin{figure}[h]
    \centering
    \subfloat[$\scm{}_1$\label{subfig:M1}]{
        \begin{tikzpicture}
            \node[draw, circle] (A) at (0,2) {$A$};
            \node[draw, circle] (B) at (3,0) {$B$};
            \node[draw, circle] (C) at (1,0) {$C$};

            \draw[->] (A) to node[midway, left]{$-1.10$} (C);
            \draw[->] (B) to node[midway, above]{$-0.53$} (C);
        \end{tikzpicture}
    }
    \hfill
    \subfloat[$\scm{}_2$\label{subfig:M2}]{
        \begin{tikzpicture}
            \node[draw, circle] (B) at (3,0) {$B$};
            \node[draw, circle] (AC) at (1,0) {$\{A,C\}$};

            \draw[->] (B) to node[midway, above]{$0.17$} (AC);
        \end{tikzpicture}
    }
    \hfill
    \subfloat[$\scm{}_3$\label{subfig:M3}]{
        \begin{tikzpicture}
            \node[draw, circle] (A) at (0,2) {$A$};
            \node[draw, circle] (BC) at (1,0) {$\{B,C\}$};

            \draw[->] (A) to node[midway, left]{$0.43$} (BC);
        \end{tikzpicture}
    }
    \caption{Causal DAGs in \Cref{ex:dynamics}.}
    \label{fig:causal_DAGs}
\end{figure}

\begin{example}\label{ex:dynamics}
    We consider a collection $\scmcollection=\{\scm{}_1,\scm{}_2,\scm{}_3\}$ made of three linear SCMs with additive Gaussian noise, whose dimensionalities are $d_1=3$, $d_2=2$, and $d_3=2$, chosen for visualization purposes.
    The SCMs induce the causal DAGs in \cref{fig:causal_DAGs}, and they are related by the CLCA relations $\{1\leqslant2, 1\leqslant3\}$, where the abstraction matrices are
    \begin{equation}\label{eq:clca_matrices}
        \text{\emph{(a)}}\quad \V_{12}^\top=\begin{bmatrix}
            0 & 1 & 0\\
            0.83 & 0 & 0.56
        \end{bmatrix}\,, \text{ and \emph{(b)}}
        \quad \V_{13}^\top=\begin{bmatrix}
            1 & 0 & 0\\
            0 & 0.29 & 0.96
        \end{bmatrix}\,.
    \end{equation}
    Thus, at the beginning of the simulation ($t=0$), we have the CAN in the top-left panel of \Cref{fig:can_dynamics}, where we report for each node $1000$ points sampled from the corresponding component of the $0$-cochain $\measure(0)=\{\measure_1(0),\measure_2(0),\measure_3(0)\}$.
    Additionally, we derive the components $\measure_2(0)$ and $\measure_3(0)$ from $\measure_1(0)$ by applying the CLCAs in \Cref{eq:clca_matrices}.
    
    At this point, we let the CK evolve according to \Cref{eq:dynsistCSprob} for $T=100$ time steps, setting the CAN parameters as follows:
    \begin{equation}
        \begin{aligned}
            & \alpha_{12}=\alpha_{13}=0.5\,; \quad \text{(aggregation)}\\
            & \bm\beta_{1}=[0.5, 0.5]\,,\; \bm\beta_2=[1]\,,\; \bm\beta_3=[1]\,; \quad \text{(assimilation)}\\
            & \bm\lambda = [0.5, 0.5, 0.5]\, .\quad \text{(update)}
        \end{aligned}
    \end{equation}
    Here $a_{ii}$ in \Cref{eq:abc} is non-null for each $i \in [3]$.
    Furthermore, the necessary and sufficient condition for the existence of global sections in \Cref{th:globalsec_existence} holds, as the multiplicity of the null eigenvalue of \CANL is $m(\mu_0)=1$.
    Then, in light of \Cref{thm:convergence}, we expect that the dynamics of CK converges to a global section.
    Accordingly, as shown in the bottom-right panel in \Cref{fig:can_dynamics} representing $t=100$, a global section emerges and its three components  lie in one-dimensional subspaces and are related by the embedding matrices $\V_{12}$ and $\V_{13}$.
    Finally, we remark that the dimensionality of the common subspace is lower than $\min_{i \in [3]} d_i = 2$.
    Indeed, the missing edge between $\scm{}_2$ and $\scm{}_3$ determines the unreachability from the coarsest node(s), and the dimension of $\ker(\CANL)$ depends on the intersection of the spans of $\V_{12}$ and $\V_{13}$.
\end{example}

\begin{figure}[t]
    \centering
    \includegraphics[width=1.0\linewidth]{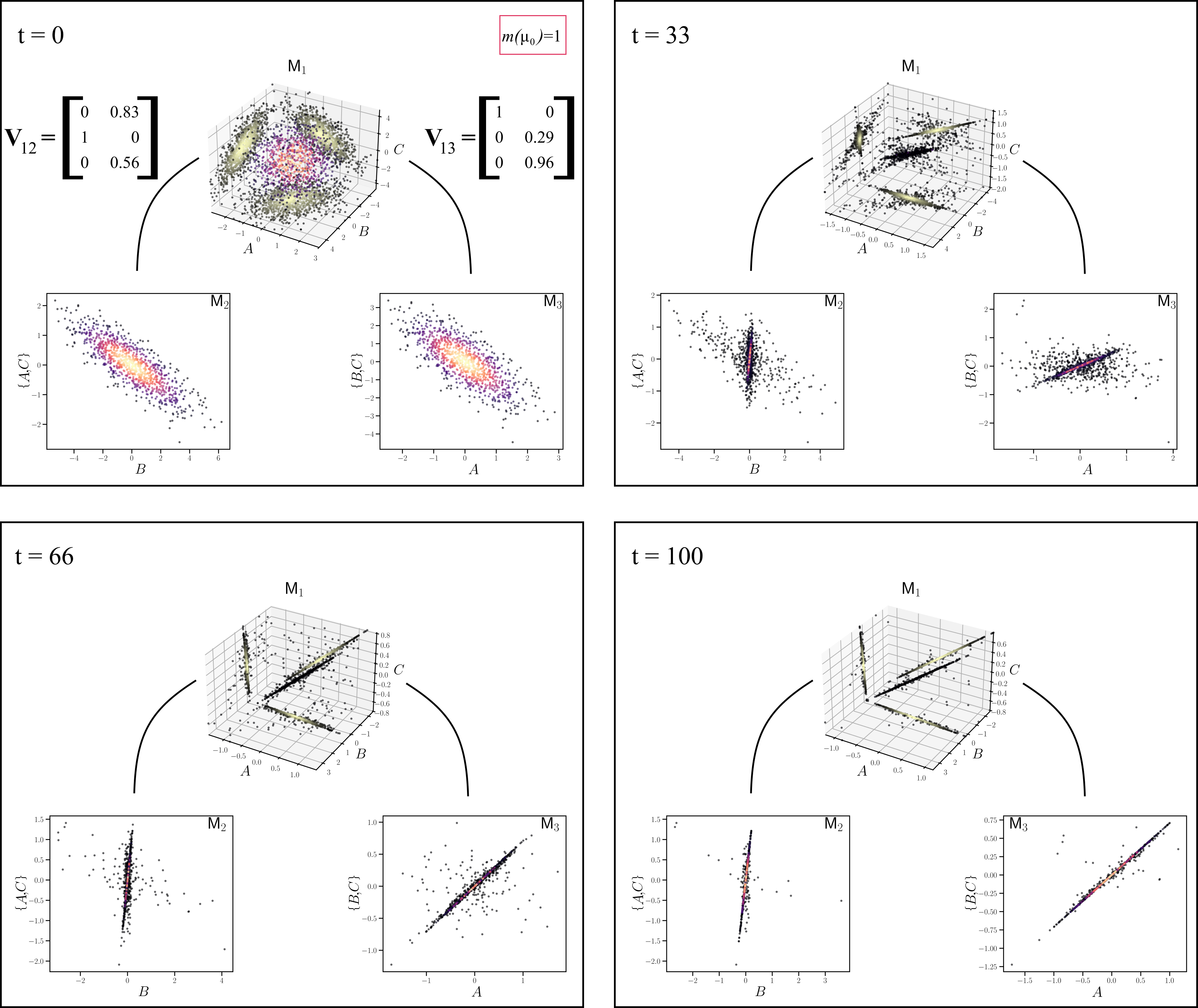}
    \caption{Dynamics of the CK over the CAN in \cref{ex:dynamics} for $T=100$ time steps (from top-left to bottom-right). In accordance with \Cref{th:globalsec_existence,thm:convergence}, the system converges and the emerging global section lives in one-dimensional subspace.}
    \label{fig:can_dynamics}
\end{figure}

\subsection{Role of hyper-parameters}\label{subsec:hyperparams}
An intriguing aspect of the proposed CAN--and more in general of the network sheaf and cosheaf of CK--is the presence of the aggregation, assimilation, and update parameters shaping the evolution of CK over time.
\Cref{fig:hyperparams-role} provides a pictorial representation for better understanding their action.
For instance, let us focus on $\mcm{}_2$ within \CAN in \Cref{fig:CAN}, and split the hyper-parameters into two sets: \emph{(i)} $\{\bar{\alpha}_{12}, \alpha_{23}, \alpha_{24}\}$ and $\{\bm\beta_2\}$ involved in the local contribution $\mathcal{L}\at{2}$ of the diffusion of CK over CAN computed according to \Cref{eq:Laplacian_op_local}; and \emph{(ii)} $\lambda_2$, which acts between consecutive time steps and is involved in the combination of the probability distribution at node \num{2} at time $t$ and the local contribution $\mathcal{L}\at{2}$, thus yielding the probability distribution at $t+1$.

Looking at the local contribution $\mathcal{L}\at{2}$ coming from the diffusion via the CAN Laplacian operator, we see that as the aggregation parameters $\alpha_e \to 1$, the resulting edge stalks valuation $\measure_{12},\measure_{23},\measure_{24}$ will tend to the pushforward distributions induced by the corresponding restriction maps in \Cref{eq:ns_restriction_can}.
Vice versa, as $\alpha_e \to 0$ the latter pushforward distributions becomes negligible in the combination performed at the edge stalk level.
Interestingly, the limit case $\alpha_e=0$ can be exploited by an external CAN controller to prevent agent \num{2} from “contaminating” the CK of the other agents, for example in the event of abnormal, suspicious or malicious behavior.
Conversely, imposing $\alpha_e=1$ can be exploited to preserve the CK of agent \num{2} while inducing a bias in the global CK towards the latter.

Concerning the assimilation parameter, $\bm\beta_2 \in \simplex{3}$, when $\beta_{2,i} \to 1$ the CK incoming from the $i$-th edge stalk will receive more importance, the opposite when $\beta_{2,i}\to 0$.
In this case, setting $\beta_{2,i}=0$ can be used to make the CK of agent \num{2} insensitive to that coming from the $i$-th edge. Conversely, setting $\beta_{2,i}=1$--and consequently zeroing the remaining entries of $\bm\beta_2$--can be used to favor the combination with the CK coming from the $i$-th edge.

Additionally, it is worth noting that zeroing both the aggregation and assimilation parameters for a certain edge amounts to removing that edge from \CAN.
In fact, \emph{(i)} in the outgoing direction, the CK of agent \num{2} will not be combined on the edge stalk because the aggregation parameter is null, and \emph{(ii)} in the incoming direction, it will not be affected by the CK of the edge because the assimilation parameter is zero.

Looking at the evolution between time steps, CK at $t+1$ consists of the combination of that at time $t$ and the local contribution  $\mathcal{L}\at{2}$ through the update parameter $\lambda_2 \in [0,1]$.
When $\lambda_2 \to 0$, CK diffusion is enhanced; conversely for $\lambda_2\to1$ CK tends to be stationary.

\begin{figure}[t]
    \centering
    \begin{tikzpicture}
        %t
        \node[draw, circle, opacity=1]   (B1) at (-2,-2) {$\mathsf{\scriptstyle  MCM}_2$};
        \node[draw, align=left] (t0) at (-2,-3) {$t$};
        
        \node[draw, circle, opacity=1]   (B) at (1,0) {$\mathcal{L}\at{2}$};
        \node[draw, rectangle, opacity=0.2] (AB) at (.5,1.5) {$\mathsf{\scriptstyle  MCM}_{12}$};
        \node[draw, rectangle, opacity=0.2] (BC) at (2.5,0) {$\mathsf{\scriptstyle  MCM}_{23}$};
        \node[draw, rectangle, opacity=0.2] (BD) at (-.5,-.5) {$\mathsf{\scriptstyle  MCM}_{24}$};
        
        %t+1
        \node[draw, align=left] (t1) at (4,-3) {$t+1$};
        \node[draw, circle, opacity=1] (B2) at (4,-2) {$\mathsf{\scriptstyle  MCM}_2$};
        
        \draw[->,Mulberry] (AB) to[bend right] node[near start,left]{$\beta_{2,1}$} (B);
        \draw[->,Periwinkle] (B) to[bend right] node[near end,right]{$\bar{\alpha}_{12}$} (AB);
        \draw[->,Periwinkle] (B) to[bend right] node[midway, below]{$\alpha_{23}$} (BC);
        \draw[->,Mulberry] (BC) to[bend right] node[midway, above]{$\beta_{2,3}$} (B);
        \draw[->,Periwinkle] (B) to[bend right] node[near end,above]{$\alpha_{24}$} (BD);
        \draw[->,Mulberry] (BD) to[bend right] node[midway,below]{$\beta_{2,2}$} (B);
    
        \draw[-, opacity=1] (AB) to (B);
        \draw[-, opacity=1] (BC) to (B);
        \draw[-, opacity=1] (BD) to (B);

        \draw[->,teal] (B) to[out=270, in=150] node[midway,above]{$\bar{\lambda}_2$} (B2);
        \draw[->,teal] (B1) to node[midway, below]{$\lambda_2$} (B2);
    \end{tikzpicture}
    \caption{Evolution over time of the CK of $\mcm{}_2$.}
    \label{fig:hyperparams-role}
\end{figure}

\subsection{Smoothness}\label{subsec:smoothness}
Recall that the structure of a consistent \CAN is shaped by the existence of CLCAs among the MCMs in \mcmcollection, represented as nodes. 
Intuitively, a collection of MCMs is smooth w.r.t. \CAN if, for each edge $e_{i j}\!:\! i \!\sim\! j$, the coarser $\mcm{}_j$ is--approximately--a CLCA of the finer $\mcm{}_i$.

Thus, smoothness can be measured in terms of how well the abstraction of the low-level probability distribution $\measure_i$ approximates the high-level $\measure_j$.
Denoting by $D_{\varphi}\left(\mu,\nu\right)$ any information-theoretic metric or $\phi$-divergence between $\mu$ and the pushforward of $\nu$ via the map $\varphi$, we measure smoothness by
\begin{equation}\label{eq:smoothness}
    f(\edgeset, \{\V_{j i}\}) \coloneqq \sum_{e_{i j} \in \edgeset} D_{\V_{ji}}\left(\measure_j,\measure_i\right)\,.
\end{equation}
Recall $\V_{j i}=\V_{i j}^\top$ and, by consistency, $\V_{i j} \in \stiefel{d_i}{d_j}$. 

It is worth recalling that in graph signal processing (GSP) \citep{ortega2018graph} and network sheaves valued in the category of vector spaces and linear maps \citep{hansen2019learning}, smoothness
is a variational notion induced by the Laplacian matrix, it is associated with a Dirichlet-like energy and is typically applied to deterministic signals. Conversely, in CAN, see e.g.   in \Cref{eq:smoothness}, we measure smoothness in terms of distances between probability distributions of different dimensions in non-Euclidean spaces. In a nutshell, smoothness can be thought of as an \emph{information-theoretic measure of causal coherence across abstraction levels.} This is a significant generalization of the approach typically followed in GSP. 

Furthermore, in GSP the Laplacian governs diffusion toward a global state and is a symmetric operator. Conversely, in CANs smoothness quantifies the fidelity of local causal abstractions. 
Consequently, smoothness is inherently directional in CAN, reflecting the orientation of causal abstractions. 
Moreover, this notion of smoothness offers the flexibility to select the information measure best suited to the application: 
for instance, the Kullback–Leibler (KL) divergence can be used for multivariate Gaussian distributions, whereas for Gaussian mixtures--where no closed-form expression for the KL divergence exists--one may instead employ alternatives such as the Mixture Projection \num{2}-Wasserstein discrepancy \citep{cai2022distances,salmona2024gromovwassersteinlike}, or resort to approximation techniques \citep{hershey2007approximating}. 

\begin{remark}\label{smoothvsglobal}
    While every global section $\chi \in \mathrm{Fix}\left(\cc{\lambda}{\catidentity{}}{\mathcal{L}}\right)$ is necessarily smooth w.r.t. a consistent \CAN, the converse does not hold in general.
    Indeed, if \Cref{eq:smoothness} vanishes, then $\measure_j=cc_{\w_i}(P_{\V_{ji}(X_i)}^{[I]})$ for all $e_{i j}$, but this does not guarantee $\measure_i=cc_{\w_j}(P_{\V_{ij}(X_j)}^{[J]})$, which is the global section condition in \Cref{eq:global_sec_conditions}.  
    Conversely, by consistency, if $\measure_i=cc_{\w_j}(P_{\V_{ij}(X_j)}^{[J]})$ then smoothness is automatically satisfied, since the pushforward via $\V_{ji}=\V_{ij}^\top$ gives
    \begin{equation}
         cc_{\w_j}\left(P_{(\V_{i j}^\top \circ \V_{i j})(X_j)}^{[J]}\right) = \measure_j\,.
    \end{equation}
\end{remark} 

% !TEX root =  ../main.tex
\section{Learning causal abstraction networks}\label{sec:learning_CAN}

Our input is a collection $\chi=\{\chi_1, \ldots, \chi_N\}$ of Gaussian mixtures $\measure_i = cc_{\w_i}(P_{X_i}^{s \in [I]})$ in \GMM{I}{d_i}, with $\w_i =[w_{i,1},\ldots,w_{i,I}]^\top \in \simplex{I}$ and $\measure_{i}^s=P_{X_i}^s\sim N(\bm\mu_{i}^s,\bm\Sigma_{i}^s)$--sorted in descending order of dimensionality--induced by $N$ MCMs assumed to be smooth w.r.t. \CAN.
Furthermore, w.l.o.g., we assume $\measure_i$ to be centered, that is, $\bm\mu_i = \zeros_{d_i}$.
Please notice that this does not imply that the components $\measure_{i}^s$ are zero-mean Gaussians.
From Def.\nb\ref{def:CAN}, learning a consistent CAN amounts to learning \emph{(i)} the CLCA relations induced by the functor $A$ and determining its edge set \edgeset, 
and \emph{(ii)} the CLCAs $\{\V_{ij}^\top\}$ for each $e_{ij} \in \edgeset$, where $\V_{ij} \in \stiefel{d_i}{d_j}$.
Let $\mathcal{F}$ be the set of edges satisfying the CA necessary condition in Cor.\nb\ref{cor:spectral_interlacing_gmm}.
Then, we want to find the maximal edge set $\edgeset \subseteq \mathcal{F}$ and the corresponding CLCAs $\{\V_{ij}^\top\}$ such that the smoothness functional in \cref{eq:smoothness} vanishes.
As in \cref{eq:smoothness} we consider nonnegative discrepancies between probability distributions, the requirement $f(\edgeset, \{\V_{j i}\})=0$ is equivalent to requiring that $D_{\V_{ij}^\top}(\measure_j,\measure_i) = 0$ for each $e_{ij} \in \edgeset$.
However, in practice, it is unrealistic to require equality to zero, and it is sufficient to find a solution satisfying the constraint up to a small tolerance. 
Then, letting $\tau_D \approx 0$, our learning problem can be formulated with edge-wise constraints as follows:
\begin{equation}\label{eq:globalProblem}
    \begin{aligned}
        \max_{\mathcal{E} \subseteq \mathcal{F}} & \quad |\mathcal{E}| \\
        \text{subject to} \quad \inf_{\V_{ij} \in \stiefel{d_i}{d_j}}&\quad D_{\V_{ij}^\top}(\measure_j,\measure_i) \leq \tau_D, \quad \forall\, e_{ij} \in \edgeset\,.
    \end{aligned}
    \tag{P1}
\end{equation}

In particular, the set of CLCAs $\V_{ij}^\top$ for which the local $D_{\V_{ij}^\top}(\measure_j,\measure_i)\leq \tau_D$ determines the set of learned edges of \CAN.
Indeed, by construction, $\V_{ij}$ only makes sense in relation to a CA, thus solving \eqref{eq:globalProblem} reduces to finding the set of CLCAs satisfying the local constraints, which in turn determines \edgeset.
Evaluating the edge-specific constraints in Problem \eqref{eq:globalProblem} requires solving local CLCA learning problems, for which we need to specify the information-theoretic measure to be used.
Although in the case of canonical CLCA the high-level Gaussian mixture has the same number of components and weights as the low-level (cf. \cref{subsec:ct_formalization}), in light of Rem.\nb\ref{rem:general-clca}, we will continue with a more general discussion of mixtures with different numbers of components and weights. 
In particular, we expect that the method for solving local CLCA learning problems will also generalize well to the case described in Rem.\nb\ref{rem:general-clca}, as it does not rely on the aforementioned equality of the number of components and weights.

Similarly to \cite{d2025causal}, we leverage the work of \cite{cai2022distances} on distances between probability distributions with different dimensions.
However, in case of centered Gaussian mixtures, viz. $\measure_i = cc_{\w_i}(P^{[I]}_{X_i})$ in \GMM{I}{d_i} and $\measure_j = cc_{\w_j}(P_{X_j}^{[J]})$ in \GMM{J}{d_j}, with $d_i>d_j$, $\w_i=[w_{i,1},\ldots, w_{i,I}]^\top \in \simplex{I}$, and $\w_j=[w_{j,1},\ldots, w_{j,J}]^\top \in \simplex{J}$; in addition to the CLCA $\V_{ij}^\top$ determining the pushforward of each mixture component, we also have to consider the coupling among the components of the two mixtures.
We represent the latter as a coupling matrix $\bm\Omega \in \Pi(\w_i,\w_j)$ (cf. \cref{eq:coupling-matrix}).
Then, considering the $2$-Wasserstein distance between Gaussians
\begin{equation}\label{eq:W2}
    \Wtwo^2(\mu,\nu) = \norm{\bm m_\mu - \bm m_\nu}_2^2 + \Tr{\bm\Sigma_\mu+\bm\Sigma_\nu - 2 \left(\bm\Sigma_\nu^{\frac{1}{2}} \bm\Sigma_\mu\bm\Sigma_\nu^{\frac{1}{2}}\right)^{\frac{1}{2}}}\,; 
\end{equation}
and following \cite{salmona2024gromovwassersteinlike}, we call our specialized version of the projection distance by \cite{cai2022distances} the \emph{Mixture Projection Wasserstein Discrepancy} (\MPW).
Denoting by $\omega_{pq}$ the entries of $\bm\Omega$, the latter is defined as  
\begin{equation}\label{eq:mpw}
    \MPW(\measure_i,\measure_j) \coloneqq \inf_{\bm\Omega \in \Pi(\w_i,\w_j)} \inf_{\V_{ij} \in \stiefel{d_i}{d_j}} \sum_{p \in [I]} \sum_{q \in [J]} \omega_{pq} \Wtwo^2\left(P_{\V_{i j}^\top(X_i)}^p, P_{X_j}^q\right)\,.
\end{equation}

Thus, starting from \cref{eq:mpw} and letting $\tau_{\mathrm{MPW}}\approx 0$, Problem \eqref{eq:globalProblem} becomes
\begin{equation}\label{eq:globalProblemGMM}
    \begin{aligned}
        \max_{\mathcal{E} \subseteq \mathcal{F}} & \quad |\mathcal{E}| \\
        \text{subject to} \quad \inf_{\substack{\{\bm\Omega \in \Pi(\w_i,\w_j)\} \\ \{\V_{ij} \in \stiefel{d_i}{d_j}\}}}&\quad \MPW(\measure_i,\measure_j) \leq \tau_{\mathrm{MPW}}, \quad \forall\, e_{ij} \in \edgeset\,.
    \end{aligned}
    \tag{P2}
\end{equation}

To solve \eqref{eq:globalProblemGMM}, we propose an efficient search strategy over the candidate edge set $\mathcal{F}$ that leverages the compositionality of CAs in Lem.\nb\ref{lemma:composedCA}.
In particular, as detailed below, compositionality allows us to exploit transitive closure to progressively propagate inferred relations across edges, thus avoiding the need to solve all $|\mathcal{F}|$ local problems and significantly reducing their number during the optimization process.
Throughout, we assume full prior knowledge of the CA structure $\B_{ji}$, which is still a challenging setting relevant for application domains such as neuroscience \citep{d2025causal} and deep learning \citep{geiger2025causal}, as also demonstrated by our financial application in \cref{sec:finappl}.

\subsection{Search procedure}\label{subsec:search-proc}
%Consider $\chi$ as above.
Let us consider a collection of Gaussian mixtures. 
For each mixture covariance $\covM_i$ (cf. \cref{eq:gmm_sigma}), $i \in [N]$, we compute its eigenvalues.  
We encode possible CA relations according to Cor.\nb\ref{cor:spectral_interlacing_gmm} into a strictly lower triangular binary matrix $\mathbf{P} \in \{0,1\}^{N \times N}$.
Rather then applying Cor.\nb\ref{cor:spectral_interlacing_gmm} $N(N-1)/2$ times we fill $\mathbf{P}$ by leveraging the compositionality of CAs.
Specifically, we first apply Cor.\nb\ref{cor:spectral_interlacing_gmm} to the $N-1$ MCM pairs on the first subdiagonal, i.e., to consecutive neighbors in the chain $\chi_1\!-\!\chi_2\!-\!\ldots\!-\!\chi_N$.
We then set to $1$ the entries of $\mathbf{P}$ corresponding to pairs satisfying Cor.\nb\ref{cor:spectral_interlacing_gmm}, leaving the others null.
Afterward, we set $\mathbf{P}$ equal to its transitive closure.
Next, we move to the second subdiagonal.
At this stage, it is necessary to test at most $N\!-\!2$ MCM pairs, since some entries may already have been filled through the transitive closure at the previous step.
After the update, we apply again the transitive closure.
We repeat this procedure for all remaining subdiagonals.

Once $\mathbf{P}$ is constructed, we solve the local CLCA learning problems corresponding to its non-null entries by means of a \textsc{LocalSolver} detailed in \cref{subsec:local-prob}.
Here too, the number of subproblems to be solved can be reduced by exploiting transitive closure.
Let $\A \in [0,1]^{N\times N}$ encode the transitive reduction of all possible CLCA relations implied by the compositionality property, and let it be initialized as empty.
We first solve the local problems for the nonnull entries of the first subdiagonal of $\mathbf{P}$, and set to $1$ the entries of $\mathbf{A}$ whenever the CLCA holds, i.e., when \textsc{LocalSolver} finds a solution such that the convergence criteria for the local problem are matched.
We then compute the transitive closure of $\mathbf{A}$ and remove from $\mathbf{P}$ those entries already implied by this closure.
The same procedure is applied iteratively to the subsequent subdiagonals.
\cref{alg:search_proc} summarizes the search procedure.
\begin{algorithm}[ht]
\caption{Search procedure}
\label{alg:search_proc}
\begin{algorithmic}[1]
\vspace{-0.1cm}
\REQUIRE Ordered $0$-cochain $\chi = \{\chi_1,\ldots,\chi_N\}$, mixture covariances $\{\covM_i\}_{i=1}^N$, loss tolerance $\tau_{\mathrm{MPW}}$, args for \textsc{LocalSolver} 
\ENSURE Binary matrix $\A$ encoding CLCA relations, $\V_{ij}$ for each $\A_{ij}\neq 0$

\STATE \textit{Step 1: Construction of $\mathbf{P}$}
\STATE Compute eigenvalues of each $\covM_i$
\STATE Initialize $\mathbf{P} \leftarrow \zeros_{N \times N}$ 
\FOR{$d = 1$ \TO $N-1$}
    \FOR{$i = d+1$ \TO $N$}
        \IF{$\mathbf{P}_{i,i-d} = 0$}
            \IF{Cor.\nb\ref{cor:spectral_interlacing_gmm} holds for $(\chi_{i-d},\chi_i)$}
                \STATE $\mathbf{P}_{i,i-d} \leftarrow 1$
            \ENDIF
        \ENDIF
    \ENDFOR
    \STATE $\mathbf{P} \leftarrow \textsc{TransitiveClosure}(\mathbf{P})$
\ENDFOR

\STATE
\STATE \textit{Step 2: Learning \A and $\{\V_{ij}\}$}
\STATE Initialize $\A \leftarrow \zeros_{N \times N}$ 

\FOR{$d = 1$ \TO $N-1$}
    \FOR{$i = d+1$ \TO $N$}
        \IF{$\mathbf{P}_{i,i-d} = 1$}
            \STATE $\V_{i,i-d}, \mathrm{ConvergenceFlag} \leftarrow \textsc{LocalSolver}(\chi_{i-d},\chi_i)$
            \IF{$\mathrm{ConvergenceFlag}$} 
                \STATE Compute $\MPW(\measure_{i-d},\measure_i)$ using $\V_{i,i-d}$
                \IF{$\MPW(\measure_{i-d},\measure_i)\leq \tau_{\mathrm{MPW}}$}
                    \STATE $\A_{i,i-d} \leftarrow 1$
                \ENDIF
            \ENDIF
        \ENDIF
    \ENDFOR
    \STATE $\A \leftarrow \textsc{TransitiveClosure}(\A)$
    \STATE Remove from $\mathbf{P}$ all entries implied by $\A$
\ENDFOR

\RETURN $\A$, $\{\V_{ij}\}$ such that $\A_{ij}=1$ 
\end{algorithmic}
\vspace{-0.1cm}
\end{algorithm}

\subsection{Local problem}\label{subsec:local-prob}
Consider a single edge $e_{ij}$.
Since $e_{ij} \in \mathcal{F}$, thus Cor.\nb\ref{cor:spectral_interlacing_gmm} holds and a CLCA between $\mcm{}_i$ and $\mcm{}_j$ \emph{might} exist.
Let $\V_{ij}=\B \odot \V$, where $\B = \B_{ij}$ and $\odot$ is the Hadamard product.
Then, starting from \cref{eq:mpw}, the local problem is 
\begin{equation}\label{eq:local-prob-gmm}
    \min_{\substack{\bm\Omega \in \Pi(\w_i,\w_j) \\ \V \in \stiefel{d_i}{d_j}}} \; \sum_{p \in [I]} \sum_{q \in [J]} \omega_{pq} \Wtwo^2\left(P_{(\B \odot \V)^\top(X_i)}^p,P_{X_j}^q\right) - \lambda_\omega H(\bm\Omega)\,;
    \tag{SP1}
\end{equation}
where we also consider the entropic regularizer $H(\bm\Omega) = -\sum_{ij} \omega_{ij} \log{\omega_{ij}}$ with penalty parameter $\lambda_\omega \in \reall_+$.
Unfortunately, \eqref{eq:local-prob-gmm} is not jointly convex in $\bm\Omega$ and \V.
Our solution consists in an alternating optimization scheme.

\spara{Subproblem in $\bm\Omega$.}
Denote by $\V^k$ the value of the embedding matrix at the $k$-th iteration.
The subproblem to be solved can be equivalently rewritten in the form of an entropic regularized optimal transport (OT) problem,
\begin{equation}\label{eq:sub-Omega}
    \min_{\bm\Omega \in \Pi(\w_i,\w_j)} \; \Tr{\bm\Omega^\top \C} - \lambda_\omega H(\bm\Omega)\;
    \tag{SUB-$\bm\Omega$}
\end{equation}
where the cost matrix has entries $c_{pq}=\Wtwo^2\left(P_{(\B \odot \V^k)^\top(X_i)}^p,P_{X_j}^q\right)$, with $(p,q) \in [I]\times [J]$. 
The entropic regularization term is beneficial from a theoretical and practical standpoint.
On the one hand, it makes \eqref{eq:sub-Omega} is $\lambda_\omega$-strongly convex, ensuring the existence of a unique solution and enabling efficient computation via the Sinkhorn algorithm \citep{COTFNT}.
In our implementation we use the method provided by \cite{cuturi2022optimal}.
On the other hand, it acts as a probabilistic prior favoring high-entropy (i.e., less informative) distributions, uniform in the limit, preventing premature commitment at initial optimization stages to sharp, potentially noisy matching.

\begin{algorithm}[ht]
\caption{\textsc{SolveSOC}}
\label{alg:solveSOC}
\begin{algorithmic}[1]
\REQUIRE low-level $\measure_i\in\GMM{I}{d_i}$, high-level $\measure_j \in \GMM{J}{d_j}$, structural prior \B, coupling matrix $\bm\Omega$, initial value $\V^0$, initial value $\Y^0$, ADMM stepsize $\rho$, GD learning rate $\eta$, GD tolerance $\tau_{\mathrm{GD}}$, absolute tolerance $\tau_a$, relative tolerance $\tau_r$, maximum numbers of iterations $T_{\mathrm{GD}}$, $T_{\mathrm{SOC}}$
\ENSURE \V, \Y, $\mathrm{ConvergenceSOC}$

\STATE Initialize $\V \leftarrow \V^0$, $\Y \leftarrow \Y^0$, $\BPsi \leftarrow \B \odot \V - \Y$, $\mathrm{ConvergenceSOC} \leftarrow \mathrm{False}$, $r\leftarrow 1$
\WHILE{$\mathrm{ConvergenceSOC} = \mathrm{False}$ \AND $r \leq T_{\mathrm{SOC}}$}
    \STATE $u \leftarrow 1$
    \WHILE{$\frob{\V^{u+1} - \V^u} > \tau_{\mathrm{GD}}$ \AND $u\leq T_{\mathrm{GD}}$ }
        \STATE $\V^{u+1} \leftarrow$ \Cref{eq:gd-V}
        \STATE $u \leftarrow u+1$
    \ENDWHILE
    \STATE $\V \leftarrow \V^{u+1}$
    \STATE $\Y \leftarrow$ \Cref{eq:updateY-gmm}
    \STATE $\BPsi \leftarrow \BPsi + \B \odot \V - \Y$
    \IF{conditions in \Cref{eq:conv_conditions-gmm} hold}
        \STATE $\mathrm{ConvergenceSOC}=\mathrm{True}$
    \ENDIF
    \STATE $r\leftarrow r+1$
\ENDWHILE
\RETURN \V, \Y, $\mathrm{ConvergenceSOC}$
\end{algorithmic}
\end{algorithm}

\spara{Subproblem in \V.}
Denote by $\bm\Omega^k$ the value of the coupling matrix at the $k$-th iteration.
The subproblem to be solved is 
\begin{equation}\label{eq:subprob-V-gmm}
    \min_{\V \in \stiefel{d_i}{d_j}} \; \sum_{p \in [I]} \sum_{q \in [J]} \omega_{pq}^k \Wtwo^2\left(P_{(\B \odot \V)^\top(X_i)}^p,P_{X_j}^q\right) \,.
    \tag{SUB-\V}
\end{equation}
Unfortunately, \eqref{eq:subprob-V-gmm} is still nonconvex due to the orthogonality constraint induced by the Stiefel manifold membership.
To solve \eqref{eq:subprob-V-gmm}, we leverage the rationale underlying the \emph{splitting of orthogonality constraints} (SOC, \citealp{lai2014splitting}).
Specifically, we relax $\V$ to the Euclidean space $\reall^{d_i \times d_j}$ and enforce its orthogonality through a splitting variable $\Y \in \stiefel{d_i}{d_j}$, subject to 
\begin{equation}\label{eq:constraintV-gmm}
    \Y - \B \odot \V = \zeros_{d_i \times d_j}\,. 
\end{equation}
Thus, \eqref{eq:subprob-V-gmm} becomes
\begin{equation}\label{eq:subprob-V-gmm-splitting}
    \begin{aligned}
        \min_{\substack{\V \in \reall^{d_i \times d_j}\\ \Y \in \stiefel{d_i}{d_j}}} \quad &\sum_{p \in [I]} \sum_{q \in [J]} \omega_{pq}^k \Wtwo^2\left(P_{(\B \odot \V)^\top(X_i)}^p,P_{X_j}^q\right) \\
        \text{subject to}\quad & \Y - \B \odot \V = \zeros_{d_i \times d_j}\,.
    \end{aligned}
    \tag{$\widetilde{\text{SUB}}$-\V}
\end{equation}

Now, starting from \eqref{eq:subprob-V-gmm-splitting}, introducing the scaled dual variable $\BPsi \in \reall^{d_i \times d_j}$, we obtain the augmented Lagrangian
\begin{equation}\label{eq:aL-gmm}
    \begin{aligned}
        \mathcal{L}_\rho(\V, \Y, \BUpsilon) &= \sum_{p \in [I]} \sum_{q \in [J]} \omega_{pq}^k \Wtwo^2\left(P_{(\B \odot \V)^\top(X_i)}^p,P_{X_j}^q\right) + \frac{\rho}{2}\frob{\B \odot \V - \Y + \BPsi}^2 \,;
    \end{aligned}
\end{equation}
where $\rho \in \reall_+$ is the penalty parameter.
At each outer iteration $k$, we then minimize \Cref{eq:aL-gmm} iteratively with respect to the primal variables and maximize it with respect to the dual variables through ADMM \citep{boyd2011distributed}.
Specifically, the update recursion at the inner iteration $r$ is
\begin{equation}\label{eq:recursion-gmm}
    \begin{aligned}
        \V^{r+1} &= \argmin_{\V\in \reall^{d_i \times d_j}} \mathcal{L}_\rho(\V, \Y^r, \BPsi^r)\,,\\
        \Y^{r+1} &= \argmin_{\Y\in \stiefel{d_i}{d_j}} \mathcal{L}_\rho(\V^{r+1}, \Y, \BPsi^r)\,,\\
        \BPsi^{r+1} &= \BPsi^r + \B\odot\V^{r+1} - \Y^{r+1}\,.\\
    \end{aligned}
\end{equation}
The solution for the updates is given below.

\epara{Update for \V.}
The subproblem to solve is 
\begin{equation}\label{eq:updateV-gmm}
        \V^{r+1} = \argmin_{\V\in \reall^{d_i \times d_j}} \sum_{p \in [I]} \sum_{q \in [J]} \omega^k_{qp}  \Wtwo^2\left(P_{(\B \odot \V)^\top(X_i)}^p,P_{X_j}^q\right) + \frac{\rho}{2}\frob{\B \odot \V - \Y^r + \BPsi^r}^2\,.
\end{equation}
We minimize \cref{eq:updateV-gmm} via gradient descent \citep{boyd2004convex}, by exploiting automatic differentiation \citep{baydin2018automatic,jax2018github} for the computation of the gradient of the objective.
By setting $\V^{u}=\V^r$ for $u=0$, and denoting by $\eta$ the learning rate, we update
\begin{equation}\label{eq:gd-V}
    \V^{u+1} = \V^u - \eta \nabla_\V\left(\Wtwo^2\left(P_{(\B \odot \V)^\top(X_i)}^p,P_{X_j}^q\right) + \frac{\rho}{2}\frob{\B \odot \V - \Y^r + \BPsi^r}^2\right)\,;   
\end{equation}
until either $\frob{\V^{u+1}-\V^u}\leq \tau_\mathrm{GD}$, where $\tau_\mathrm{GD} \approx 0$, or a maximum number of iterations $T_{\mathrm{GD}}$ is reached.

\epara{Update for \Y.}
The subproblem to solve is
\begin{equation}\label{eq:updateY-gmm}
    \begin{aligned}
        \Y^{r+1} &= \argmin_{\Y} \frac{1}{2}\frob{\B \odot \V^{r+1} - \Y + \BPsi^r}^2 =\\
        &= \prox_{\stiefel{d_i}{d_j}}(\B \odot \V^{r+1} + \BPsi^r)\,;
    \end{aligned}
\end{equation}
where $\prox_{\stiefel{d_i}{d_j}}(\mathbf{S})$ is given in closed-form by the $\U$ factor of the polar decomposition of $\mathbf{S}$ \citep{d2025causal}.

\epara{Stopping criteria.}
Empirical convergence for the subproblem in \V is assessed via primal and dual residuals, which must vanish as $k \to \infty$. 
Following \cite{boyd2011distributed}, in our case
\begin{equation}\label{eq:residuals-gmm}
    \begin{aligned}
        \R_{\rm{p_1}}^{r+1} &= \Y^{r+1}-\B \odot \V^{r+1}\,; \quad \textit{(primal)}\\
        \R_{\rm{d_1}}^{r+1} &= \B \odot\left(\Y^{r+1}-\Y^{r}\right)\,; \quad \textit{(dual)}\\
    \end{aligned}
\end{equation}
Let $\tau_a, \tau_r \in \reall_+$ denote absolute and relative tolerances. 
According to \cite{boyd2011distributed}, convergence is established once the following conditions hold:
\begin{equation}\label{eq:conv_conditions-gmm}
    \begin{aligned}     
        \frob{\R_{\rm{p_1}}^{r+1}} &\leq \tau_a \sqrt{d_i d_j} + \tau_r \mathrm{max}\left(\frob{\Y^{r+1}},\frob{\B \odot \V^{r+1}}\right)\,;\\
        \frob{\R_{\rm{d_1}}^{r+1}} &\leq \tau_a \sqrt{d_i d_j} + \tau_r \frob{\B \odot \BPsi^{r+1}}\,.\\
    \end{aligned}
\end{equation}
The solver for the subproblem in \V is called \textsc{SolveSOC} and summarized in \Cref{alg:solveSOC}.

\spara{\mcalsep .}
We term the overall algorithm for solving the local problem \mcalsep, whose pseudo-code is reported in \cref{alg:mcalsep}.

\begin{algorithm}[ht]
\caption{\mcalsep}
\label{alg:mcalsep}
\begin{algorithmic}[1]
\REQUIRE low-level $\measure_i=cc_{\w_i}\left(P_{X_i}^{[I]}\right)$, high-level $\measure_j=cc_{\w_j}\left(P_{X_j}^{[J]}\right)$, structural prior \B, entropic penalty $\lambda_\omega$, ADMM stepsize $\rho$, GD learning rate $\eta$, GD tolerance $\tau_{\mathrm{GD}}$, absolute tolerance $\tau_a$, relative tolerance $\tau_r$, $\bm\Omega$ tolerance $\tau_{\bm\Omega}$, \V tolerance $\tau_\V$, maximum numbers of iterations $T_{\mathrm{GD}}$, $T_{\mathrm{SOC}}$, and $T_{\mathrm{GL}}$   
\ENSURE \V, $\bm\Omega$, $\mathrm{ConvergenceFlag}$ 

\STATE Initialize $\Y$ in \stiefel{d_i}{d_j}, $\V\leftarrow \Y$, \C with entries $c_{pq} \leftarrow \Wtwo^2\left(P_{(\B \odot \V)^\top(X_i)}^p,P_{X_j}^q\right)$ for each $(p,q) \in [I]\times [J]$, $\bm\Omega \leftarrow \textsc{Sinkhorn}(\C, \w_i, \w_j, \lambda_\omega)$, $\mathrm{ConvergenceFlag}\leftarrow \mathrm{False}$  

\FOR{$k=1$ to $T_\mathrm{GL}$}
    \STATE $\V_{\mathrm{prev}} \leftarrow \V$, $\Y_{\mathrm{prev}} \leftarrow \Y$, $\bm\Omega_{\mathrm{prev}} \leftarrow \bm\Omega$
    \STATE $\V, \Y, \mathrm{ConvergenceSOC} \leftarrow \textsc{SolveSOC}(\measure_i, \measure_j, \B, \bm\Omega, \V_\mathrm{prev}, \Y_\mathrm{prev}, \rho, \eta, \tau_{\mathrm{GD}}, \tau_a, \tau_r, T_{\mathrm{GD}}, T_{\mathrm{SOC}})$
    \STATE $c_{pq} \leftarrow \Wtwo^2\left(P_{(\B \odot \V)^\top(X_i)}^p,P_{X_j}^q\right)$ for each $(p,q) \in [I]\times [J]$
    \STATE $\bm\Omega \leftarrow \textsc{Sinkhorn}(\C, \w_i, \w_j, \lambda_\omega)$
    \IF{$\frob{\V -\V_{\mathrm{prev}}}\leq \tau_\V$ \AND $\frob{\bm\Omega - \bm\Omega_{\mathrm{prev}}}\leq \tau_{\bm\Omega}$ \AND $\mathrm{ConvergenceSOC}$}
        \STATE $\mathrm{ConvergenceFlag}\leftarrow \mathrm{True}$
        \RETURN \V, $\bm\Omega$, $\mathrm{ConvergenceFlag}$
    \ENDIF{}
\ENDFOR
\RETURN \V, $\bm\Omega$, $\mathrm{ConvergenceFlag}$

\end{algorithmic}
\end{algorithm}

\begin{remark}\label{remark:MG-case}
    Although the \mcalsep algorithm applies to the multivariate Gaussian case as well--basically, it only runs the \textsc{SolveSOC} procedure in \Cref{alg:solveSOC} with fixed $\bm\Omega=[1]$,--in our previous work \citep{dacunto2026learningconsistentcausalabstraction} we propose a specific problem formulation.
    The latter results in a feasibility problem that avoids nonconvex objectives and explicitly characterizes the set of optimal CLCAs.
    The algorithmic solution is a fast and specialized method, called \spectral, suitable for both positive definite and semidefinite covariances, which is relevant for global sections (cf. Cor.\nb\ref{cor:same_spectrum}).
\end{remark}
% !TEX root =  ../main.tex
\section{Empirical assessment}\label{sec:empirical_assessment}

In \Cref{subsec:clca-learning} we empirically evaluate the performance of \mcalsep in CLCA learning.
We start by comparing it with the baselines in \citep{d2025causal,dacunto2026learningconsistentcausalabstraction} in the multivariate Gaussian setting.
Since the methods in \citep{d2025causal} assume positive definite covariances, we focus on this case. 
Then we move to the mixture setting, where only \mcalsep is applicable.
In \Cref{subsec:can-learning}, we test our search procedure in \Cref{alg:search_proc} to learn different CAN structures.
The hyperparameter values used are given in App.\nb\ref{app:hyperparams}.

\subsection{CLCA learning}\label{subsec:clca-learning}
\spara{Multivariate Gaussian case.}
Given as input the low- and high-level SCMs, viz. \scm{\ell} and \scm{h}, the goal is to recover the CLCA matrix $\V^\top$, assuming full prior knowledge of \B.
As already remarked, this is still a challenging setting relevant for application domains such as neuroscience \citep{d2025causal} and deep learning \citep{geiger2025causal}.
In the multivariate Gaussian case, we compare \mcalsep with the \spectral, CLinSEPAL, LinSEPAL-ADMM, and LinSEPAL-PG methods by \citep{d2025causal,dacunto2026learningconsistentcausalabstraction}.  
While \spectral solves a feasibility problem, the other baselines minimize the KL divergence for the case of positive definite matrices, relying on different optimization frameworks.  
Specifically, CLinSEPAL solves a smooth constrained Riemannian problem (Prob. 3 in \citealp{d2025causal}) by combining SOC, ADMM \citep{boyd2011distributed}, and SCA \citep{nedic2018parallel}.
Conversely, LinSEPAL-ADMM and LinSEPAL-PG solve a nonsmooth unconstrained Riemannian problem (Prob. 2 in \citealp{d2025causal}), the former using manifold ADMM \citep{kovnatsky2016madmm} and the latter manifold proximal gradient \citep{chen2020}.

We tested the methods on the same synthetic dataset used in our previous work \citep{dacunto2026learningconsistentcausalabstraction}.
Thus, the results for the baselines are the same as \cite{dacunto2026learningconsistentcausalabstraction}.
The dataset includes three configurations $(\ell, h) \in \{(12,2), (12,4), (12,6)\}$, corresponding to \emph{high}, \emph{medium-high}, and \emph{medium} abstraction levels.  
For each configuration, $S=30$ ground-truth matrices \V are generated.  
In each simulation $s \in [S]$, the methods are run $\ntrials=50$ times with different initializations, and the best solution is selected according to the KL divergence for all the methods but \mcalsep, which optimizes \Wtwo instead.  
LinSEPAL-PG terminates when the KL divergence falls below $10^{-4}$, whereas the other methods stop upon residual convergence. 

To evaluate performance, we monitor 
\emph{(i)} the normalized absolute Frobenius distance between the estimate \Vhat and the ground-truth \V,
\emph{(ii)} the KL divergence $\mathrm{KL}\left(P_{X^h}||P_{\widehat{\V}^\top(X^\ell)}\right)$  and the $2$-Wasserstein distance $\Wtwo^2\left(P_{\Vhat^\top(X^\ell)},P_{X^h}\right)$ to quantify the alignment between $\measure^h=P_{X^h}$ and the abstraction $P_{\widehat{\V}^\top(X^\ell)}$ of $\measure^\ell=P_{X^\ell}$ via $\widehat{\V}$, in accordance with the objectives used by the methods.
Furthermore, we ensured that the returned \Vhat are constructive for all the methods and that they respect the structural prior \B.

\Cref{fig:mvG_synth_data} shows the performance of the tested methods.
The vertical bars indicate the interquartile ranges across the $S$ simulations. 
All methods achieve good alignment for all $s \in [S]$, with KL (center) and \Wtwo (right) showing low values and similar behaviors across methods and settings.  
However, the Frobenius absolute distance (left) shows that, under high coarse-graining, identifying \V--up to sign--is more challenging for all methods, as the set of matrices that zero the objectives enlarges.
For medium-high and medium coarse-graining instead, the estimates get closer to the ground-truth with a similar level of alignment.

In summary, these results show that \spectral and \mcalsep achieve comparable performance to the methods in \cite{d2025causal} in learning CLCAs with positive definite covariances as input, while being more widely applicable.
Since \spectral avoids nonconvex, costly objectives and further reduces the computational cost by providing updates in closed form, we recommend the latter method in the multivariate Gaussian setting.

\begin{figure}[ht]
    \centering
    \includegraphics[width=\columnwidth]{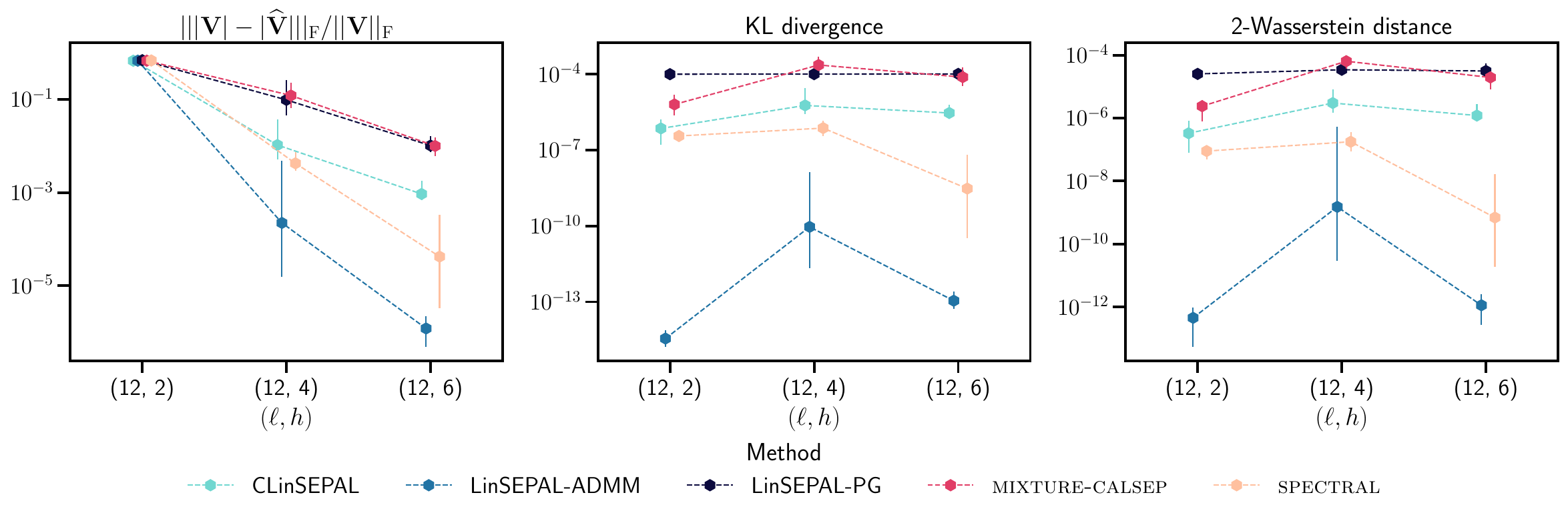}
    \caption{Synthetic results for the solution of the local problem across all settings $(\ell, h)$ from \cite{d2025causal}.}
    \label{fig:mvG_synth_data}
\end{figure}

\spara{Gaussian mixture case.}
We then turn to CLCA learning in the case where the input distributions \measurelow and \measurehigh are Gaussian mixtures in \GMM{3}{\ell} and \GMM{3}{h}, respectively.
Since here we also have to learn the coupling between components, for the performance evaluation, in addition to the normalized absolute Frobenius distance between the estimate \Vhat and the ground-truth \V, we monitor the normalized Frobenius distance between the estimate $\widehat{\bm\Omega}$ and the ground-truth $\bm\Omega$.
Instead, to quantify the alignment, we look at $\MPW\left(\measurelow, \measurehigh\right)$ in \Cref{eq:mpw}, since the only suitable method for this setting is \mcalsep.
Finally, we also ensured that the learned causal abstractions are constructive and adhere to \B. 

We consider the same settings $(\ell, h)$ as in the multivariate Gaussian case.
For each configuration, we generate $S=30$ ground-truth \V and random couplings $\bm\Omega$ for both cases of Gaussian mixtures whose components have \emph{(i)} positive semidefinite (psd) and \emph{(ii)} positive definite (pd) covariances.
Furthermore, to better assess the performance of the method, we test \mcalsep both when $\bm\Omega$ is specified as prior and when it must be learned.
Thus, for each configuration, we have four cases $(\textsf{\small{omega prior}}, \textsf{\small{psd}}) \in \{\mathrm{False},\mathrm{True}\} \times \{\mathrm{False},\mathrm{True}\}$.
For each of the four cases, \mcalsep is run $\ntrials=30$ times with different initializations, and the best solution is selected by looking at the lower values of \MPW.

\begin{figure}[ht]
    \centering
    \includegraphics[width=\columnwidth]{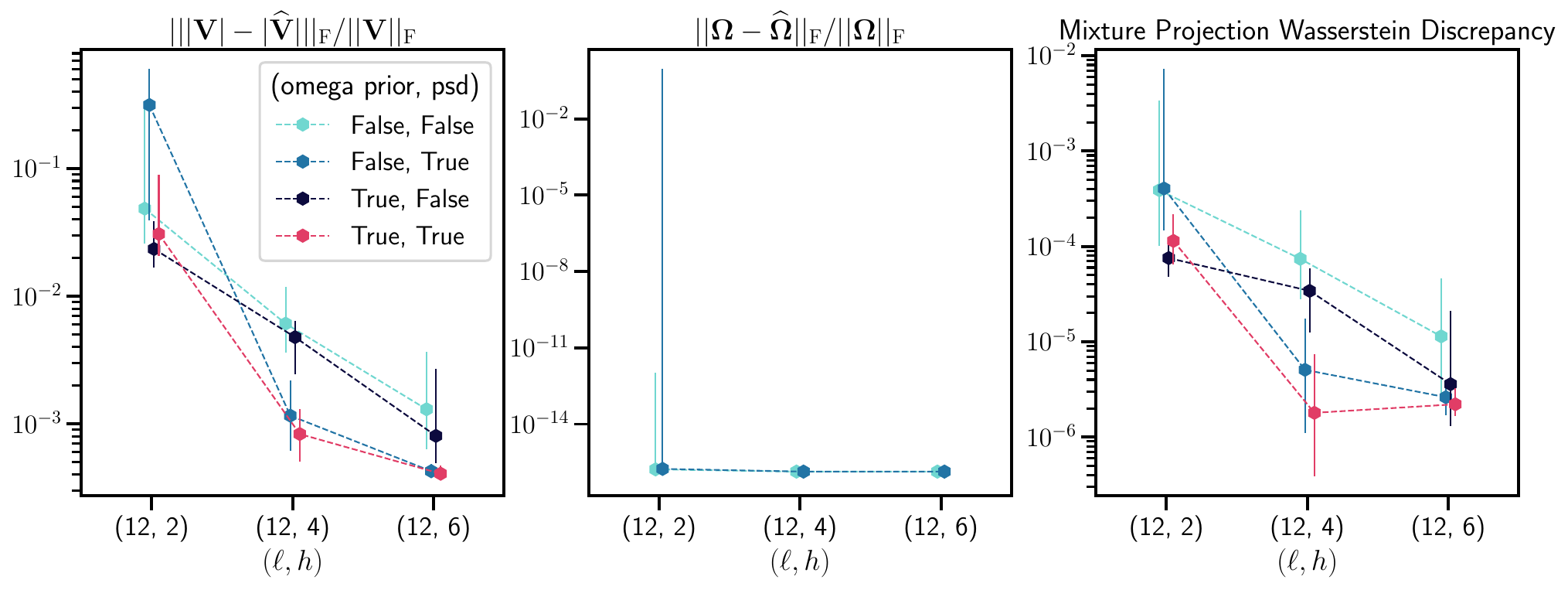}
    \caption{Synthetic results for the solution of the local problem across all settings $(\ell, h)$, when the input distributions are Gaussian mixtures.}
    \label{fig:gmm_synth_data}
\end{figure}

\Cref{fig:gmm_synth_data} shows the results, where the vertical bars indicate the interquartile ranges across the $S$ simulations. 
Overall, also in this case the level of alignment (right) is good for all the configurations and investigated cases.
Additionally, under high coarse-graining, the Frobenius absolute distance (left) tends to be higher when $\bm\Omega$ is not given as prior knowledge.
Indeed, looking at the performance on the estimation of $\bm\Omega$ (center), we see that for the latter setting the interquartile ranges are higher, for both psd and pd covariances.
Inevitably, errors on the estimation of $\bm\Omega$ negatively affect the learning of \V and the alignment.
Interestingly, when $\bm\Omega$ is provided instead, the Frobenius absolute distance as well as the interquartile ranges consistently reduce w.r.t. the multivariate Gaussian case.
We hypothesize that this result is due to the presence of multiple components related by the same CLCA which, compared to the latter case, facilitate the identifiability--up to sign--of \V.  
Finally, in the medium-high and medium coarse-graining, \mcalsep perfectly recovers $\bm\Omega$ and $\widehat{\V}$ gets very close in absolute terms to \V, leading to a better alignment.

\spara{Running without structural prior.}
So far, in the data-generating process, a CLCA between the low- and high-level representations always exists.
Moreover, all algorithms receive the ground-truth CLCA structure $\B^\top$ as input.
Here, instead, we investigate the capability of \mcalsep to find an alignment--in the form of a CLCA \V adhering to SEP--between the low- and high-level when \emph{(i)} a CLCA may not exist in the data-generating process and \emph{(ii)} the structural prior \B is not provided as input.
We proceed as follows.

First, we generate an MCM with $\ell=6$ causal variables, yielding a Gaussian mixture $\measurelow \in \GMM{3}{\ell}$.
Then, we coarsen \measurelow in two different ways, producing two Gaussian mixtures in $\GMM{3}{h}$, namely $\measure^{\mathrm{CLCA}}$ and $\measure^{\mathrm{Prob}}$, where $h=2$.
In the first case, the coarsening $\measure^{\mathrm{CLCA}}$ is a proper CLCA of \measurelow adhering to SEP.
In the second case, instead, $\measure^{\mathrm{Prob}}$ is not a CLCA; rather, it is a probabilistic Gaussian mixture model obtained from \measurelow via a dense linear map belonging to $\stiefel{\ell}{h}$.
However, in both cases, the necessary spectral interlacing condition in Cor.\nb\ref{cor:spectral_interlacing_gmm} for the existence of a CLCA holds.

At this point, we enumerate all possible CLCA structures $B$ and encode them with matrices $\B_i^\top$, $i \in [B]$.
Furthermore, in the case of $\measure^{\mathrm{CLCA}}$, the true CLCA has the same structure as $\B_2^\top$.
Then, for each $i \in [B]$, and considering $\B_i^\top$ as a working hypothesis for the CLCA structure, we solve Prob.\nb\ref{eq:globalProblemGMM} using \mcalsep without any prior on the coupling matrix $\bm\Omega$.
Specifically, we first provide as input to \Cref{alg:mcalsep} the pair $(\measurelow, \measure^{\mathrm{CLCA}})$ and then $(\measurelow, \measure^{\mathrm{Prob}})$.
Additionally, for each $i$ and each case, we run \mcalsep $\ntrials=30$ times with different initialization seeds.

\begin{figure}[tb]
\centering
\includegraphics[width=\linewidth]{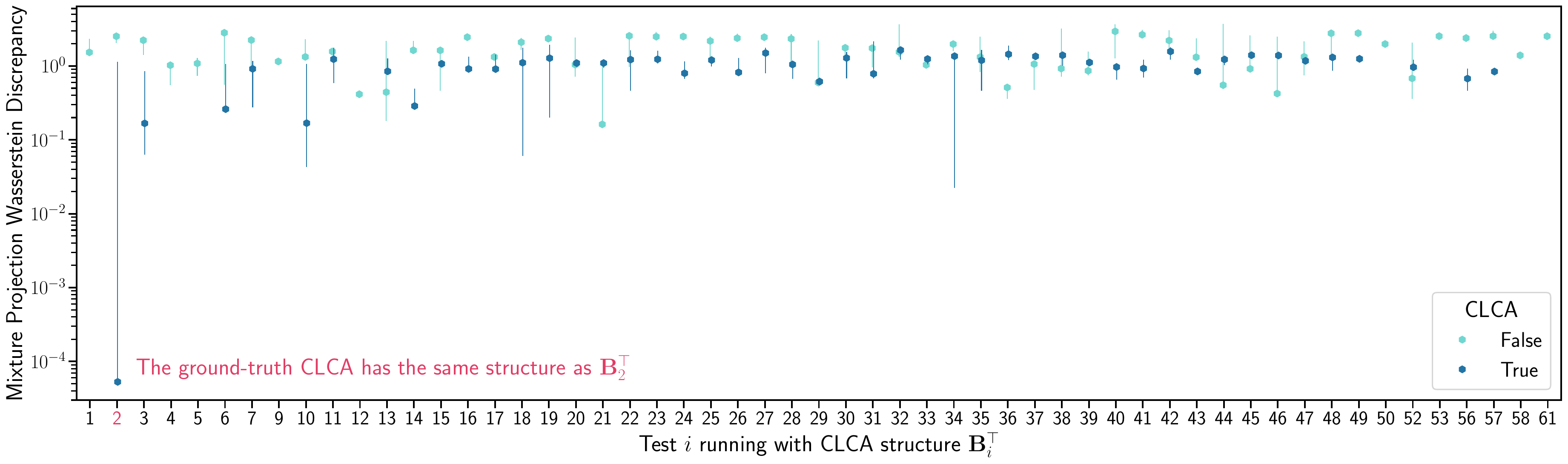}
\caption{\MPW across tests $i \in B$ using the hypothetical CLCA structure $\B_i^\top$.
Only converged trials are shown.
$\mathrm{CLCA}=\text{True}$ corresponds to running \mcalsep with $(\measurelow, \measure^{\mathrm{CLCA}})$ as input; $\mathrm{CLCA}=\text{False}$ corresponds to using $(\measurelow, \measure^{\mathrm{Prob}})$.
Markers denote the $50$th percentile, while bars represent the interquartile ranges.}
\label{fig:B_not_given}
\end{figure}

\Cref{fig:B_not_given} shows the results, considering only converged trials.
Hence, when none of the \ntrials converges, no data is plotted.
Restricting attention to converged trials is important for our goal, as it ensures that the learned maps belong to $\stiefel{\ell}{h}$ and respect the structure of the corresponding $\B_i$.

On the x-axis of \Cref{fig:B_not_given} we report the index $i$, while the y-axis shows \MPW.
Additionally, as indicated in the legend, $\mathrm{CLCA}=\text{True}$ corresponds to the case where \mcalsep is run with $(\measurelow, \measure^{\mathrm{CLCA}})$ as input--hence a CLCA exists in the data-generating process--whereas $\mathrm{CLCA}=\text{False}$ corresponds to using $(\measurelow, \measure^{\mathrm{Prob}})$.
Finally, markers denote the $50$th percentile, while bars represent the interquartile ranges computed over converged trials.

The first take-home message is that \mcalsep minimizes $\MPW(\measurelow, \measure^{\mathrm{CLCA}})$ below $10^{-4}$ at $i=2$, that is, when the CLCA structure $\B_2^\top$ is used.
This is consistent with the CLCA learning performance of \mcalsep shown in \Cref{fig:mvG_synth_data,fig:gmm_synth_data}.
Conversely, when $i \neq 2$, \mcalsep is unable to find a CLCA map that produces an alignment comparable to that obtained at $i=2$ in all trials.

Moreover, \mcalsep does not find any CLCA map that reduces $\MPW(\measurelow, \measure^{\mathrm{Prob}})$ below $10^{-1}$.
Thus, the second take-home message is that, for this dataset, it is difficult to find a CLCA that replicates the probabilistic coarsening and consistently minimizes the misalignment between the input distributions.

\begin{figure}[th]
    \centering
    \subfloat[Chain\label{subfig:chain}]{
        \includegraphics[width=.3\textwidth]{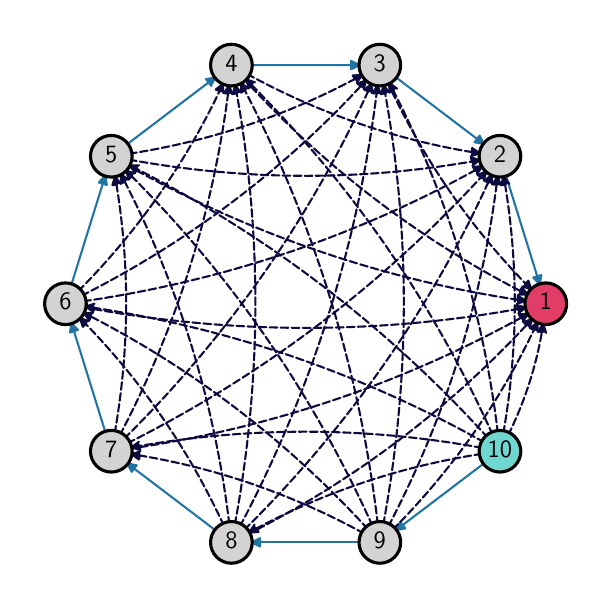}
    }
    \hfill
    \subfloat[Star\label{subfig:star}]{
        \includegraphics[width=.3\textwidth]{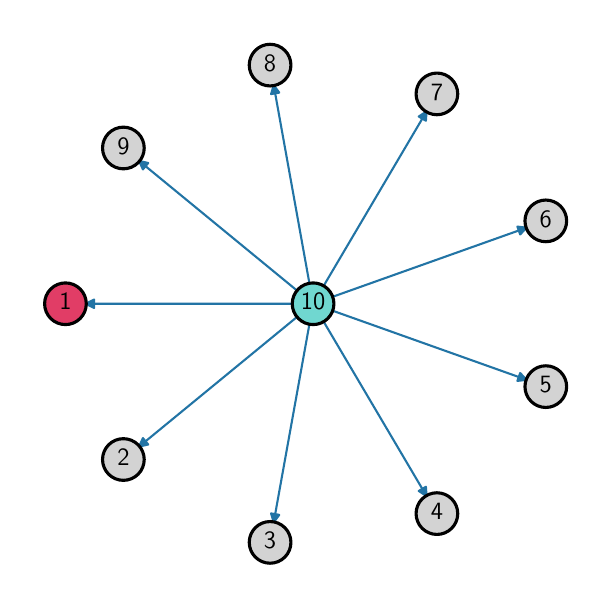}
    }
    \hfill
    \subfloat[Tree\label{subfig:tree}]{
        \includegraphics[width=.3\textwidth]{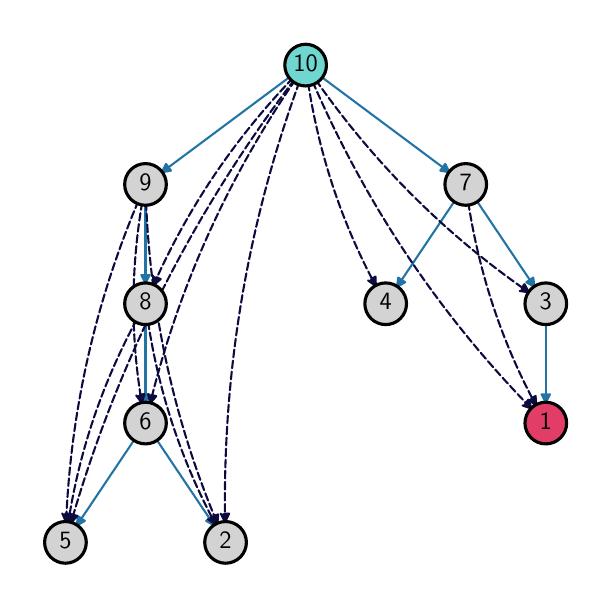}
    }
    \caption{CANs used in the empirical evaluation. 
    Nodes \myred{$1$} and \mycyan{$10$} are the finest and coarsest MCMs, respectively. 
    Edges in the transitive reduction are drawn in \mylightblue{light blue}; whereas edges appearing only in the transitive closure are dashed \mydarkblue{dark blue}.
    Panels correspond to (a) chain, (b) star, and (c) tree transitive reductions.}
    \label{fig:graphs}
\end{figure}

\subsection{CAN learning}\label{subsec:can-learning}

We then turn to the learning of consistent CANs.
We set $N=10$ and assign random dimensions $d_i \in [2, 2N]$ to each MCM, $i \in [N]$.
We then attach the MCMs to the nodes of three graphical structures in \Cref{fig:graphs}, where index \myred{$1$} corresponds to the finest model and index \mycyan{$10$} to the coarsest.
Edges in the transitive reduction are drawn in \mylightblue{light blue}, while edges belonging only to the transitive closure--the full CAN--are dashed \mydarkblue{dark blue}.
First, we have the chain in \Cref{subfig:chain}, whose full CAN is the fully-connected graph.
Second, the star in \Cref{subfig:star}, which coincides with the full CAN.
Third, the tree in \Cref{subfig:tree}, whose full CAN results in a denser tree.
For each CLCA $i \leqslant j$ in the transitive closure, we randomly generate the structure $\B_{ij} \in \{0,1\}^{d_i \times d_j}$ and sample $\V_{ij} \in \stiefel{d_i}{d_j}$ consistent with $\B_{ij}$. 

In light of \Cref{th:kernelCAN,th:globalsec_existence}, all considered structures admit global sections, which can be generated iteratively starting from the coarsest $\measure_{10}$.  
For each structure, we generate $S=30$ global sections with \num{3} components.
We then run the search procedure described in \Cref{alg:search_proc}, using \mcalsep--without any prior knowledge about the coupling matrix--to solve each local problem.  
Because of nonconvexity, each local problem is solved by running \mcalsep \ntrials times, with different initializations.
Specifically, up to $\ntrials \in \{10,50\}$ times.  
In case convergence is established--that is, $\mathrm{ConvergenceFlag}$ in \Cref{alg:mcalsep} is equal to $\mathrm{True}$--within $\ntrials$ runs, we stop and set the corresponding entry of \A to $1$.  
Performance is evaluated in terms of \emph{false positive rate} (FPR), \emph{true positive rate} (TPR), and \emph{false discovery rate} (FDR) w.r.t. the full CAN.  
By Cor.\nb\ref{cor:same_spectrum}, the covariance matrices of nodes $\measure_1, \ldots, \measure_9$ are positive semidefinite and share the same nonzero eigenvalues, corresponding to the spectrum of $\measure_{10}$.  
This is an extreme scenario, since the necessary condition in Cor.\nb\ref{cor:spectral_interlacing_gmm} holds true for all $N(N-1)/2$ MCM pairs.  
Consequently, the matrix $\mathbf{P}$ has all entries equal to $1$ in its strictly lower triangular part in both cases, and thus provides no information to reduce the number of local CA learning problems.  

\begin{figure}[t]
    \centering
    \includegraphics[width=\columnwidth]{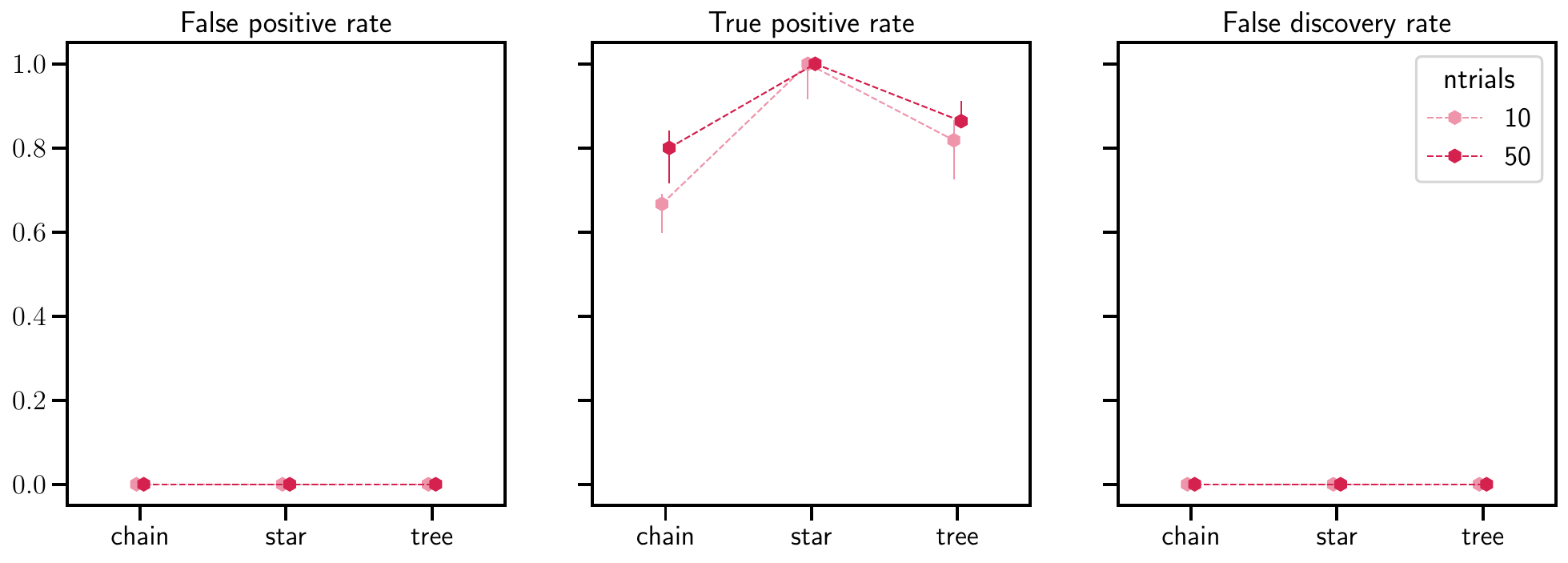}
    \caption{False positive (left) and true positive (center) and false discovery (right) rates for the proposed method on the three CANs in \Cref{fig:graphs}, when the observed global section \measure is a collection of Gaussian mixtures with three components, with different dimensions.}
\label{fig:CAN_learning}
\end{figure}

\Cref{fig:CAN_learning} shows the results.
Different values of \ntrials are given in hues of red, while vertical bars indicate the interquartile ranges across the $S$ simulations.
Although Cor.\nb\ref{cor:spectral_interlacing_gmm} holds true, the FPR (left) and FDR (right) confirm that \mcalsep never identifies a CLCA when none exists in the ground truth.
This holds across all settings and values of \ntrials and is consistent with the results in \Cref{fig:B_not_given}.
Conversely, the TPR (right) depends on \ntrials.
As expected, due to nonconvexity, the initialization is important since residual convergence may not be reached within the maximum number of iterations (cf. App.\nb\ref{app:hyperparams}).  
The TPR increases with \ntrials as expected.
Additionally, the results suggest that a lower value for \ntrials is sufficient to recover the star graph; the opposite is true for the chain.
This is an intriguing evidence that we plan to investigate further in future work.
% !TEX root =  ../main.tex
\section{Application to Financial Data}\label{sec:finappl}

This section presents an application of the proposed framework to financial data.

We collect daily returns of equally weighted industry portfolios at two different levels of granularity from the Fama--French data library\footnote{\url{https://mba.tuck.dartmouth.edu/pages/faculty/ken.french/data_library.html}}
, namely the datasets \texttt{5\_Industry\_Portfolios\_Daily.csv} and \texttt{10\_Industry\_Portfolios\_Daily.csv}.
Both datasets span the period from July \num{1}, \num{1926} to October \num{31}, \num{2025}. 
Industry portfolios are constructed by aggregating stocks traded on the NYSE, AMEX, and NASDAQ exchanges according to Compustat or CRSP SIC codes.
The former dataset aggregates stocks into \num{5} portfolios, while the latter aggregates them into \num{10}.
Hereafter, we refer to these granularities as level I and level II, respectively.
The aggregation induces the surjective map in \Cref{subfig:ind_level_map} from level II to level I, which we exploit to define priors on the structure of the CLCAs.

\begin{figure}[t]
    \centering

    \subfloat[Surjective map between industry portfolios.\label{subfig:ind_level_map}]{
        \includegraphics[width=0.45\linewidth]{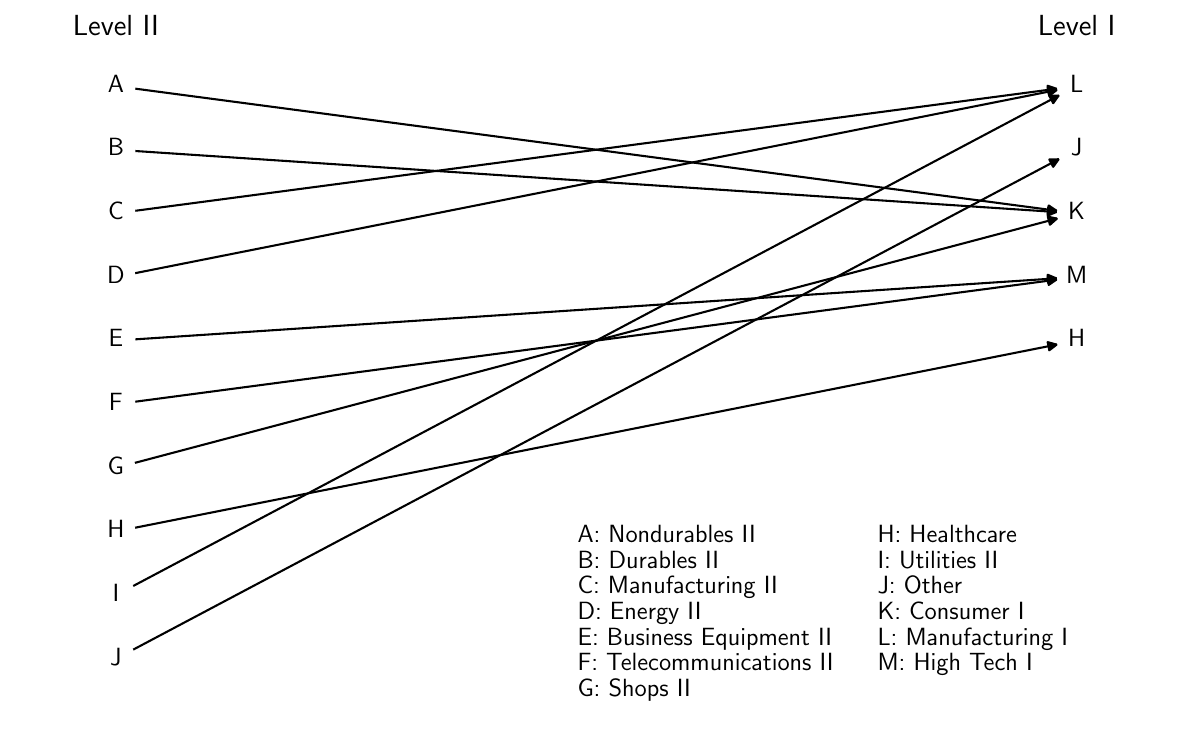}
    }\hfill
    \subfloat[CAN-based trading system. For each AI agent, the colored bar represents the industry portfolios in which the agent invests. \label{subfig:gtcan}]{
        \includegraphics[width=0.45\linewidth]{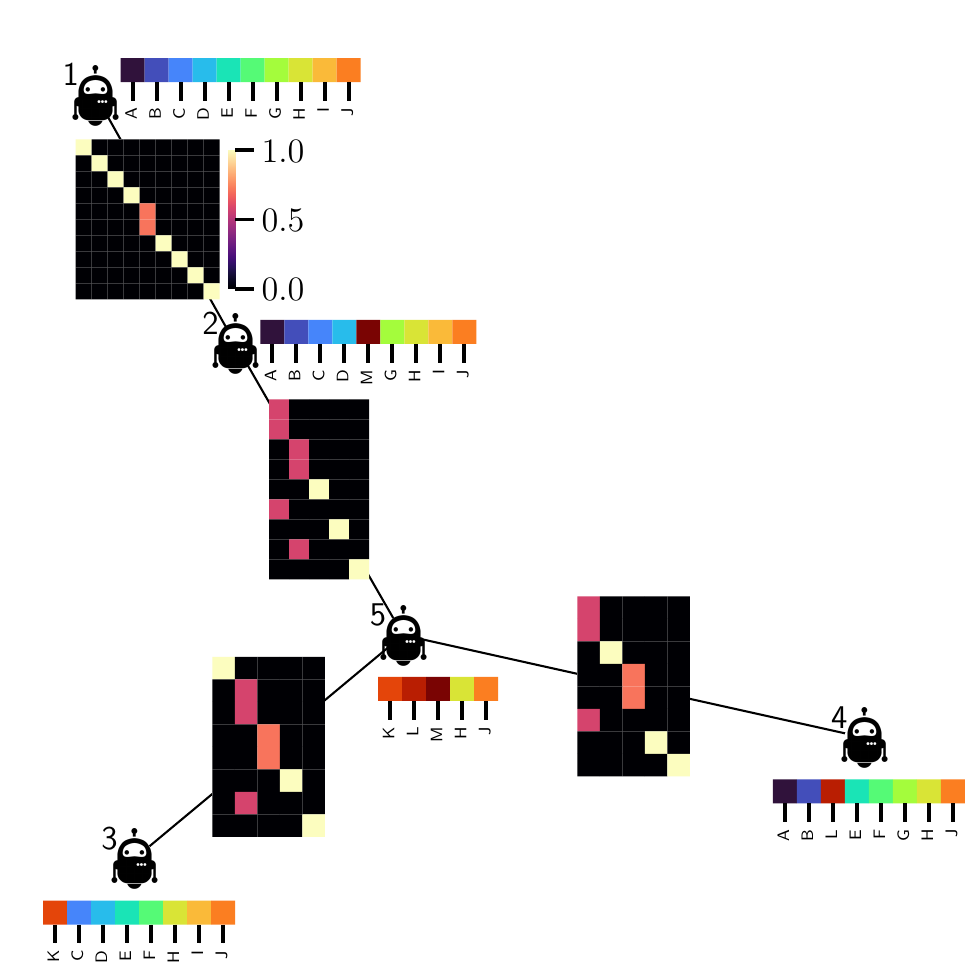}
    }

    \caption{Industry mapping and CAN-based trading system.}
    \label{fig:combined_fig}
\end{figure}

We simulate a CAN-based trading system composed of $N=5$ AI agents.
Each agent invests in industry portfolios from both level I and level II.
The system is anchored to the CK of the \enquote{coarsest} agent at index $5$, which invests exclusively in level I portfolios.
This corresponds to considering the CAN at equilibrium, i.e., the CK $\measure=\{\measure_1,\ldots,\measure_5\}$ is a global section belonging to the set of fixed points of \Cref{eq:dynsistCSprob} and can be recursively generated starting from $\measure_5$.
Specifically, we propagate the latter over the CAN via Stiefel embedding maps $\V_{ij}$, defined as
\begin{equation}\label{eq:gt-V-finappl}
    [\V_{ij}]_{mn}=\begin{cases}
        0 &\quad \text{if } [\B_{ij}]_{mn}=0\,,\\
        1/\sqrt{c_n} &\quad \text{otherwise;}
    \end{cases}
\end{equation}
where $c_n$ denotes the number of nonzero entries in the $n$-th column of $\B_{ij}$.
Notice that the definition of the maps $\V_{ij}$ reflects the aggregation strategy in the data, that is, industry portfolios are constructed using an equally-weighted aggregation strategy.
The resulting CAN is shown in \Cref{subfig:gtcan}.
For each agent, the colored bar indicates the industry portfolios in which it invests.
The coarsest agent is located at the center of the CAN, and all other agents are reachable from it.
Additionally, for each edge $i \sim j$, we plot the corresponding embedding map $\V_{ij} \in \stiefel{d_i}{d_j}$, where the color legend is located near $\V_{12}$.

The CK of the coarsest agent is a Gaussian mixture with three components induced by the agent’s MCM, capturing \emph{downturn}, \emph{mixed market}, and \emph{upturn} scenarios.
This corresponds to $w\!=\!3$, $u\!=\!5$ and $k_j\!=\!1$ in Def.\nb\ref{def:mcm-fun}, implying no local variation of causal mechanisms.
Each scenario is described by a structural vector autoregressive model (SVAR, \cite{kilian2017structural}) learned from data.
Specifically, for each scenario, we select a representative recent year based on annual returns and volatility of the industry portfolios: \emph{(i)} \num{2022} for the downturn, \emph{(ii)} \num{2023} for the mixed market, and \emph{(iii)} \num{2020} for the upturn.
The time series rebased to start at \num{1} and corresponding to each scenario are shown in \Cref{subfig:tsplot}.
For each scenario $s$, let \emph{(i)} $\X_s$ denote the time series; \emph{(ii)} $\X_{s,l}$ its lagged values at lag $l \in [L_s]_0$, where we denote by $[n]_0$ the range of integers from $0$ to $n$; \emph{(iii)} $\C_{s,l}$ the matrix of causal coefficients at lag $l \in [L_s]_0$; and \emph{(iv)} $\Z_s$ the values of the exogenous variables.
We apply the SSCASTLE method \citep{d2022multiscale} to fit the SVAR
\begin{equation}\label{eq:svar}
    \X_s = \sum_{l=0}^{L_s} \C_{s,l} \X_{s,l} + \Z_s\,,
\end{equation}
with an acyclicity constraint on $\C_{s,0}$ \citep{zheng2018dags} and $\ell_1$-norm sparsity regularization on $\C_{s,l}$.
The maximum lag $L_s$ and sparsity penalty $\lambda_s$ are selected in a data-driven manner via the BIC over the grid $(L_s, \lambda_s) \in [10]_0 \times [10^{-4}, 1]$.
In all scenarios $s$, the BIC selects the model with $L_s=0$, corresponding to the following structural equation model,
\begin{equation}\label{eq:sem}
    \X_s = \C_{s,0} \X_{s} + \Z_s\,.
\end{equation}
The learned causal structures are reported in \Cref{subfig:causalstructures}.

\begin{figure}[t]
    \centering

    \subfloat[Industry portfolios over the considered scenarios\label{subfig:tsplot}]{
        \includegraphics[width=\linewidth]{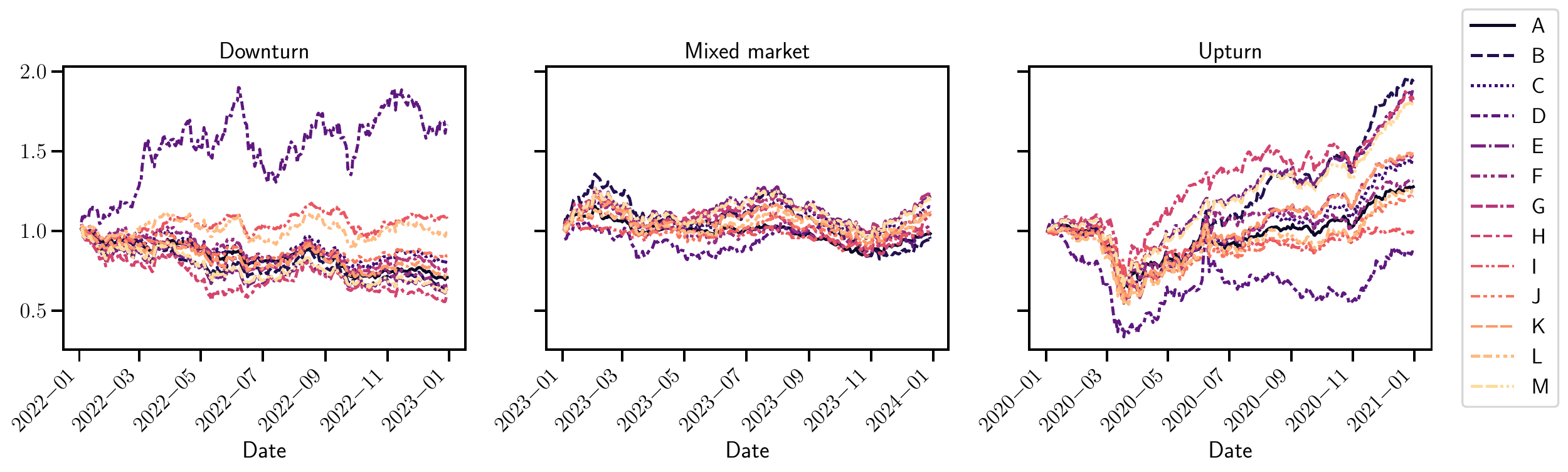}
    }

    \subfloat[Learned linear causal models\label{subfig:causalstructures}]{
        \includegraphics[width=\linewidth]{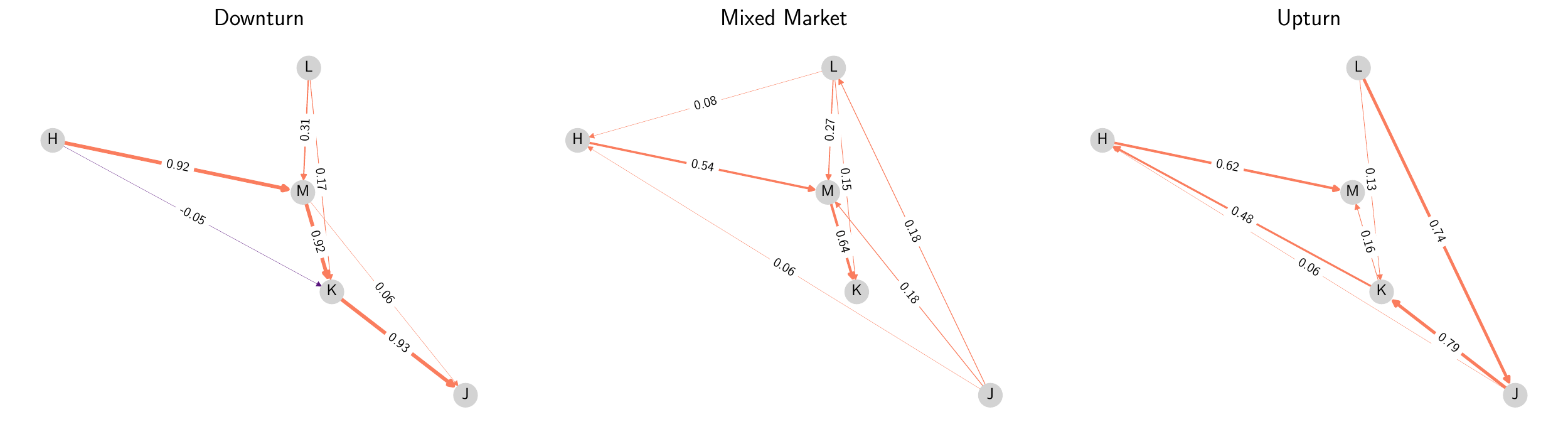}
    }

    \caption{The figure shows, for each scenario, (a) the behavior of the considered industry portfolios and (b) the linear causal model of the $5$-th AI agent learned via SSCASTLE from the level I portfolios , corresponding to $(\mathcal{D}_r, \mathcal{F}_r)$ in Def.\nb\ref{def:mcm-fun}.}
    \label{fig:vertical_figs}
\end{figure}

From the learned $\widehat{\C}_{s,0}$, we define the induced mixing matrix $\M_s = (\identity_{d_5} - \widehat{\C}_{s,0})^{-1}$.
The subjective CK of the coarsest agent, given its \emph{forward-looking market outlooks} $\myexogenous_s \sim N\left(\bm\mu_s^\myexogenous,\bm\Sigma_s^\myexogenous\right)$ with $\bm\Sigma_s^\myexogenous=\mathrm{diag}(\sigma_{s,1}^\myexogenous, \ldots, \sigma_{s, d_N}^\myexogenous)$, is
\begin{equation}\label{eq:coarsestCK}
    \begin{aligned}
        \GMM{S}{d_5} \ni \measure_5 &= cc_{\w}\left(P_{X_5}^{[S]}\right)=\sum_{s=1}^{S=3} w_s N(\bm\mu_s^\myendogenous, \bm\Sigma_s^\myendogenous)\\
        &=\sum_{s=1}^{S=3} w_s N(\M_s\bm\mu_s^\myexogenous, \M_s \bm\Sigma_s^\myexogenous \M_s^\top)\,.
    \end{aligned}
\end{equation}
For simplicity, defining the residuals
\begin{equation}\label{eq:residuals-sem}
    \widehat{\Z}_s=(\identity_{d_5} - \widehat{\C}_{s,0})\X_s\,,
\end{equation} 
we set  
\emph{(i)} $\bm\mu_s^\myexogenous=[\mu_{s,1}^\myexogenous, \ldots, \mu_{s,d_5}^\myexogenous]^\top$ equal to the mean vector of $\widehat{\Z}_s$, which is denoted by $\bm\mu_s =[\mu_{s,1}, \ldots, \mu_{s,d_5}]^\top$; and \emph{(ii)} $\sigma_{s,1}^\myexogenous, \ldots, \sigma_{s,d_5}^\myexogenous$ to the main diagonal of its covariance matrix, which is denoted by $\sigma_{s,1}\ldots,\sigma_{s,d_5}$.
We further set $\w=[0.3, 0.6, 0.1]^\top$, representing a setting in which the coarsest agent expects a mixed market scenario to be dominant in the upcoming period.
Finally, we construct the global section $\measure=\{\measure_1, \ldots, \measure_5\}$ by propagating the coarsest CK via Stiefel embeddings, that is, 
$\measure_i = cc_{\w}\left(P_{\V_{i5}(X_5)}^{[S]}\right)$ for $i \in [4]$, where $\V_{15}=\V_{12}\V_{25}$.

\subsection{CAN learning}\label{subsec:net_sheaf_learning_appl}

As a first task, we assume the CAN to be unknown and apply \cref{alg:search_proc} equipped with \cref{alg:mcalsep} to learn it from the observed global section $\measure$ and the structural matrices $\B_{ij}$, which are available from data (cf. \cref{subfig:ind_level_map}).
The hyperparameters used in the learning procedure are reported in App.\nb\ref{app:hyperparams}.

\begin{wrapfigure}{r}{0.45\linewidth}
    \centering
    % \vspace{-10pt}
    \includegraphics[width=\linewidth]{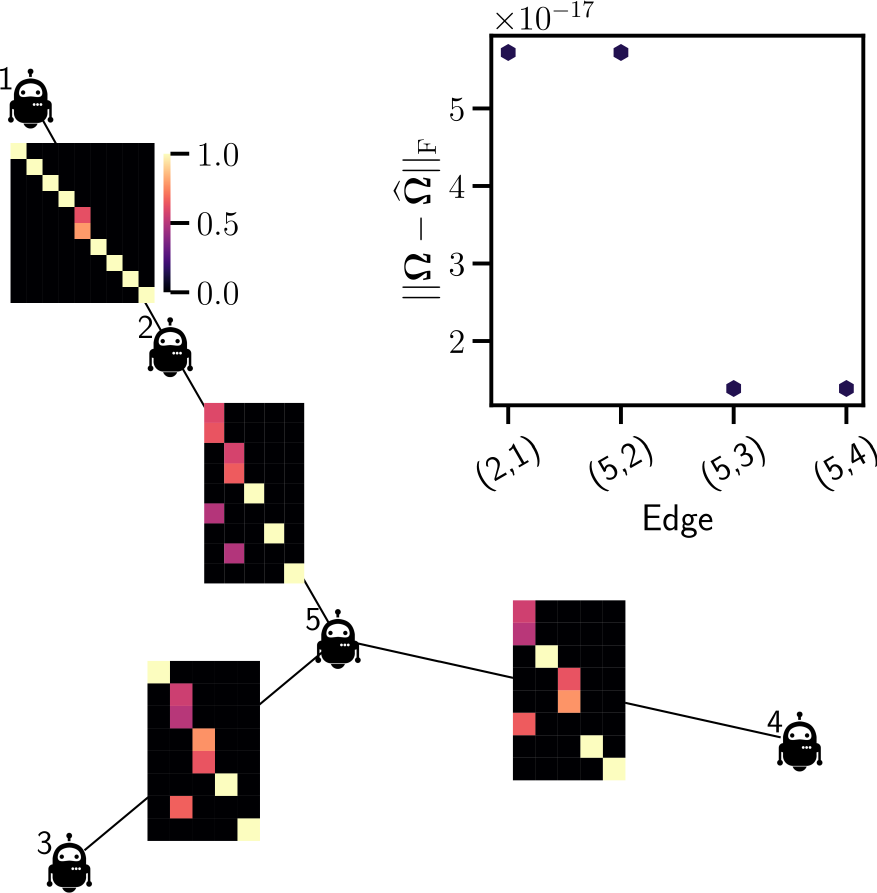}
    \caption{CAN learned by \cref{alg:search_proc} equipped with \cref{alg:mcalsep}. Ground-truth is shown in \cref{subfig:gtcan}. }
    \label{fig:finappl-learned-ns}
    \vspace{-10pt}
\end{wrapfigure}
\Cref{fig:finappl-learned-ns} displays the learned CAN, namely its topology--given by the transitive reduction of all admissible CLCAs relations--and the associated causal abstraction mappings encoded by the transposes of the embedding morphisms $\widehat{\V}_{ij} \in \stiefel{d_i}{d_j}$, with $d_i>d_j$.
Specifically, as done with those $\V_{ij}$ in \Cref{subfig:gtcan}, we plot these latter $\widehat{\V}_{ij}$ on the corresponding edges.
Additionally, we also show the Frobenius norm between the ground truth coupling matrix $\bm\Omega$ across GM components and the learned $\widehat{\bm\Omega}$, across all the learned edges.

From the comparison of \Cref{fig:finappl-learned-ns} with the ground truth in \Cref{subfig:gtcan}, we see that \cref{alg:search_proc} equipped with \cref{alg:mcalsep} perfectly recovers the CAN topology.
Additionally, the learned coupling matrix $\widehat{\bm\Omega}$ exactly matches the ground-truth coupling $\bm\Omega$ for each edge.
Furthermore, the learned embedding morphisms $\widehat{\V}_{ij}$ shown on the edges are consistent with the structural priors $\B_{ij}$, although their nonzero entries slightly deviate from the uniform values $1/\sqrt{c_n}$ used to define the ground-truth $\V_{ij}$ in \Cref{eq:gt-V-finappl}.
We therefore investigate how these deviations in the learned $\widehat{\V}_{ij}$ affect a downstream portfolio allocation task, which is relevant from the perspective of both portfolio and risk management.
In the sequel, the \emph{portfolio} of the agent is a linear combination of the industry portfolios in which it can invest, where the latter information is given in \Cref{subfig:gtcan}.

\spara{Mean-variance optimization over CAN.}
We consider a CAN optimizing a global mean--variance trade-off based on the CK of the coarsest agent.
Mean--variance optimization is a cornerstone of modern portfolio theory, originating from the seminal work of \citet{markowitz1952}.
The objective is to allocate capital across assets by jointly accounting for expected returns and their dependence structure, which captures the portfolio risk.
Indeed, the more the holdings behave similarly, the higher/lower the risk/diversification of the portfolio.
Let $\x\coloneqq[\x_1^\top,\ldots,\x_5^\top]^\top$ denote the portfolio weights of the AI agents, and let $\widehat{\V}_{55}=\identity_{d_5}$.
We restrict attention to \emph{(i)} \enquote{long-only} and \emph{(ii)} \enquote{fully-invested} strategies, i.e.,
$\x_i\geq\zeros_{d_i}$ and $\ones_{d_i}^\top \x_i=1$ for all $i \in [5]$, which means $\x_i \in \Delta_{d_i}$.
Let $\gamma \in \reall_+$ denote the risk-aversion parameter \citep{markowitz1952}.
The CAN mean--variance objective is then given by
\begin{equation}\label{eq:mvtradeoff}
    \mathcal{L}(\x ; \{\widehat{\V}_{i5}\}) =  \sum_{i=1}^5 \sum_{s=1}^S w_s \left(\left(\widehat{\V}_{i5}\bm\mu_s^\myendogenous\right)^\top\x_i - \gamma\, \x_i^\top \widehat{\V}_{i5} \bm\Sigma_s^\myendogenous \widehat{\V}_{i5}^\top \x_i \right)\,.
\end{equation}
As $\gamma$ increases, agents place greater emphasis on risk reduction through the covariance term.
The embedding matrices $\widehat{\V}_{ij}$ play a key role, as they induce the expected returns and covariance matrices of each agent’s assets from those of the coarsest agent.
Furthermore, we recall that $\{w_s, \bm\mu_s^\myendogenous, \bm\Sigma_s^\myendogenous\}_s$ are known as they represent the forward looking market outlook of the coarsest agent, who drives the overall CAN CK at equilibrium.

To compute the optimal allocations, we solve
\begin{equation}\label{eq:mvprob}
    \begin{aligned}
        \max_{\x} & \quad \mathcal{L}(\x ; \{\widehat{\V}_{i5}\})\\
        \text{subject to} & \quad \x_i \in \Delta_{d_i}\,,\quad \text{for each } i \in [5]\,. 
    \end{aligned}
    \tag{P3}
\end{equation}
Prob.\nb\eqref{eq:mvprob} is separable along the agent dimension $i$, and the subproblem for each agent is a constrained quadratic program that can be solved using standard optimization solvers \citep{diamond2016cvxpy,osqp}.  

\begin{figure}[t]
    \centering
    \includegraphics[width=\linewidth]{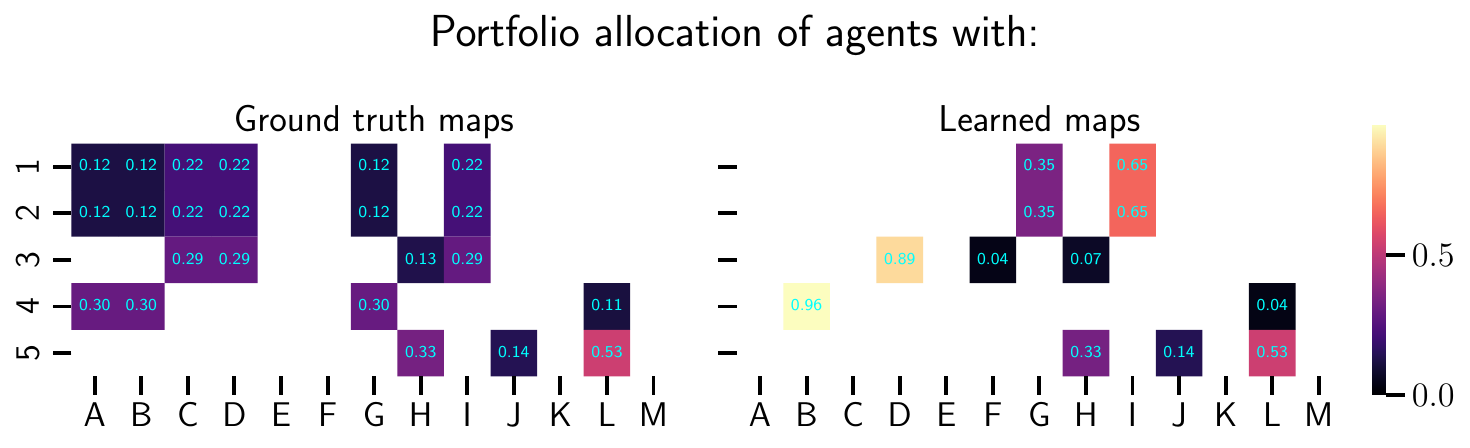}
    \caption{
    Optimized portfolio allocations obtained using (left) the ground-truth embedding matrices $\V_{i5}$ and (right) the learned $\widehat{\V}_{i5}$.
    In each matrix, rows correspond to AI agents, while columns correspond to industry portfolios. The rows are normalized, although rounding may cause the sum on some rows to appear greater than 1.}
    \label{fig:ptfalloc}
\end{figure}

\Cref{fig:ptfalloc} reports the optimized portfolio allocations obtained using in Prob.\nb\eqref{eq:mvprob} the ground-truth embeddings $\V_{i5}$ (left) and the learned embeddings $\widehat{\V}_{i5}$ (right), with $\gamma=0.5$.
For agent \num{5}, corresponding to the coarsest agent, the allocations coincide in both cases, since $\V_{55}$ and $\widehat{\V}_{55}$ are both set to the identity matrix.
For the remaining agents, the allocations appear markedly different at first glance.
However, when interpreted through the CAN structure and the aggregation map in \Cref{subfig:ind_level_map}, the differences are largely superficial.

Consider, for instance, agent \num{2}.
Under the ground-truth embeddings, the allocation splits into two clusters: \emph{(i)} industry portfolios $\{\mathsf{A}, \mathsf{B}, \mathsf{G}\}$, abstracted to $\mathsf{K}$ and accounting for $35\%$ of the portfolio; and \emph{(ii)} $\{\mathsf{C}, \mathsf{D}, \mathsf{I}\}$, abstracted to $\mathsf{L}$ and accounting for the remaining $65\%$.
Within each cluster, assets are equally weighted due to the degeneracy of the pushforward distribution, which induces identical expected returns and variances.

Instead, using the learned embeddings, a single industry portfolio is selected within each of the above mentioned clusters, namely $\mathsf{G}$ and $\mathsf{I}$, with weights close to those of the corresponding clusters.
Given the chosen risk-aversion level $\gamma$ and the slightly reduced values in the learned $\widehat{\V}_{25}$ with respect to those in $\V_{25}$ (cf. \Cref{subfig:gtcan,fig:finappl-learned-ns}), these portfolios are selected because they exhibit the lowest variance within their respective clusters.
A similar behavior is observed for the other agents; for agent \num{3}, a small additional weight is assigned to portfolio $\mathsf{F}$.
Overall, from the causal AI agents perspective, the portfolio allocations obtained using the learned embeddings are not qualitatively different from those obtained using the ground-truth embeddings. 
However, under realistic backtesting on market data--beyond the scope of this illustrative application--the strategies in \Cref{fig:ptfalloc} may lead to different performance due to several factors, such as the rebalancing frequency, transaction costs, and market liquidity.

\subsection{Counterfactual reasoning}\label{subsec:counterfactual_reasoning_appl}

We now consider a reasoning task over the CAN, which we consider as fully-specified system.
Consequently, we use the embedding maps $\V_{ij}$ in \Cref{eq:gt-V-finappl}.
Assume $N=5$ investors whose investment universes coincide with those of the AI agents, i.e., the $i$-th investor trades the same industry portfolios as agent~$i$.
Let $\y_i$ denote the portfolio allocation of the $i$-th investor, with $i \in [5]$.
We address the following questions:
\begin{squishlist}
    \item \emph{Q1)} How can we determine a \emph{counterfactual global section}--that is, a hypothetical CK for agent $5$--such that the observed allocations $\y=[\y_1^\top, \ldots, \y_5^\top]^\top$ maximize the mean--variance trade-off in \cref{eq:mvtradeoff}?
    \item \emph{Q2)} How does the composition of the counterfactual global section vary as a function of the risk-aversion parameter $\gamma$?
    \item \emph{Q3)} How far is the optimal CAN allocation $\x$--computed by solving Prob.\nb\eqref{eq:mvprob}--from $\y$ in terms of mean--variance trade-off, under the counterfactual global section, as $\gamma$ varies?
\end{squishlist}
These questions are intended as illustrative examples of the reasoning capabilities enabled by our framework.

\spara{Question Q1.}
Here, the portfolio allocations in \Cref{eq:mvtradeoff} are fixed to \y, and we optimize the CK--forward looking market outlook--of the coarsest agent in \Cref{eq:coarsestCK}.
Considering the MCM in \Cref{subfig:causalstructures} as fixed, the CK is determined by the parameters $\{w_s, \bm\mu_s^\myexogenous, \bm\Sigma_s^\myexogenous\}$ for $s \in [3]$, which we aim to learn.
For concreteness, we set $\y_i$ to be the equally weighted portfolio, commonly referred to as the \enquote{$1/n$ portfolio}.
To ensure a bounded optimization problem, we impose box constraints on $\mu_{s,i}^\myexogenous$ and $\sigma_{s,i}^\myexogenous$ for all $(s,i) \in [3]\times[5]$.
Additionally, we include an entropic regularization term on $\mathbf{w}$, leading to the following optimization problem:
\begin{equation}\label{eq:reasoning-prob}
    \begin{aligned}
        \max_{\mathbf{w} \in \Delta_3, \{\bm\mu_s^\myexogenous, \bm\Sigma_s^\myexogenous\}} & \quad \sum_{i=1}^5 \sum_{s=1}^3 w_s \left(\left(\V_{i5} \M_s \bm\mu_s^\myexogenous\right)^\top\y_i - \gamma\, \y_i^\top \V_{i5} \M_s \bm\Sigma_s^\myexogenous \M_s^\top \V_{i5}^\top \y_i \right) - \lambda_w \sum_s w_s \log w_s \, \\
        \text{subject to} & \quad \mu_{s,i}^\myexogenous \in [(1-\varphi) \mu_{s,i}, (1+\varphi)\mu_{s,i}],\quad \forall\,(s,i) \in [3]\times[5]\,,\\
        & \quad \sigma_{s,i}^\myexogenous \in [(1-\varphi) \sigma_{s,i}, (1+\varphi)\sigma_{s,i}],\quad \forall\,(s,i) \in [3]\times[5]\,. 
    \end{aligned}
    \tag{P4}
\end{equation}
In our experiments, we set $\varphi=0.1$.
Prob. \eqref{eq:reasoning-prob} is nonconvex, but admits a closed-form solution detailed in App.~\ref{app:sol-reasoning-prob} and denoted by $\{\widehat{w}_s, \widehat{\bm\mu}_s^\myexogenous, \widehat{\bm\Sigma}_s^\myexogenous\}_s$, thus yielding the counterfactual global section for a fixed value of $\gamma$.

\begin{figure}[t]
    \centering
    \includegraphics[width=\linewidth]{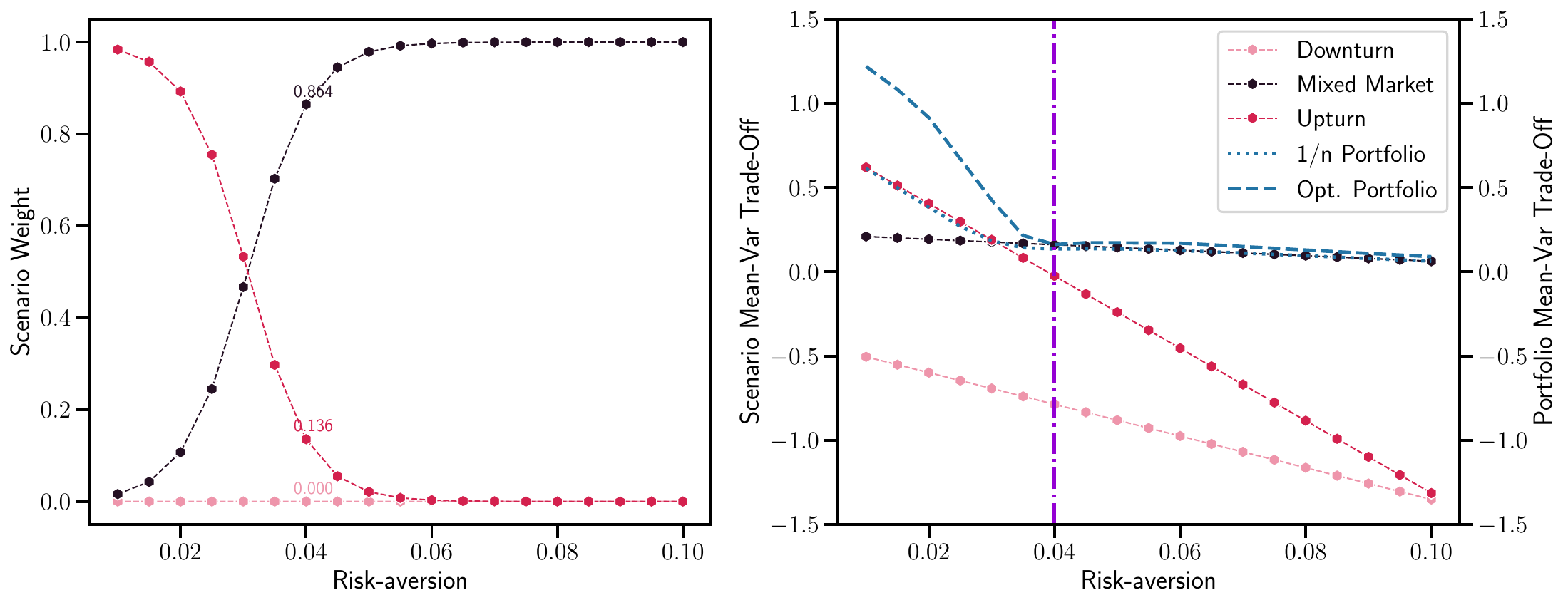}
    \caption{
    Optimized scenario weights $\mathbf{w}$ (left) and mean--variance trade-offs (right) as functions of the risk-aversion parameter $\gamma$.
    The right plot reports scenario-level trade-offs (left y-axis) and portfolio-level trade-offs (right y-axis).}
    \label{fig:reasoning}
\end{figure}

\spara{Question Q2.}
\Cref{fig:reasoning} illustrates \emph{(i)} the optimized scenario weights $\mathbf{w}$ of the counterfactual global section (left), and \emph{(ii)} the corresponding mean--variance trade-offs (right), as functions of $\gamma$.
In the right panel, scenario-level trade-offs are reported on the left y-axis, while portfolio-level trade-offs are shown on the right y-axis.
For a given $\gamma$, the portfolio trade-off corresponds to the first term of the objective in \eqref{eq:reasoning-prob}, evaluated at the optimized counterfactual global section and at the considered portfolio.
Instead, for each value of $\gamma$ and scenario $s \in [S]$, the scenario-level trade-off is defined as
\begin{equation}\label{eq:scenario-trade-off}
    R_{\gamma,s}=\sum_{i=1}^5\left(\V_{i5} \M_s \widehat{\bm\mu}_s^\myexogenous\right)^\top\y_i - \gamma\, \y_i^\top \V_{i5} \M_s \widehat{\bm\Sigma}_s^\myexogenous \M_s^\top \V_{i5}^\top \y_i \,.
\end{equation}

The trade-off of the $1/n$ portfolio initially tracks that of the upturn scenario up to $\gamma=0.03$, after which it aligns with the mixed market scenario.
This behavior is consistent with the evolution of the scenario weights in the counterfactual global section (cf. left panel in \cref{fig:reasoning}).
Indeed, for small values of $\gamma$--that is, when we disregard the volatility and focus mainly on returns,--the upturn component dominates, as it yields the highest trade-off.
At $\gamma=0.03$, the upturn and mixed market components have comparable trade-offs and thus weights.
For larger values of $\gamma$, the mixed market scenario outperforms and thus becomes dominant with a weight rapidly converging to \num{1}.

\spara{Question Q3.}
To answer this question, for each learned counterfactual global section and corresponding $\gamma$, we solve Prob.\nb\ref{eq:mvprob} to retrieve the optimal CAN allocation.
Similarly to the case of $1/n$ portfolio, the mean-variance trade-off of the optimal CAN allocation is the first term in \cref{eq:reasoning-prob}, where we replace each $\y_i$ with $\x_i$.
The right panel of \Cref{fig:reasoning} shows that, for $\gamma<0.035$, the mean--variance trade-off of the optimal CAN allocation consistently exceeds that of the $1/n$ portfolio.
Indeed, when risk aversion is low, agents concentrate their investments in the industry portfolios with the highest expected returns, largely disregarding risk, that is, the variance.
Conversely, as $\gamma$ increases, agents place greater emphasis on risk reduction and diversify their investments across multiple industry portfolios.
Consequently, the trade-off of the optimized portfolio approaches that of the $1/n$ portfolio, with the smallest gap occurring at $\gamma=0.04$.
This indicates that the counterfactual global section at $\gamma=0.04$ best approximates, from the CAN perspective, the financial beliefs underlying the $1/n$ investment strategy.
In other words, according to the constructed CAN, it provides the best causal representation of the five investors’ beliefs.
% !TEX root =  ../main.tex
\section{Discussion}\label{sec:discussion}

A first salient feature of CAN framework is the explicit separation between causal structure and probabilistic semantics. 
Structural information is encoded at the level of abstraction relations and network topology, while CK is represented as probability distributions and their transformations. 
This separation allows CANs to support reasoning and learning directly at the level of CK, without committing to a specific parametric or graphical representation of the underlying causal mechanisms.

Second, modeling CK through MCMs yields a highly expressive class of distributions. 
MCMs induce distributions in the set of all finite Gaussian mixtures, endowing the framework with universal approximation capabilities while preserving a tractable categorical description. 
This expressivity provides a theoretical justification for the use of MCMs as building blocks for abstraction and alignment across models, and supports both inductive uses of the framework--by designing CANs with desired structural properties--and deductive ones, in which abstraction relations and network structure are inferred from data, as shown in the financial application in \Cref{sec:finappl}.

Third, the notion of smoothness with respect to a CAN is deliberately defined in a flexible manner. 
Rather than fixing a specific discrepancy measure, smoothness can be instantiated using information-theoretic metrics or more general $\phi$-divergences, depending on the application context. 
This design choice makes the learning procedures adaptable to different domains and data modalities, while leaving the underlying categorical and sheaf-theoretic structure unchanged.
This is in stark contrast with sheaf-theoretic frameworks working on deterministic signals \citep{hansen2019learning}, where the notion of smoothness coincides with that of global section and is tightly coupled to Dirichlet-like energy. 

A further salient feature of the framework is the compositionality of causal abstractions, which arises naturally. 
Constructive linear causal abstractions compose along the edges of a CAN in a principled way, allowing global alignment and consistency properties to be characterized in terms of local relations. 
This compositional structure reflects the categorical perspective underlying the framework and enables learning procedures that operate locally while supporting coherent global reasoning.

Alongside these strengths, this work has several limitations that we plan to address in future works. 
The empirical results are intended to be illustrative rather than fully optimized, as no extensive hyper-parameter tuning was performed. 
This choice likely leads to suboptimal quantitative performance.

Scalability also remains a challenge. 
Our \mcalsep learning algorithm is not designed for high-dimensional settings, either in terms of the number of causal variables within each MCM or the size of the CAN. 
Addressing these issues will require additional algorithmic developments.

Moreover, the learning procedures focus on linear causal abstractions, which are well suited to Gaussian and Gaussian mixture settings. 
Although the categorical framework and most theoretical results are not tied to Gaussianity, extending the learning algorithms to non-Gaussian distributions would require more general abstraction maps and corresponding optimization methods.

Finally, assumptions concerning consistency of the CAN, and the availability of global sections as well as full structural priors, primarily affect the learning component rather than the theoretical framework. 
While relaxing these assumptions would increase the generality of the algorithms, in many application domains, including neuroscience \citep{d2025causal}, mechanistic interpretability of deep learning models \citep{geiger2025causal}, and finance (cf. \Cref{sec:finappl}), meaningful structural prior information is often already available in the data and naturally incorporated into the modeling process.
% !TEX root =  ../main.tex
\section{Conclusions and future works}\label{sec:conclusions}

This paper takes a step toward a principled foundation for collaborative causal AI systems. 
We introduced CAN as a general framework for representing, learning, and reasoning across collections of MCMs. 
By combining a categorical semantics with sheaf-theoretic constructions, CAN provides a principled way to formalize CLCA relations among subjective causal models operating at different levels of granularity, while remaining agnostic to explicit causal graphs, functional forms, interventional data, or jointly sampled observations.

Starting from the categorical formulation of MCM, we provide a series of theoretical results, ranging from the consistency of a CAN--closely related to SEP--to the spectral properties of the connection Laplacian, which characterize the existence of global sections, that is, subjective CK inducing perfect alignment across the network.  
Notably, we establish conditions for the convergence of the diffusion of CK over the CAN to the space of global sections, and characterize the means and covariances of the Gaussian mixtures composing the global section in terms of the eigenvectors of the normalized connection Laplacian.
We then addressed the problem of learning consistent CANs from data, proposing a principled solution that show promising results on synthetic data across different settings.
Finally, we showcase an application of the framework to real-world financial data.

In addition to the extensions discussed in \Cref{sec:discussion}, several directions for future work naturally emerge. 
A first direction involves extending the framework to higher-order topological structures, where the study of associated cohomology may shed light on obstructions to consistency and the existence of global sections.

Beyond the causal setting, both the theoretical apparatus and the proposed method for the CLCA local problems developed in this work are already naturally suited to handle non-causal network sheaves and cosheaves of probabilistic models, a setting that remains understudied in the literature. 
Investigating this class of probabilistic network models represents a promising avenue for future research.

Finally, a deeper categorical analysis of \MCMcat, potentially leveraging more advanced tools from category theory, may further clarify its structural properties and its role as a foundation for abstraction, composition, and learning.

More broadly, this work points to a set of research directions that are still largely underexplored and that may be of interest not only to the causality community, but also to a wider audience in theoretical and applied machine learning.
\medskip

\acks{The work was supported by the SNS JU project 6G-GOALS \citep{strinati2024goal} under the EU’s Horizon program Grant Agreement No 101139232, and by Huawei Technology France SASU under Grant N. Tg20250616041. 
Gabriele D'Acunto and Paolo Di Lorenzo are also with the National Inter-University Consortium for Telecommunications (CNIT), Parma, Italy.}
% \newpage

\appendix
% !TEX root =  ../main.tex
\section{Proofs}\label{app:proofs}

\begin{restated}{Lemma}{lemma:composedCA}
Given two CLCAs from $\mcm{}_i$ to $\mcm{}_j$ and from $\mcm{}_j$ to $\mcm{}_k$, their composition is a valid CLCA.
\end{restated}
\begin{proof}
    Let $\abst_{j i} = \langle \B_{j i}, \V_{j i} \rangle$ be a CLCA from $\mcm{}_i$ to $\mcm{}_j$, and $\abst_{k j} = \langle \B_{k j}, \V_{k j} \rangle$ a CLCA from $\mcm{}_j$ to $\mcm{}_k$.  
    Then,  
    \begin{equation}
    \abst_{k i} \coloneq \langle \B_{k i}, \V_{k i} \rangle = \langle \B_{k j} \B_{j i}, \, \V_{k j} \V_{j i} \rangle \,,
    \end{equation}
    is a valid CLCA.
\end{proof}

\begin{restated}{Theorem}{th:consistencyCAN}
    Let \CAN be a connected CAN.
    Then, \CAN is consistent $\iff$ for every oriented edge $e_{ij} \in \edgeset$, $\V_{ij}$ is the right-inverse of a CLCA adhering to SEP.
\end{restated}
\begin{proof}
    \spara{ $(\implies)$}
    Consistency implies that every oriented loop of arbitrary length $p$ commutes. 
    Take an undirected edge $e_{ij}: i \sim j$, with $d_j < d_i$.
    By CAN construction (cf. Def.\nb\ref{def:CAN}) $\V_{ji}=\V_{ij}^\top$.
    Thus, consistency implies that $\V_{ij}^\top \V_{ij}=\identity_{d_j}$, which corresponds to SEP.
    This holds for all oriented edges $e_{ij} \in \edgeset$.

    \spara{$(\impliedby)$}
    Consider every $\V_{ij} \in \stiefel{d_i}{d_j}$, thus being the right-inverse of a CLCA adhering to SEP.
    Consider an oriented path of arbitrary length, viz. $\mcm{}_j \rightarrow \mcm{}_{j-1} \rightarrow \dots \rightarrow \mcm{}_{i+1}\rightarrow \mcm{}_i$.
    Since the composition of the embeddings along the path amounts to a Stiefel matrix $\V_{ij}$, it follows that closing the loop yields consistency, that is, $\V_{ij}^\top \V_{ij}=\identity_{d_j}$.
    The same holds even when $j$ can reach $i$ by more than one path. 
    Let $\V_{ij}^\prime$ and $\V_{ij}^{\prime\prime}$ be the Stiefel embeddings obtained by composition along two paths from $j$ to $i$, namely $p^\prime$ and $p^{\prime\prime}$. 
    Indeed,
    \begin{equation}
        \begin{aligned}
            \V_{ij}^{\prime^\top} &= \V_{ij}^{{\prime\prime}^\top}\,, \quad \text{(by Lem.\nb\ref{lemma:composedCA})}\\
            \V_{ij}^{\prime^\top}\V_{ij}^{{\prime\prime}} &= \V_{ij}^{{\prime\prime}^\top}\V_{ij}^{{\prime}} = \eye{d_j}\,; \quad \text{(by SEP)}
        \end{aligned}
    \end{equation}
    That is, we can reach $i$ via $p^\prime$ and then go back via $p^{\prime\prime}$, and vice versa. 
\end{proof}

\begin{lemma}\label{lemma:intersection_maps}
    Consider $\V_1 \in \stiefel{i}{h}$ and $\V_2 \in \stiefel{i}{j}$, with $h<j$.
    Let $\Pi_1=\V_1\V_1^\top$ and $\Pi_2=\V_2\V_2^\top$ be the projectors induced by $\V_1$ and $\V_2$, respectively.
    Then, 
    \begin{equation}
        \dim\left(\im{\myP_1}\cap\im{\myP_2}\right)=h \iff \exists \, \V \in \stiefel{j}{h} \mid \V_1 = \V_2 \V\,.
    \end{equation}
\end{lemma}
\begin{proof}
    \spara{$(\implies)$}
    Let $\dim\left(\im{\myP_1}\cap\im{\myP_2}\right)=h$.
    Then, since $h<j$, $\im{\myP_1}\subseteq\im{\myP_2}$. 
    Hence, the columns of $\V_1$ are linear combinations of the columns of $\V_2$, meaning that it exists $\V \in \stiefel{j}{h}$ such that $\V_1 = \V_2 \V$. 
    Specifically, orthonormality of the columns of \V is implied by $\V_1^\top \V_1=\identity_h$.
    
    \spara{$(\impliedby)$}
    Consider it exists $\V \in \stiefel{j}{h}$ such that $\V_1 = \V_2 \V$.
    Then, $\im{\myP_1}\cap\im{\myP_2}=\im{\myP_1}$, and thus $\dim\left(\im{\myP_1}\cap\im{\myP_2}\right)=\dim\left(\im{\myP_1}\right)=h$.
\end{proof}

\begin{restated}{Theorem}{th:kernelCAN}
    Consider a consistent and connected \CAN whose coarsest $\mcm{}_N$ has dimension $h$.
    Then, \CANL has null eigenvalue $\mu_0$ with multiplicity $m(\mu_0)=h$ $\iff$ every $\mcm{}_i$ is reachable from $\mcm{}_{N}$ via an oriented path.
\end{restated}
\begin{proof}
    \spara{ $(\implies)$}
    Assume $m(\mu_0)=h$.
    Recall $\ker{\CANL}=\ker{\CANB^\top}$.
    A vector $\x \in \reall^{\sum_{i\in [N]} d_i}$ in $\ker{\CANB^\top}$ is a concatenation of samples $\x_i\in \reall^{d_i}$ from all the MCMs $\mcm{}_i \in \mcmcollection$ 
    such that for every oriented edge $e_{ij}$,
    \begin{equation}\label{eq:edgesection}
        \CANB(i, e_{ij})^\top \x_i + \CANB(j, e_{ij})^\top \x_j = \zeros_{d_i}\; \implies \;\x_i=\V_{ij}\x_j\,.
    \end{equation}
    Suppose there exists $\mcm{}_u$ not reachable from $\mcm{}_N$ following the \CAN orientation.
    By connectivity, this means that in \CAN there exist at least one $\mcm{}_m$ on an undirected path between $\mcm{}_N$ and $\mcm{}_u$ such that $d_u < d_m$.
    W.l.o.g., consider the path $\mcm{}_u \sim \mcm{}_m \sim \mcm{}_N$.
    By \Cref{eq:edgesection},
    \begin{equation}\label{eq:reachable}
        (i) \; \x_m = \V_{mN}\x_N \quad \text{and} \quad (ii)\; \x_m =\V_{mu}\x_u\,.      
    \end{equation} 
    However, this contradicts the structure of \CAN, implying that $u$ is reachable from $N$ as well.
    Indeed, let $\Pi_u\coloneqq \V_{mu}\V_{mu}^\top$ and $\Pi_N\coloneqq \V_{mN}\V_{mN}^\top$ be the projectors induced by $\V_{mu}$ and $\V_{mN}$, respectively.
    Thus, by Lem.\nb\ref{lemma:intersection_maps} it exists $\V_{uN} \in \stiefel{d_u}{h}$ such that $\V_{mN}=\V_{mu}\V{uN}$.

    \spara{$(\impliedby)$}
    Conversely, if every $\mcm{}_i$, $i \in [N-1]$, is reachable from $\mcm{}_N$ with an oriented path, compositionality of Stiefel matrices allows writing $\x_i = \V_{i N} \x_N$, for all $i \in [N-1]$.
    Hence, \CANL has a nontrivial kernel with $m(\mu_0)=h$ corresponding to the $h$ arbitrary dimensions. 
\end{proof}

\begin{restated}{Lemma}{lemma:fixequiv}
    Let $\chi  \in \mathrm{Fix}\left(\cc{\bm\lambda}{\catidentity{}}{\mathcal{L}}\right)$ with $\bm\lambda \in [0,1]^N$. 
    Then, $\chi=\mathcal{L}(\chi)$.
\end{restated}
\begin{proof}
    Immediate from the definition $\chi_i = \cc{\lambda_i}{\chi_i}{\mathcal{L}(\chi)\at{i}}$, for each $i \in [N]$. 
\end{proof}

\begin{restated}{Theorem}{th:globalsectionCAN}
    Let \CAN be a consistent CAN and $\mathcal{L}: \cochainspace{0}{\graph}{E}\rightarrow \cochainspace{0}{\graph}{E}$ the corresponding Laplacian operator.
    Then,
    \begin{equation}
        \mathrm{Fix}\left(\cc{\bm\lambda}{\catidentity{}}{\mathcal{L}}\right) = \gsspace{\graph}{E} \,, \quad \text{with } \bm\lambda \in [0,1]^N\,.
    \end{equation}    
\end{restated}
\begin{proof}
    \spara{$(\implies)$}
    Let $\chi \in \mathrm{Fix}\left(\cc{\bm\lambda}{\catidentity{}}{\mathcal{L}}\right)$.
    By Lem.\nb\ref{lemma:fixequiv}, $\chi_i=\mathcal{L}(\chi)\at{i}$, for all $i \in [N]$.
    Hence, starting from \Cref{eq:Laplacian_op_local}, for all $i \in [N]$ we have
    % \begin{equation}
    %     \begin{aligned}
    %         \chi_i &= \push{\V_{ij}}{\measure_j}\; \text{(embedding edges)} \\
    %         \text{and} \; \chi_k &= \push{\V_{ki}}{\chi_i} \;\text{(abstraction edges)}\,.  
    %     \end{aligned}
    % \end{equation}
    \begin{equation}
        \begin{aligned}
            \chi_i &= cc_{\w_j}\left(P_{\V_{ij}(X_j)}^{[J]}\right)\; \text{(embedding edges)} \\
            \text{and} \; \chi_k &= cc_{\w_i}\left(P_{\V_{ki}(X_i)}^{[I]}\right) \;\text{(abstraction edges)}\,.  
        \end{aligned}
    \end{equation}
    These are precisely the conditions in \Cref{eq:global_sec_conditions}, hence
    \begin{equation}\label{eq:fixglobal}
        \chi \in \mathrm{Fix}\left(\cc{\bm\lambda}{\catidentity{}}{\mathcal{L}}\right) \implies \chi \in \gsspace{\graph}{E}  \,. 
    \end{equation}

    \spara{$(\impliedby)$}
    Conversely, let $\chi \in \gsspace{\graph}{E}$.
    By substituting \Cref{eq:global_sec_conditions} in \Cref{eq:Laplacian_op_local}, for an arbitrary node $i$ of \CAN we have 
    % \begin{equation}\label{eq:global_to_fix}
    %     \begin{aligned}
    %         \mathcal{L}(\chi)\at{i} &= \sum_{e\in \starr{i}^-}\bm\beta_{i,e} \cc{\alpha_e}{\chi_i}{\chi_i} + \\
    %         &+\sum_{e\in \starr{i}^+} \bm\beta_{i,e} \push{\V_{ji}^\top}{\cc{\alpha_e}{\push{\V_{ji}}{\chi_i}}{\push{\V_{ji}}{\chi_i}}}\\
    %         &\stackrel{(a)}{=} \sum_{e \in \starr{i}} \bm\beta_{i,e} \chi_i = \chi_i\,;
    %     \end{aligned}
    % \end{equation}
    \begin{equation}\label{eq:global_to_fix}
        \begin{aligned}
            \mathcal{L}(\chi)\at{i} &= \sum_{e\in \starr{i}^-}\bm\beta_{i,e} cc_{\alpha_e}(P_{X_i},P_{X_i}) + 
            \sum_{e\in \starr{i}^+} \bm\beta_{i,e} cc_{\alpha_e}\left(P_{(\V_{ji}^\top \circ \V_{ji})(X_i)},P_{(\V_{ji}^\top \circ \V_{ji})(X_i)}\right) \\
            &\stackrel{(a)}{=} \sum_{e \in \starr{i}} \bm\beta_{i,e} P_{X_i} = \chi_i\,;
        \end{aligned}
    \end{equation}
    where $(a)$ follows by consistency. 
    Then, by Lem.\nb\ref{lemma:fixequiv}
    \begin{equation}\label{eq:globalfix}
        \chi \in \gsspace{\graph}{E} \implies \chi \in \mathrm{Fix}\left(\cc{\bm\lambda}{\catidentity{}}{\mathcal{L}}\right)\,. 
    \end{equation}
\end{proof}

\begin{restated}{Theorem}{th:globalsec_existence}
    Consider a consistent and connected \CAN, and let $h$ be the dimension of the coarsest $\mcm{}_N$.
    Then,
    \begin{equation}
        \mathrm{Fix}\left(\cc{\bm\lambda}{\catidentity{}}{\mathcal{L}}\right) \neq \emptyset \iff 0< \dim \left( \ker{\CANL} \right)\leq h\,.
    \end{equation}    
\end{restated}
\begin{proof}
    \spara{$(\implies)$}
    Assume $\mathrm{Fix}\left(\cc{\bm\lambda}{\catidentity{}}{\mathcal{L}}\right) \neq \emptyset$.
    Then, by \Cref{th:globalsectionCAN} there exist a $0$-cochain \measure such that \Cref{eq:global_sec_conditions} holds.
    Consider the component of \measure at node $i$, viz. $\measure_i$, and denote the smallest linear subspace of $\reall^{d_i}$ containing the support of $\measure_i$ by 
    \begin{equation}
        \Psi_i = \spn {\supp(\measure_i)}\,.
    \end{equation}
    By \Cref{th:globalsectionCAN}, for each edge $e_{ij} \in \edgeset$, $\Psi_i=\V_{ij} \Psi_j$. 

    Define 
    \begin{equation}
        \mathcal{K} \coloneqq \Bigl\{[\x_1^\top,\dots,\x_N^\top]^\top\in \prod_i \Psi_i \;\mid\; \x_i = \V_{ij}\x_j \,,\, \forall e_{ij}\in\edgeset \Bigr\}\,.    
    \end{equation}
    Since \measure is a non-trivial global section, there exists at least a point
    $[\x_1^\top,\dots,\x_N^\top]^\top$ with $\x_i\in\supp(\measure_i)$ satisfying the above relations, hence $\mathcal{K}\neq\{\zeros\}$.
    By definition, $\mathcal{K}\subseteq \ker(\CANB^\top)$, and therefore $\ker(\CANB^\top)\neq\{\zeros\}$ too.
    Hence, $\dim\left(\ker \CANB^\top \right)=\dim\left(\ker \CANL \right)=k\leq h$ by Lem.\nb\ref{lemma:intersection_maps}.

    \spara{$(\impliedby)$}
    Assume $0<\dim\left(\ker{\CANL} \right)=k\leq h$.
    Then, there exist a nontrivial basis 
    \begin{equation}
        \mathcal{B} \coloneqq \{\vecb_1, \ldots\, ,\vecb_k\}, \quad \vecb_b \in \reall^{\sum_i^N d_i}\,. 
    \end{equation}
    For each node $i$ let
    \begin{equation}
        \Psi_i \coloneqq \spn{\vecb_{1,i}, \ldots, \vecb_{k,i}}, \quad \vecb_{b,i} \in \reall^{d_i}\,,
    \end{equation}
    be the kernel-induced subspace having dimension $\dim \left(\Psi_i\right)=k$.
    Since $\vecb_b \in \ker{\CANL}$, for each $e_{ij} \in \edgeset$, $\vecb_{b,i} = \V_{ij} \vecb_{b,j}$, thus implying $\Psi_i = \V_{ij} \Psi_j$ for all $b \in [k]$.
    Thus, the collection $\{\Psi_i\}$ represents local realizations of the same $k$-dimensional latent vector space, coherently embedded in each node’s ambient space.

    For each node $i$, let $\B_i \in \reall^{d_i \times k}$ be the matrix whose columns are $\vecb_{1,i}, \ldots, \vecb_{k,i}$.
    Then $\im{\B_i}=\Psi_i$ and, for each $e_{ij} \in \edgeset$ it holds $\B_i = \V_{ij}\B_j$.  
    
    Let $\nu$ be any probability distribution over a random vector $X_k \in \reall^k$--e.g., a Gaussian mixture--and define for each node 
    % $\measure_i = \push{\B_i}{\nu}$.
    $P_{\B_i(X_k)}$.
    Since $\im{\B_i}=\Psi_i$ and $\Psi_i$ is a linear subspace, the support of $\measure_i$ satisfies $\supp(\measure_i) \subseteq \Psi_i$.
    Additionally, by construction, for each $e_{ij} \in \edgeset$
    % \begin{equation}
    %     \push{\V_{ij}}{\measure_j} = \push{\left(\V_{ij} \circ \B_j\right)}{\nu} = \push{\B_i}{\nu} = \measure_i\,; 
    % \end{equation}
    \begin{equation}
        P_{\V_{ij}(X_j)}=P_{(\V_{ij}\circ \B_j)(X_k)}=P_{\B_i (X_k)}=\measure_i
    \end{equation}
    where we used $\B_i = \V_{ij}\B_j$.
    Hence, the collection $\measure=\{\measure_i\}$ satisfies the global section conditions in \Cref{eq:global_sec_conditions}, and then $\mathrm{Fix}\left(\cc{\bm\lambda}{\catidentity{}}{\mathcal{L}}\right) \neq \emptyset$.
\end{proof}

\begin{restated}{Lemma}{lemma:sameeigenvalues}
    Consider $\covM_i$ and $\covM_j$ having size $d_i \!\times\! d_i$ and $d_j \!\times\! d_j$, respectively, with $\covM_i\!=\!\V_{ij}\covM_j \V_{ij}^\top$ and $\V_{ij} \!\in\! \stiefel{d_i}{d_j}$.
    Then, $\covM_i$ and $\covM_j$ have the same nonzero eigenvalues.
\end{restated}
\begin{proof}
    Using the eigendecomposition,
    \begin{equation}
        \covM_i\!=\!\V_{ij}\covM_j \V_{ij}^\top\!=\!\left(\V_{ij}\U_j \Lambda_j^{\frac{1}{2}}\right)\!\left(\V_{ij}\U_j \Lambda_j^{\frac{1}{2}}\right)^\top \!=\! \A \A^\top\,.
    \end{equation}
    Then, the proof follows by noticing that $\A^\top \A = \Lambda_j$, and recalling that $\A \A^\top$ and $\A^\top \A$ share the same nonzero eigenvalues. 
\end{proof}

\begin{restated}{Corollary}{cor:same_spectrum}
    Consider $\chi \in \mathrm{Fix}\left(\cc{\bm\lambda}{\catidentity{}}{\mathcal{L}}\right)$.
    Then, component-wise, the covariance matrices of its elements $\measure_i \in \GMM{I}{d_i}$, $i \in [N]$, have the same $k\leq \min_i d_i=h$ nonzero eigenvalues, with the equality holding in case $m(\mu_0)=h$.
\end{restated}
\begin{proof}
    The proof follows by imposing the global section conditions in \Cref{eq:global_sec_conditions} and then applying Lem.\nb\ref{lemma:sameeigenvalues} to each component.
\end{proof}

\begin{restated}{Corollary}{cor:spectral_interlacing_gmm}
    Let $\measurelow \in \GMM{S}{\ell}$, $\measurehigh \in \GMM{S}{h}$, and let \covlow and \covhigh be their mixture covariances (cf. \cref{eq:gmm_sigma}).
    Denote by $0\leq\lambda_1\leq \ldots \leq \lambda_\ell$ the eigenvalues of \covlow, and by $0\leq\kappa_1 \leq \ldots\leq \kappa_h$ those of \covhigh.
    If a CLCA complying with SEP from \measurelow to \measurehigh exists, then
    \begin{equation}\label{eq:spectralCAgmm}
        \lambda_i \leq \kappa_i \leq \lambda_{i + \ell -h}, \quad \forall \,i \in [h]\,.
    \end{equation}
\end{restated}
\begin{proof}
    If the CLCA exists, then $\measurehigh=P_{\V_{\ell h}^\top(X^\ell)}$ thus implying
    \begin{equation}\label{eq:push-mix-mom}
        \begin{aligned}
            \text{\emph{(a)}} \quad \bm\mu^h &= \sum_{s \in [S]} w_s^\ell \V_{\ell h}^\top \bm\mu_s^\ell = \V_{\ell h}^\top \bm\mu^\ell\,,\\
            \text{\emph{(b)}} \quad \covhigh&= \sum_{s \in [S]} w_s^\ell \left(\V_{\ell h}^\top \covlow_s \V_{\ell h} + \left( \V_{\ell h}^\top \bm\mu_s^\ell - \bm\mu^h\right)\left( \V_{\ell h}^\top \bm\mu_s^\ell - \bm\mu^h\right)^\top\right)\\
            \text{by \emph{(a)}}&= \V_{\ell h}^\top \left(\sum_{s \in [S]} w_s^\ell \left( \covlow_s + \left(\bm\mu_s^\ell - \bm\mu^\ell\right) \left(\bm\mu_s^\ell - \bm\mu^\ell\right)^\top\right) \right) \V_{\ell h}\\
            &= \V_{\ell h}^\top \covlow \V_{\ell h}\,.
        \end{aligned}
    \end{equation}
    Then, starting from \emph{(b)} in \cref{eq:push-mix-mom} and following the same steps in the proof of Thm.\nb4.3 in \cite{d2025causal}, the application of the Ostrowski's theorem for rectangular matrices \citep{higham1998modifying} gives \cref{eq:spectralCAgmm}.
\end{proof}

\begin{restated}{Theorem}{thm:convergence}
    Consider a consistent and connected \CAN and let $\mathrm{Fix}\left(\cc{\bm\lambda}{\catidentity{}}{\mathcal{L}}\right) \neq \emptyset$.
    Then, if $a_{ii}\neq0$ for each $i \in [N]$, \Cref{eq:rewritten-dynamics} converges to a global section $\measure^\star$ where:
    \begin{itemize}
        \item $\mathcal{M}^\star$ has elements 
        \begin{equation}
            \bm\mu_i^\star=\vecu_i^i\,; 
        \end{equation}
        where $\vecu^i=[\vecu_1^{i^\top},\ldots,\vecu_N^{i^\top}]^\top \in \reall^d$, $d =\sum_i d_i$, is an eigenvector in the $k$-dimensional kernel of the normalized Laplacian $\ker{\widetilde{\CANL}}$, viz.
        \begin{equation}
            \widetilde{\CANL}= \CAND^{-\frac{1}{2}}\CANL \CAND^{-\frac{1}{2}}\,;
        \end{equation}
        \item $\mathcal{S}^\star$ has elements 
        \begin{equation}
            \covM^\star_n \in \{\covM \in \reall^{d_n \times d_n} \mid \covM=\frac{1}{2}(\vecu^i_n\vecu_n^{j^\top}+\vecu_n^j\vecu_n^{i^\top});\, (i,j) \in [k]\times[k] \text{ and } i\leq j \}\,.
        \end{equation}
    \end{itemize}
\end{restated}
\begin{proof}
    To prove the result, we separately tackle the collection of means $\mathcal{M}(t)$ and that of covariances $\mathcal{S}(t)$ for the $0$-cochain $\measure(t)$ at time $t$.\\
    \spara{Mean.}
    From \cref{eq:rewritten-dynamics}, the dynamics for the $i$-th Gaussian mixture's mean vector is
    \begin{equation}\label{eq:local-mean-dynamics}
        \bm\mu_i(t+1) = a_{ii} \bm\mu_{i}(t)+\sum_{e \in \starr{i}^-}b_{ij}\V_{ij}\bm\mu_j(t) + \sum_{e \in \starr{i}^+}c_{ij}\V_{ji}^\top\bm\mu_j(t)\,.
    \end{equation}
    Define the block matrix $\blockM \in \reall^{d \times d}$ with blocks
    \begin{equation}\label{eq:blockMij}
        \blockM_{ij} = \begin{cases}
            a_{ii} \eye{d_i}, & \quad \text{if } i=j\,;\\
            b_{ij} \V_{ij}, & \quad \text{if } i\sim j \in \edgeset\,;\\
            c_{ij} \V_{ji}^\top, & \quad \text{if } j\sim i \in \edgeset\,;\\
            \zeros_{d_i\times d_j}, & \quad \text{otherwise.} 
        \end{cases}
    \end{equation}
    Then, denoting $\bm\mu(t)\coloneqq [\bm\mu_1^\top(t),\ldots,\bm\mu_N^\top(t)]^\top$ in $\reall^d$ and using \Cref{eq:local-mean-dynamics,eq:blockMij}, we have
    \begin{equation}\label{eq:global-mean-dynamics}
        \bm\mu(t+1) = \blockM \bm \mu(t)\,.
    \end{equation}
    We can bound the eigenvalues $\lambda^{\blockM}$ of \blockM by the Gerschgorin theorem for block matrices \citep{feingold1962block},
    \begin{equation}
        \lambda^{\blockM} \in \bigcup_{i \in [N]}\{\varphi : \norm{\varphi \eye{d_i} - a_{ii}\eye{d_i}}\leq r_i \}\,;
    \end{equation}
    where, exploiting $\norm{\V_{ij}} = 1$ as $\V_{ij} \in \stiefel{d_i}{d_j}$, 
    \begin{equation}\label{eq:ri}
        r_i = \sum_{j\neq i}\norm{\blockM_{ij}}= \sum_{j\neq i}\norm{(b_{ij}+c_{ij})\V_{ij}} = \sum_{j\neq i}(b_{ij}+c_{ij}) = 1-a_{ii}\,.
    \end{equation}
    Then, 
    \begin{equation}
        \abs{\lambda^\blockM - a_{ii}}\leq 1 - a_{ii} \implies -1+2a_{ii}\leq \lambda^\blockM\leq 1\,.
    \end{equation}
    Hence, considering all blocks $i \in [N]$, we obtain the bound
    \begin{equation}\label{eq:bound-eigvals-blockM}
        -1+2 \min_{i\in [N]} a_{ii}\leq \lambda^{\blockM}\leq 1\,.
    \end{equation}
    Specifically, $\lambda^{\blockM}=-1$ can only occur if $a_{ii}=0$ for some $i \in [N]$, which means no autoregressive contribution from $\measure_i(t)$.
    Then, supposing $a_{ii}\neq 0$ for all $i \in [N]$, the dynamics in \Cref{eq:global-mean-dynamics} is coercive on the eigenvectors corresponding to $\lambda^\blockM=1$.
    To characterize the latter eigenvectors, notice that
    \begin{equation}\label{eq:blockM-laplacian}
        \blockM = \eye{d} - \sum_{i,j} w_{ij} \mathbb{P}^i \widetilde{\CANL} \mathbb{P}^j\,;
    \end{equation}
    where 
    \begin{equation}
        w_{ij}=\begin{cases}
            1-a_{ii}, & \quad \text{if } i=j\,;\\
            b_{ij}, & \quad \text{if } i\sim j \in \edgeset\,;\\
            c_{ij}, & \quad \text{if } j\sim i \in \edgeset\,;\\
            0, & \quad \text{otherwise;} 
        \end{cases}
    \end{equation}
    and $\mathbb{P}^i$ being a selection matrix with block $\mathbb{P}_{ii}^i=\eye{d_i}$ and zero elsewhere.
    As by assumption \cref{th:globalsec_existence} holds, $\dim\ker{\CANL}=k$.
    Let $\vecv^i \in \reall^d$ be an eigenvectors of $\CANL$ corresponding to the zero eigenvalue, then $\vecu^i=\CAND^{\frac{1}{2}}\vecv^i$ is an eigenvector of $\widetilde{\CANL}$ in \Cref{eq:normalized-Laplacian} with null eigenvalue.
    As such, for each $i\sim j \in \edgeset$, it holds $\vecu_i^i=\V_{ij}\vecu_j^i$.
    Then, starting from \Cref{eq:blockM-laplacian} and denoting by $\vecu \in \reall^d$ an eigenvector in the kernel of $\widetilde{\CANL}$, for each $i \in [N]$,
    \begin{equation}
        \begin{aligned}
            \blockM \vecu\at{i}&=\vecu^i - \sum_{j \in [N]}w_{ij}\CANLl_{ij}\vecu^j \\
            &=\vecu^i - (1-a_{ii})\vecu^i + \sum_{j\neq i}(b_{ij}+c_{ij})\CAND^{-\frac{1}{2}}_{ii}\CANA_{ij} \CAND^{-\frac{1}{2}}_{jj}\vecu^{j}\\
            &=a_{ii}\vecu^i + \sum_{j\neq i}(b_{ij}+c_{ij})\CAND^{-\frac{1}{2}}_{ii}\V_{ij} \CAND^{-\frac{1}{2}}_{jj}\vecu^{j}= (a_{ii}+\sum_{j\neq i}(b_{ij}+c_{ij})) \vecu^i=\vecu^i\,.
        \end{aligned}
    \end{equation}
    Hence, \vecu is an eigenvector of \blockM corresponding to the eigenvalue 1, thus giving \Cref{eq:mu-i-star}.\\
    \spara{Covariance.}
    From \cref{eq:rewritten-dynamics}, the dynamics for the $i$-th Gaussian mixture's covariance is
    \begin{equation}\label{eq:local-cov-dynamics}
        \covM_i(t+1) = a_{ii} \covM_{i}(t)+\sum_{e \in \starr{i}^-}b_{ij}\V_{ij}\covM_j(t)\V_{ij}^\top + \sum_{e \in \starr{i}^+}c_{ij}\V_{ji}^\top\covM_j(t)\V_{ji}\,.
    \end{equation}
    For each $i \in [N]$, let $\bm\sigma_i(t)$ be equal to the column-wise vectorization of $\covM_i(t)$. 
    Then, exploiting $\myvec{\A\B\C}=(\C^\top \otimes \A)\myvec{\B}$, from \cref{eq:local-cov-dynamics} we get
    \begin{equation}\label{eq:local-cov-dynamics-vec}
        \bm\sigma_i(t+1) = a_{ii}\bm\sigma_i(t) + \sum_{e \in \starr{i}^-}b_{ij}(\V_{ij}\otimes\V_{ij})\bm\sigma_j(t) + \sum_{e \in \starr{i}^+}c_{ij}(\V_{ji}^\top\otimes \V_{ji}^\top)\bm\sigma_j(t)\,.
    \end{equation}
    Let $\bm\sigma(t)\coloneqq [\bm\sigma_1^\top(t),\ldots,\bm\sigma_N^\top(t)]^\top$ in $\reall^{d^s}$, where $d^s=d^2$, and define the block matrix $\blockMl \in \reall^{d^s\times d^s}$ with blocks
    \begin{equation}
        \blockMl_{ij}=\begin{cases}
            a_ii\eye{d^2}, &\quad\text{if $i=j$}\,;\\
            b_{ij}(\V_{ij}\otimes\V_{ij}), &\quad\text{if } i\sim j \in \edgeset\,;\\
            c_{ij}(\V_{ji}^\top\otimes\V_{ji}^\top), & \quad \text{if } j\sim i \in \edgeset\,;\\
            \zeros_{d_i\times d_j}, & \quad \text{otherwise.}
        \end{cases}
    \end{equation}
    Then, the overall dynamics for the covariances of the Gaussian mixtures is
    \begin{equation}\label{eq:global-cov-dynamics}
        \bm\sigma(t+1)=\blockMl\bm\sigma(t)\,.
    \end{equation}
    At this point define $\CANDl \in \reall^{d^s \times d^s}$ block-diagonal with blocks $\CANDl_{ii}=\mathrm{deg}(i)\eye{d_i^2}$, where $\mathrm{deg}(i)=|\starr{i}|$ is the degree of node $i$.
    Additionally, define the symmetric block matrix $\CANAl \in \reall^{d^s \times d^s}$ with blocks
    \begin{equation}
        \CANAl_{ij}=\begin{cases}
            \V_{ij}\otimes\V_{ij}, &\quad\text{if } i\sim j \in \edgeset\,;\\
            \V_{ji}^\top\otimes\V_{ji}^\top, & \quad \text{if } j\sim i \in \edgeset\,;\\
            \zeros_{d_i\times d_j}, & \quad \text{otherwise;}
            \end{cases}
    \end{equation}
    and let $\widetilde{\CANL}^{\ell} = \eye{d^s}-{\CANDl}^{-\frac{1}{2}}\CANAl{\CANDl}^{-\frac{1}{2}}$.
    Then, denoting by $\mathbb{P}^{i^\ell} \in \reall^{d^s\times d^s}$ the block-diagonal selection matrix with $\mathbb{P}^{i^\ell}=\eye{d_i^2}$ and zero elsewhere,
    we have
    \begin{equation}\label{eq:blockMlif-laplacian}
        \blockMl = \eye{d^s}-\sum_{i,j}w_{ij} \mathbb{P}^{i^\ell} \widetilde{\CANL}^{\ell}_{ij} \mathbb{P}^{j^\ell}\,.
    \end{equation}
    By exploiting that $\norm{\V_{ij}\otimes\V_{ij}}=\norm{\V_{ij}}=1$ and applying the Gerschgorin theorem as above, we get the following bounds for the eigenvalues:
    \begin{equation}\label{eq:bound-eigvals-blockMl}
        -1+2 \min_{i\in [N]} a_{ii}\leq \lambda^{\blockMl}\leq 1\,.
    \end{equation}
    Then, considering $a_{ii}\neq0$ for all $i \in [N]$, \Cref{eq:global-cov-dynamics} is coercive on the eigenspace corresponding to the eigenvalue $\lambda^{\blockMl}=1$.
    To characterize the latter, consider again the $d$-dimensional eigenvector $\vecu^i=[\vecu_1^{i^\top},\ldots,\vecu_N^{i^\top}]^\top$ belonging to the kernel of $\widetilde{\CANL}$.
    Since \Cref{th:globalsec_existence} holds, we have $k$ of these eigenvectors.
    Let us build $k(k+1)/2$ vectors  
    \begin{equation}
        \vecu^{ij} = \frac{1}{2} (\vecu^i \otimes \vecu^j + \vecu^j \otimes \vecu^i)\,;
    \end{equation}
    where $(i,j) \in [k]\times[k]$ and $i\leq j$.
    Then, it can be shown that those $\vecu^{ij}$ form a basis for the eigenspace of the eigenvalue 1 of \blockMl:
    \begin{equation}
        \begin{aligned}
            \blockMl \vecu^{ij}\at{n} &= \vecu^{ij}_n - (1-a_{nn}) \vecu^{ij}_n + \sum_{m\neq n}w_{nm} {\CANDl}^{-\frac{1}{2}}_{nn} \CANAl_{nm} {\CANDl}^{-\frac{1}{2}}_{mm} \vecu^{ij}_m \\
            &\stackrel{(a)}{=} a_{nn} \vecu^{ij}_n + \sum_{m\neq n}(b_{nm}+c_{nm})\vecu^{ij}_n = \vecu^{ij}_n, \quad \,\forall n \in [N]\,;
        \end{aligned}    
    \end{equation}
    where in $(a)$ we exploit the Kronecker product property $(\A \otimes \B)(\veca \otimes \vecb)=(\A \veca) \otimes (\B\vecb)$ and the fact that both $\vecu^i$ and $\vecu^j$ belong to $\ker{\widetilde{\CANL}}$.
    Recognizing that, for each $n$, 
    \begin{equation}
        \vecu^{ij}_n = \frac{1}{2}\left(\myvec{\vecu^i_n\vecu_n^{j^\top} + \vecu_n^j\vecu_n^{i^\top}}\right)=\myvec{\covM_n^\star}\,,
    \end{equation}
    gives \Cref{eq:sigma-i-star}.
\end{proof}

% \begin{restated}{Lemma}{lem:updateV}
%     Let $\mathcal{B}\coloneq\{i \in [d_i d_j] \mid b_i=1, b_i \in \underline{B}\}$, $\K = \C \otimes \A$, and $\mathbf{b}= \underline{\Y}^k - \underline{\BPsi}^k +\underline{\A \T^k \C^\top} -\underline{\A \BUpsilon^k \C^\top}$.
%     Then,
%     \begin{equation}\label{eq:updateV}
%         \underline{\V}_{\mathcal{B}} = \left(\identity_{|\mathcal{B}|}+\K_\mathcal{B} \K_\mathcal{B}^\top\right)^{-1} \mathbf{b}_{\mathcal{B}}\,.
%     \end{equation}
% \end{restated}
% \begin{proof}
%     Differentiating \Cref{eq:aL} with respect to $\V$ and vectorizing via the Kronecker product yields the stationarity condition
%     \begin{equation}\label{eq:stationarityV}
%         \zeros_{\ell h} = \left(\D + \D \K (\D \K)^\top\right) \underline{\V} - \D \mathbf{b}\,,
%     \end{equation}
%     where $\D = \mathrm{diag}(\underline{\B})$.
%     Only entries with $\D_{ii}=1$ are active in $\underline{\V}$; the others are zero. 
%     Restricting to indices in $\mathcal{B}$ gives \Cref{eq:updateV}.
% \end{proof}
% !TEX root =  ../main.tex
\section{Solution of Prob.~\texorpdfstring{\eqref{eq:reasoning-prob}}{(\ref{eq:reasoning-prob})}}
\label{app:sol-reasoning-prob}

Let $\veca_{s,i}\coloneqq\M_s^\top \V_{i5}^\top \y_i \in \reall^{d_5}$.
Starting from Prob.\nb\eqref{eq:reasoning-prob}, the first term in the objective can be rewritten as
\begin{equation}\label{eq:rewrite-first-term-reasoning}
    \begin{aligned}
        & \sum_{i=1}^5 \sum_{s=1}^3 w_s \left(\left(\V_{i5} \M_s \bm\mu_s^\myexogenous\right)^\top\y_i - \gamma\, \y_i^\top \V_{i5} \M_s \bm\Sigma_s^\myexogenous \M_s^\top \V_{i5}^\top \y_i \right) = \\
        & =\sum_s w_s \sum_i \left(\bm\mu_s^{\myexogenous^\top} \veca_{s,i} - \gamma \,\veca_{s,i}^\top \bm\Sigma_s^\myexogenous \veca_{s,i}\right)=\\
        & =\sum_s w_s \sum_i \left(\bm\mu_s^{\myexogenous^\top} \veca_{s,i} - \gamma \,\sum_{j=1}^{d_5} \sigma_{s,j}^\myexogenous [\veca_{s,i}]_j^2  \right)=\\
        & =\sum_s w_s \left(\bm\mu_s^{\myexogenous^\top} \underbrace{\sum_i \veca_{s,i}}_{\vecb_s \in \reall^{d_5}} - \gamma\, \sum_j \sigma_{s,j}^\myexogenous \underbrace{\sum_i [\veca_{s,i}]_j^2}_{\vecc_s \in \reall^{d_5}}  \right)=\\
        & = \sum_s w_s \left(\bm\mu_s^{\myexogenous^\top} \vecb_s - \gamma \,\bm\sigma_s^{\myexogenous^\top} \vecc_s \right)\,.
    \end{aligned}
\end{equation}
Exploiting \Cref{eq:rewrite-first-term-reasoning}, Prob. \eqref{eq:reasoning-prob} becomes
\begin{equation}\label{eq:reasoning-prob-equiv}
    \begin{aligned}
        \max_{\mathbf{w} \in \Delta_3, \{\bm\mu_s^\myexogenous, \bm\Sigma_s^\myexogenous\}} & \quad \sum_s w_s \left(\bm\mu_s^{\myexogenous^\top} \vecb_s - \gamma \,\bm\sigma_s^{\myexogenous^\top} \vecc_s \right) - \lambda_w \sum_s w_s \log w_s \, \\
        \text{subject to} & \quad \mu_{s,j}^\myexogenous \in [(1-\varphi) \mu_{s,j}, (1+\varphi)\mu_{s,j}],\quad \forall\,(s,j) \in [3]\times[d_5]\,,\\
        & \quad \sigma_{s,j}^\myexogenous \in [(1-\varphi) \sigma_{s,j}, (1+\varphi)\sigma_{s,j}],\quad \forall\,(s,j) \in [3]\times[d_5]\,.
    \end{aligned}
    \tag{$\widetilde{\text{P4}}$}
\end{equation}
Prob. \eqref{eq:reasoning-prob-equiv} is not jointly convex.
However, when $\mathbf{w}$ is fixed, the subproblem separates along $s$ and is linear and bounded in $\bm\mu_s^\myexogenous$ and $\bm\Sigma_s^\myexogenous$.
Notably, the optimal values for $\bm\mu_s^\myexogenous$ and $\bm\Sigma_s^\myexogenous$ do not depend on $w_s$ since it is nonnegative. Specifically, for $j \in [d_5]$, we have
\begin{equation}\label{eq:solution-mu}
    \widehat{\bm\mu}_{s,j}^{\myexogenous}=\begin{cases}
        (1+\varphi)\mu_{s,j} &\quad \text{if } \vecb_{s,j}>0\,,\\
        (1-\varphi)\mu_{s,j} &\quad \text{if } \vecb_{s,j}<0\,,\\
        [(1+\varphi)\mu_{s,j}, (1-\varphi)\mu_{s,j}] &\quad \text{if } \vecb_{s,j}=0\,;
    \end{cases}
\end{equation}
and 
\begin{equation}\label{eq:solution-sigma}
    \widehat{\bm\sigma}_{s,j}^{\myexogenous}=\begin{cases}
        [(1+\varphi)\sigma_{s,j}, (1-\varphi)\sigma_{s,j}] &\quad \text{if } \vecc_{s,j}=0\,,\\
        (1-\varphi)\sigma_{s,j} &\quad \text{otherwise}\,.
    \end{cases}
\end{equation}
Setting $\bm\mu_s^\myexogenous$ and $\bm\Sigma_s^\myexogenous$ to the optimal values above for each $s \in [3]$, and recalling \Cref{eq:scenario-trade-off}, the subproblem in $\mathbf{w}$ is 
\begin{equation}\label{eq:subproblem-w}
    \begin{aligned}
        \max_{\mathbf{w} \in \Delta_3} \quad \sum_s w_s R_{\gamma,s} - \lambda_w \sum_s w_s \log w_s\,.
    \end{aligned}
\end{equation}
Prob. \eqref{eq:subproblem-w} is $\lambda_w$-strongly concave, and admits a well-known closed-form solution \citep{cover1999elements}
\begin{equation}
    \widehat{w}_s = \frac{e^{R_{\gamma,s}/\lambda_w}}{\sum_s e^{R_{\gamma,s}/\lambda_w}}\,,\quad \text{for each } s \in [3]\,.
\end{equation}
% !TEX root =  ../main.tex
\section{Hyperparameters}\label{app:hyperparams}

% \begin{table}[h]
% \caption{\spectral algorithm.}
% \label{tab:spectral}
%     \centering    
%     \begin{tabular}{lrrrr}
%           Experiment&$\rho$  &  $\tau_a$&  $\tau_r$&$T$\\
%  CLCA Learning - pd case (\Cref{fig:mvG_synth_data})& $1$& $10^{-4}$& $10^{-4}$&$1000$\\
%  CAN Learning - psd case (\Cref{subfig:mvG-case})& $1$& $10^{-3}$& $10^{-3}$&$1000$\\
%     \end{tabular}
% \end{table}

\begin{table}[h]
\caption{\mcalsep algorithm.}
\label{tab:mcalsep}
    \centering
    \small
    \begin{tabular}{lrrrrrrrrrrr}
         Experiment&  $\lambda_\omega$&  $\rho$&   $\eta$&$\tau_{a}$&  $\tau_{r}$&  $\tau_\mathrm{GD}$&  $\tau_{\bm \Omega}$&  $\tau_{\V}$&  $T_\mathrm{GD}$&  $T_\mathrm{SOC}$& $T_\mathrm{GL}$\\
 CLCA Learning \\(pd case, \Cref{fig:mvG_synth_data})& $0.05$& $1$& $10^{-2}$& $10^{-4}$& $10^{-4}$& $10^{-3}$& $10^{-3}$& $10^{-3}$& $10$& $1000$&$1000$\\
 CLCA Learning \\(GM case, \Cref{fig:gmm_synth_data})& $0.05$& $1$& $10^{-2}$& $10^{-4}$& $10^{-4}$& $10^{-3}$& $10^{-3}$& $10^{-3}$& $10$& $100$&$100$\\
 CLCA Learning \\(GM case, \Cref{fig:B_not_given})& $0.05$& $1$& $10^{-2}$& $10^{-4}$& $10^{-4}$& $10^{-3}$& $10^{-3}$& $10^{-3}$& $10$& $100$&$100$\\
 CAN Learning \\(GM case, \Cref{fig:CAN_learning})& $0.05$& $1$& $10^{-2}$& $10^{-3}$& $10^{-3}$& $10^{-3}$& $10^{-2}$& $10^{-2}$& $10$& $100$&$100$\\
 CAN Learning \\(financial data, \Cref{fig:finappl-learned-ns})& $0.05$& $1$& $10^{-2}$& $10^{-4}$& $10^{-4}$& $10^{-3}$& $10^{-3}$& $10^{-3}$& $10$& $100$&$100$\\
    \end{tabular}
\end{table}

\Cref{tab:mcalsep} provide the hyperparameters of  \mcalsep used in the experiments in \Cref{sec:empirical_assessment,sec:finappl}.
Additionally, the results in \Cref{fig:CAN_learning,fig:finappl-learned-ns} were obtained by setting, in \cref{alg:search_proc}, the loss tolerance $\tau_\mathrm{MPW} = 10^{-2}$, in accordance with the results in \Cref{fig:B_not_given}. 

\bibliography{bibliography}

\end{document}